\documentclass[12pt]{alt2023} 

\title[Settling the Sample Complexity of SSPs]{Reaching Goals is Hard: Settling the Sample Complexity of the  Stochastic Shortest Path}
\usepackage{times}
\usepackage{minitoc}
\renewcommand \thepart{}
\renewcommand \partname{}

\usepackage[colorinlistoftodos,bordercolor=orange,backgroundcolor=orange!20,linecolor=orange,disable]{todonotes}
\usepackage{todonotes}
\newcommand{\ale}[1]{\todo[inline,color=blue!10]{\textbf{Ale: }#1}}
\newcommand{\lc}[1]{\todo[inline,color=yellow!10]{\textbf{Liyu: }#1}}

\definecolor{Green}{rgb}{0.13, 0.65, 0.3}
\definecolor{Amber}{rgb}{0.3, 0.5, 1.0}
\usepackage[]{color-edits}

\usepackage[utf8]{inputenc} 
\usepackage[T1]{fontenc}    
\usepackage{hyperref}       
\usepackage{url}            
\usepackage{booktabs}       
\usepackage{amsfonts}       
\usepackage{nicefrac}       
\usepackage{microtype}      
\usepackage{mathtools}
\usepackage{natbib}
\usepackage{xcolor}
\usepackage{amsmath}
\usepackage{mathtools}
\usepackage{amssymb}
\usepackage{algpseudocode}

\usepackage{hhline}
\usepackage{makecell}
\usepackage{multirow}

\usepackage{algorithm}

\definecolor{azure(colorwheel)}{rgb}{0.0, 0.5, 1.0}

\newcommand{\sinit}{s_{\text{init}}}

\newcommand{\calA}{{\mathcal{A}}}

\newcommand{\calX}{{\mathcal{X}}}
\newcommand{\calS}{{\mathcal{S}}}
\newcommand{\calF}{{\mathcal{F}}}

\newcommand{\calI}{{\mathcal{I}}}

\newcommand{\calE}{{\mathcal{E}}}

\newcommand{\calT}{{\mathcal{T}}}

\newcommand{\calM}{{\mathcal{M}}}
\newcommand{\calN}{{\mathcal{N}}}

\DeclareMathOperator*{\argmin}{argmin}

\newcommand{\eat}[1]{}

\newcommand{\rbr}[1]{\left(#1\right)}
\newcommand{\sbr}[1]{\left[#1\right]}
\newcommand{\cbr}[1]{\left\{#1\right\}}

\newcommand{\abr}[1]{\left|#1\right|}

\newcommand{\tilO}[1]{\otil\left( #1 \right)}
\newcommand{\lowO}[1]{\lorder\left( #1 \right)}
\newcommand{\bigo}[1]{\order( #1 )}
\newcommand{\tilo}[1]{\otil( #1 )}
\newcommand{\lowo}[1]{\lorder( #1 )}

\DeclarePairedDelimiter\floor{\lfloor}{\rfloor}

\newcommand{\tilC}{\widetilde{C}}

\newcommand{\sumhm}{\sum_{h=1}^{H_m}}
\newcommand{\sumhmp}{\sum_{h=1}^{H_m+1}}

\newcommand{\summp}{\sum_{m=1}^{M'}}
\newcommand{\frA}{\mathfrak{A}}
\newcommand{\jstar}{j^{\star}}

\newcommand{\hatn}{\widehat{n}}

\newcommand{\T}{\ensuremath{T_\star}}
\newcommand{\B}{B_\star}
\newcommand{\cmin}{\ensuremath{c_{\min}}}

\newcommand{\dev}{\textsc{Dev}}

\newcommand{\var}{\textsc{Var}}

\newcommand{\SA}{\calS\times\calA}

\renewcommand{\P}{\bar{P}}

\newcommand{\istar}{i^{\star}}
\newcommand{\Np}{\N^+}

\newcommand{\optV}{V^{\star}}
\newcommand{\optQ}{Q^{\star}}

\newcommand{\hatQ}{\widehat{Q}}
\newcommand{\hatV}{\widehat{V}}

\newcommand{\tilV}{\widetilde{V}}

\newcommand{\opttilV}{\widetilde{V}^{\star}}

\newcommand{\sumh}{\sum_{h=1}^H}

\newcommand{\sumk}{\sum_{k=1}^K}

\newcommand{\optpi}{\pi^\star}

\newcommand{\hatT}{\widehat{T}}

\newcommand{\tils}{\widetilde{s}}
\newcommand{\tila}{\widetilde{a}}

\newcommand{\tilP}{\widetilde{P}}

\newcommand{\N}{\mathbf{N}} 

\newcommand{\hatv}{\widehat{v}}

\newcommand{\Tc}{T_{\ddagger}}
\newcommand{\hatN}{\widehat{N}}
\newcommand{\da}{a_{\dagger}}
\newcommand{\hattau}{\widehat{\tau}}
\newcommand{\LCBVI}{\text{LCBVI}}

\newcommand{\uT}{\overline{T}}

\newcommand{\sstar}{s^{\star}}

\newcommand{\hatpi}{\widehat{\pi}}
\newcommand{\tstar}{t^{\star}}
\newcommand{\bars}{\bar{s}}
\newcommand{\barh}{\bar{h}}

\newcommand{\field}[1]{\mathbb{#1}}

\newcommand{\fR}{\field{R}}

\newcommand{\fN}{\field{N}}
\newcommand{\E}{\field{E}}
\newcommand{\fV}{\field{V}}

\newcommand{\Ind}{\field{I}}

\newcommand{\norm}[1]{\left\|{#1}\right\|}
\newtheorem{assumption}{Assumption}

\newcommand{\order}{\ensuremath{\mathcal{O}}}
\newcommand{\lorder}{\ensuremath{\Omega}}
\newcommand{\otil}{\ensuremath{\tilde{\mathcal{O}}}}

\usepackage{prettyref}
\newcommand{\pref}[1]{\prettyref{#1}}
\newcommand{\pfref}[1]{Proof of \prettyref{#1}}
\newcommand{\savehyperref}[2]{\texorpdfstring{\hyperref[#1]{#2}}{#2}}
\newrefformat{eq}{\savehyperref{#1}{Eq.~\textup{(\ref*{#1})}}}
\newrefformat{eqn}{\savehyperref{#1}{Equation~\ref*{#1}}}
\newrefformat{lem}{\savehyperref{#1}{Lemma~\ref*{#1}}}
\newrefformat{def}{\savehyperref{#1}{Definition~\ref*{#1}}}
\newrefformat{line}{\savehyperref{#1}{Line~\ref*{#1}}}
\newrefformat{thm}{\savehyperref{#1}{Theorem~\ref*{#1}}}
\newrefformat{corr}{\savehyperref{#1}{Corollary~\ref*{#1}}}
\newrefformat{cor}{\savehyperref{#1}{Corollary~\ref*{#1}}}
\newrefformat{sec}{\savehyperref{#1}{Section~\ref*{#1}}}
\newrefformat{subsec}{\savehyperref{#1}{Section~\ref*{#1}}}
\newrefformat{app}{\savehyperref{#1}{Appendix~\ref*{#1}}}
\newrefformat{assum}{\savehyperref{#1}{Assumption~\ref*{#1}}}
\newrefformat{ex}{\savehyperref{#1}{Example~\ref*{#1}}}
\newrefformat{fig}{\savehyperref{#1}{Figure~\ref*{#1}}}
\newrefformat{alg}{\savehyperref{#1}{Algorithm~\ref*{#1}}}
\newrefformat{rem}{\savehyperref{#1}{Remark~\ref*{#1}}}
\newrefformat{conj}{\savehyperref{#1}{Conjecture~\ref*{#1}}}
\newrefformat{prop}{\savehyperref{#1}{Proposition~\ref*{#1}}}
\newrefformat{proto}{\savehyperref{#1}{Protocol~\ref*{#1}}}
\newrefformat{prob}{\savehyperref{#1}{Problem~\ref*{#1}}}
\newrefformat{claim}{\savehyperref{#1}{Claim~\ref*{#1}}}
\newrefformat{que}{\savehyperref{#1}{Question~\ref*{#1}}}
\newrefformat{op}{\savehyperref{#1}{Open Problem~\ref*{#1}}}
\newrefformat{fn}{\savehyperref{#1}{Footnote~\ref*{#1}}}
\newrefformat{tab}{\savehyperref{#1}{Table~\ref*{#1}}}
\newrefformat{fig}{\savehyperref{#1}{Figure~\ref*{#1}}}
\newrefformat{prop}{\savehyperref{#1}{Property~\ref*{#1}}}

\altauthor{
 \Name{Liyu Chen}\thanks{Research conducted when the author was an intern at Meta.} \Email{liyuc@usc.edu}\\
 \addr University of Southern California
 \AND
 \Name{Andrea Tirinzoni} \Email{tirinzoni@meta.com}\\
 \addr Meta
 \AND
 \Name{Matteo Pirotta} \Email{pirotta@meta.com}\\
 \addr Meta
 \AND
 \Name{Alessandro Lazaric} \Email{lazaric@meta.com}\\
 \addr Meta
}

\begin{document}

\maketitle


\doparttoc 
\faketableofcontents 

\begin{abstract}
We study the sample complexity of learning an $\epsilon$-optimal policy in the Stochastic Shortest Path (SSP) problem. We first derive sample complexity bounds when the learner has access to a generative model. We show that there exists a worst-case SSP instance with $S$ states, $A$ actions, minimum cost $c_{\min}$, and maximum expected cost of the optimal policy over all states $B_{\star}$, where any algorithm requires at least $\Omega(SAB_{\star}^3/(c_{\min}\epsilon^2))$ samples to return an $\epsilon$-optimal policy with high probability. Surprisingly, this implies that whenever $\cmin=0$ an SSP problem may not be learnable, thus revealing that learning in SSPs is strictly harder than in the finite-horizon and discounted settings. We complement this result with lower bounds when prior knowledge of the hitting time of the optimal policy is available and when we restrict optimality by competing against policies with bounded hitting time. Finally, we design an algorithm with matching upper bounds in these cases. This settles the sample complexity of learning $\epsilon$-optimal polices in SSP with generative models.

We also initiate the study of learning $\epsilon$-optimal policies without access to a generative model (i.e., the so-called best-policy identification problem), and show that sample-efficient learning is impossible in general. On the other hand, efficient learning can be made possible if we assume the agent can directly reach the goal state from any state by paying a fixed cost. We then establish the first upper and lower bounds under this assumption.

Finally, using similar analytic tools, we prove that horizon-free regret is impossible in SSPs under general costs, resolving an open problem in \citep{tarbouriech2021stochastic}.
\end{abstract}


\tolerance 1414
\hbadness 1414 
\emergencystretch 1.5em
\hfuzz 0.3pt
\widowpenalty=10000
\vfuzz \hfuzz
\raggedbottom

\section{Introduction}

The Stochastic Shortest Path (SSP) formalizes the problem of finding a policy that reaches a designated goal state while minimizing the cost accumulated over time. This setting subsumes many important application scenarios, such as indoor and car navigation, trade execution, and robotic manipulation. The SSP is strictly more general than the popular finite-horizon and discounted settings~\citep[see e.g.,][]{bertsekas2012dynamicvol2,tarbouriech2020no} and it poses specific challenges due to the fact that no horizon is explicitly prescribed in the definition of the problem. In fact, different policies may have varying hitting times to the goal, e.g., the optimal policy may not be the policy with smallest hitting time whereas some policies may not even reach the goal.

While planning in SSPs is a widely studied and well-understood  topic~\citep{bertsekas1991analysis,bertsekas2013stochastic}, the problem of online learning in SSP, often referred to as \emph{goal-oriented reinforcement learning (GRL)}, only recently became an active venue of research~\citep{tarbouriech2020no,tarbouriech2021sample,tarbouriech2021stochastic,rosenberg2020adversarial,cohen2020near,cohen2021minimax,chen2021implicit,chen2021improved,chen2021minimax,chen2022policy,chen2021finding,chen2022near,jafarnia2021online,vial2021regret,min2021learning,zhao2022dynamic}.
Most of the literature focuses on the regret minimization objective\footnote{We refer the reader to \pref{app:related} for a detailed summary of prior works in related settings.}, for which learning algorithms with minimax-optimal performance are available even when no prior knowledge about the optimal policy is provided (e.g., its hitting time or the range of its value function).
On the other hand, the \emph{probably approximately correct} (PAC) objective, i.e., to learn an $\epsilon$-optimal policy with high probability with as few samples as possible, has received little attention so far. One reason is that, as it is shown in~\citep{tarbouriech2021sample}, in SSP it is not possible to convert regret into sample complexity bounds through an online-to-batch conversion~\citep{jin2018q} and PAC guarantees can only be derived by developing specific algorithmic and theoretical tools. Assuming access to a generative model, \citet{tarbouriech2021sample} derived the first PAC algorithm for SSP with sample complexity upper bounded as  $\tilo{\frac{\Tc\B^2\Gamma SA}{\epsilon^2}}$, where $S$ is the number of states, $A$ is the number of actions, $\Gamma$ is the largest support of the transition distribution, $\B$ is the maximum expected cost of an optimal policy over all states, $\Tc=\B/\cmin$, where $\cmin$ is the minimum cost over all state-action pairs, and $\epsilon$ is the desired accuracy.
The most intriguing aspect of this bound is the dependency on $\Tc$, which represents a worst-case bound on the hitting time $\T$ of the optimal policy (i.e., the horizon of the SSP) and it depends on the inverse of the minimum cost. While some dependency on the horizon may be unavoidable, as conjectured in~\citep{tarbouriech2021sample}, we may expect the horizon to be independent of the cost function\footnote{Notice that in general $\T \ll \B/\cmin$.}, as in finite-horizon and discounted problems. Moreover, in regret minimization, there are algorithms whose regret bound only scales with $\T$, with no dependency on $\cmin$, even when $\cmin=0$ and no prior knowledge is available. It is thus reasonable to conjecture that the sample complexity should also scale with $\T$ instead of $\Tc$. This leads us to the first question addressed in this paper:

\begin{center}
    {\textit{Question 1: Is the dependency on $\Tc=\B/\cmin$ in the sample complexity of learning with a generative model unavoidable?}}
\end{center}

Surprisingly, we derive a lower bound providing an affirmative answer to the question. In particular, we show that $\lowo{\frac{\Tc\B^2SA}{\epsilon^2}}$ samples are needed to learn an $\epsilon$-optimal policy, showing that a dependency on $\Tc$ is indeed unavoidable and that it is not possible to adapt to the optimal policy hitting time $\T$\footnote{In our proof, we construct SSP instances where $\T < \Tc$.}. This result also implies that there exist SSP instances with $c_{\min}=0$ (i.e., $\Tc = \infty$) that are \textit{not learnable}. This shows for the first time that not only SSP is a strict generalization of the finite-horizon and discounted settings, but it is also strictly harder to learn. We then derive lower bounds when  prior knowledge of the form $\uT \geq \T$ is provided or when an optimality criterion restricted to policies with bounded hitting time is defined. Finally, we propose a simple algorithm based on a finite-horizon reduction argument and we prove upper bounds for its sample complexity matching the lower bound in each of the cases considered above; see Table~\ref{tab:generative}.

When no access to a generative model is provided, the learner needs to directly execute a policy to collect samples to improve their estimate of the optimal policy. No result is currently available for this setting, often referred to as the \textit{best-policy identification} (BPI) problem, and in this paper we address the following question:

\begin{center}
    {\textit{Question 2: Is it possible to efficiently learn a near-optimal SSP policy when no access to a generative model is provided?}}
\end{center}

In this setting, we first derive a lower bound showing that in general sample efficient BPI is impossible.  To resolve this negative result, we introduce an extra assumption requiring that the learner can reach the goal by paying a fixed cost $J$ from any state. We then establish a $\lowo{\min\{\Tc, \uT\}\frac{\B^2SA}{\epsilon^2} + \frac{J}{\epsilon}}$ lower bound under this assumption. We also develop a finite-horizon reduction based algorithm with sample complexity $\tilo{\frac{\Tc\B^2SA}{\epsilon^2} + \frac{\B J^4S^2A^2}{\cmin^3\epsilon}}$, whose dominating term is minimax-optimal when $\cmin>0$ and prior knowledge $\uT$ is unavailable. This result is summarized in Table~\ref{tab:generative}.

Finally, we show how similar technical tools derived for our lower bounds can be adapted to resolve an open question in~\cite{tarbouriech2021stochastic} by showing that in regret minimization, a worst-case dependency on the hitting time of the optimal policy is indeed unavoidable without any prior knowledge (see Appendix \ref{app:hf}).

\begin{table}[t]
	\centering
	\resizebox{.8\columnwidth}{!}{
		\renewcommand{\arraystretch}{2}
        \begin{tabular}{|c|c|c|c|}
        \hline
        \makecell{Performance\\(gen model)} & Lower Bound  & Upper Bound & \cite{tarbouriech2021sample} \\
        \hhline{|====|}
    	\makecell{$(\epsilon,\delta)$\\\pref{def:correct}} & $\min\cbr{\Tc, \uT}\frac{\B^2SA}{\epsilon^2}$ & $\min\cbr{\Tc, \uT}\frac{\B^2SA}{\epsilon^2}$ & $\frac{\Tc\B^2\Gamma SA}{\epsilon^2}$ \\
    	\hline
        \makecell{$(\epsilon,\delta,T)$\\\pref{def:generative.T}} & $\frac{TB_{\star,T}^2SA}{\epsilon^2}$ when $\min\{\Tc,\uT\}=\infty$ & $\min\cbr{\Tc, T}\frac{B_{\star,T}^2SA}{\epsilon^2}$ & $\frac{TB_{\star,T}^3\Gamma SA}{\epsilon^3}$ \\
        \hline
        \multicolumn{1}{c}{}&\multicolumn{1}{c}{}&\multicolumn{1}{c}{}&\multicolumn{1}{c}{}\\
        \hline
         \makecell{Performance\\(BPI)} &Assumption & Lower Bound & Upper Bound \\
         \hhline{|====|} 
         \multirow{2}{*}{\makecell{$(\epsilon,\delta)$\\\pref{def:correct}}} & None & $\frac{A^{\lowo{\min\{\B,S\}}}}{\epsilon}$ & -\\
         \hhline{~|-|-|-|}
         & \pref{assum:terminal} & $\min\cbr{\Tc, \uT}\frac{\B^2SA}{\epsilon^2} + \frac{J}{\epsilon}$ & $\frac{\Tc\B^2SA}{\epsilon^2} + \frac{\B J^4S^2A^2}{\cmin^3\epsilon}$ \\
         \hline
    \end{tabular}
    }
 
    \caption{\small Result summary \emph{with (upper table) and without (lower table) a generative model}.
    Here, $\uT$ is a known upper bound on the hitting time of the optimal policy ($\uT=\infty$ when such a bound is unknown), $\Tc=\frac{\B}{\cmin}$, $B_{\star,T}$ is the maximum expected cost over all starting states of the restricted optimal policy with hitting time bounded by $T$, and $J$ is the cost to directly reach the goal from any state (\pref{assum:terminal}).
    Operators $\tilo{\cdot}$ and $\lowo{\cdot}$ are hidden for simplicity.
    }
    \label{tab:generative}
\end{table}

\section{Preliminaries}\label{sec:preliminaries}

An SSP instance is denoted by a tuple $\calM=(\calS, \calA, g, c, P)$, where $\calS$ with $S=|\calS|$ is the state space, $\calA$ with $A=|\calA|$ is the action space, $g\notin\calS$ is the goal state, $c\in[\cmin,1]^{\calS_+\times \calA}$ with $\cmin\in[0,1]$, $c(g,a)=0$ for all $a$, and $\calS_+=\calS\cup\{g\}$ is the cost function, and $P=\{P_{s,a}\}_{(s,a)\in \calS_+\times\calA}$ with $P_{s,a}\in\Delta_{\calS_+}$ and $P(g|g,a)=1$ for all $a$ is the transition function, where $\Delta_{\calS_+}$ is the simplex over $\calS_+$.
All of these elements are known to the learner except the transition function.



A stationary policy $\pi$ assigns an action distribution $\pi(\cdot|s)\in\Delta_{\calA}$ to each state $s\in\calS$.
A policy is \textit{deterministic} if $\pi(\cdot|s)$ concentrates on a single action (denoted by $\pi(s)$) for all $s$. Denote by $T^{\pi}(s)$ the expected number of steps it takes to reach $g$ starting from state $s$ and following $\pi$.
A policy is \textit{proper} if starting from any state it reaches the goal state with probability $1$ (i.e., $T^\pi(s) < \infty$ for all $s\in\calS$), and it is \textit{improper} otherwise (i.e., there exists $s\in\calS$ such that $T^\pi(s)=\infty$). 
We denote by $\Pi$ the set of stationary policies, and $\Pi_{\infty}$ the set of stationary proper policies.

Given a cost function $c$ and policy $\pi$, the value function of $\pi$, $V^{\pi}\in[0,\infty]^{\calS_+}$ is defined as  $V^{\pi}(s)=\E_{\pi}[\sum_{i=1}^{\infty}c(s_i, a_i)|s_1=s]$, where the randomness is w.r.t.\ $a_i\sim \pi(\cdot|s_i)$ and $s_{i+1}\sim P_{s_i,a_i}$.
We define the optimal proper policy $\optpi=\argmin_{\pi\in\Pi_{\infty}}V^{\pi}(s)$ for all $s\in\calS$, and we write $V^{\optpi}$ as $\optV$. 
It is known that $\optpi$ is stationary and deterministic.

We introduce a number of quantities that play a major role in characterizing the learning complexity in SSP: $\B=\max_s\optV(s)$, the maximum expected cost of the optimal policy starting from any state, $\T=\max_sT^{\optpi}(s)$,
and $D=\max_s\min_{\pi\in\Pi}T^{\pi}(s)$, the diameter of the SSP instance. Then, we have that
\[
\B\leq D\leq\T\leq\frac{\B}{\cmin} =: \Tc.
\]
Note that these inequalities may be strict and the gap arbitrarily large. Furthermore, this shows that the knowledge of $\B$ (or an upper bound) does not only provide an information about the range of the value function but also a worst-case bound $\Tc$ on the horizon $\T$.
%
We assume $\B\geq 1$, a commonly made assumption in previous work of SSP~\citep{tarbouriech2021stochastic,chen2021implicit}.\footnote{
Note that in regret minimization, the lower bounds for $\B\geq1$ and $\B<1$ are different~\citep{cohen2021minimax}.}

\paragraph{Learning objective}
The goal of the learner is to identify a near-optimal policy of desired accuracy with high probability, with or without a generative model.
We formalize each component below.

\subparagraph{Sample Collection} \textit{With a generative model (PAC-SSP)}, the learner directly selects a state-action pair $(s, a)\in\SA$ and collects a sample of the next state $s'$ drawn from $P_{s, a}$.
\textit{Without a generative model (i.e., Best Policy Identification (BPI))}, the learner directly interacts with the environment through episodes starting from an initial state $\sinit$ and sequentially taking actions until $g$ is reached. 

\subparagraph{$\epsilon$-Optimality} With a generative model, we say a policy $\pi$ is $\epsilon$-optimal if $V^{\pi}(s)-\optV(s)\leq\epsilon$ for all $s\in\calS$.
Without a generative model, a policy $\pi$ is $\epsilon$-optimal if $V^{\pi}(\sinit)-\optV(\sinit)\leq\epsilon$.

\ale{Alternative version?}
\begin{definition}[$(\epsilon,\delta)$-Correctness]
    \label{def:correct}
    Let $\calT$ be the random stopping time by when an algorithm terminates its interaction with the environment and returns a policy $\hatpi$. We say that an algorithm is $(\epsilon,\delta)$-correct with sample complexity $n(\calM)$ if $\mathbb{P}_{\calM}(\calT\leq n(\calM), {\hatpi}\text{ is $\epsilon$-optimal in $\calM$})\geq1-\delta$ for any SSP instance $\calM$, where $n(\calM)$ is a deterministic function of the characteristic parameters of the problem (e.g., number of states and actions, inverse of the accuracy $\epsilon$).
\end{definition}

\paragraph{Other notation}
We denote by $\uT$ an upper bound of $\T$ known to the learner, and let $\uT=\infty$ if such knowledge is unavailable.
The $\tilo{\cdot}$ operator hides all logarithmic dependency including $\ln\frac{1}{\delta}$ for some confidence level $\delta\in(0, 1)$. 
For simplicity, we often write $a=\tilo{b}$ as $a\lesssim b$.
Define $(x)_+=\max\{0, x\}$.
For $n\in\fN_+$, define $[n]=\{1,\ldots,n\}$.

\section{Lower Bounds with a Generative Model}\label{sec:lower.bounds}


\begin{figure}
    \centering
    \tikz[font=\scriptsize, scale=0.9]{

        \begin{scope}
            \node[draw, circle] (g) at (0,0) {$g$};
            \node[draw, circle] (s0) at (0,-2) {$s_0$};
            \node[circle, fill, minimum size=4pt,inner sep=0pt, outer sep=0pt] (da1) at (-0.4,-1.2) {};
            \node[circle, fill, minimum size=4pt,inner sep=0pt, outer sep=0pt] (da2) at (0.4,-1.2) {};
            \draw[double,->] (s0) -- (da1);
            \draw[double,->] (s0) -- (da2);
            \path[->] (da1) edge[bend left] node[left,anchor=east] {$p_0 \simeq \frac{1+\epsilon/B_\star}{T_\star}$} (g);
            \path[->] (da2) edge[bend right] node[right, yshift=-5pt, text width=2cm] {$p_i\simeq\frac{1}{\Tc}$} (g); 
            \path[->] (da1) edge[bend right, out=-90, in=-80] (s0);
            \path[->] (da2) edge[bend left, out=90, in=80] (s0);
            \node at (0,-1) {$\ldots$};
            \node[anchor=west] at (-2.5,0) {Action $0$};
            \node[anchor=west] at (-2.2,-1.25) {$c_0 \simeq \frac{\B}{\T}$};
            \node[anchor=west] at (0.5,0) {Action $i\geq 1$};
            \node[anchor=west] at (0.6,-1.5) {$c_i \simeq \frac{\B}{\Tc}$};
        \end{scope}

        \begin{scope}[shift={(.23\textwidth,0)}]
            
            \node[draw, circle, inner sep=2pt, minimum size=12pt, font=\tiny] (s021) at (0,-1.5) {$s_0$};
            \node[draw, circle, inner sep=2pt, minimum size=12pt, font=\tiny] (g21) at (0,0) {$g$};
            \node[draw, circle, inner sep=2pt, minimum size=12pt, font=\tiny] (s121) at (1,0) {$s_1$};
            \path[->] (s021) edge[] node[text width=1cm] {$a_g$\\$c=\frac{1}{2}$} (g21);
            \path[->] (s121) edge[] node[above, yshift=5pt] {$a_0,a_g, c=1$} (g21);
            \path[->] (s021) edge[loop right] node[above, xshift=4pt] {$a_0, c=0$} (s021);

            \node[draw, circle, inner sep=2pt, minimum size=12pt, font=\tiny] (s022) at (3,-1.5) {$s_0$};
            \node[draw, circle, inner sep=2pt, minimum size=12pt, font=\tiny] (g22) at (3,0) {$g$};
            \node[draw, circle, inner sep=2pt, minimum size=12pt, font=\tiny] (s122) at (4,0) {$s_1$};
            \path[->] (s022) edge[] node[text width=1cm] {$a_g$\\$c=\frac{1}{2}$} (g22);
            \path[->] (s122) edge[] node[above, yshift=5pt] {$a_0,a_g, c=1$} (g22);
            \node[circle, fill, minimum size=4pt,inner sep=0pt, outer sep=0pt] (da122) at (4.8,-2) {};
            \draw[double,->] (s022) -- (da122);
            \path[->] (da122) edge[bend left, out=90] node[below] {$a_0, c=0$} (s022);
            \path[->,dashed] (da122) edge[bend right] node[xshift=5pt] {in $\calM_+$} (s122);
            \path[->,dotted] (da122) edge[bend left] node[xshift=15pt] {in $\calM_-$} (g22);

        \end{scope}

        \begin{scope}[shift={(.1\textwidth,0)}]
            \node[draw, circle] (s03) at (8.5,-2) {$s_0$};
            \node[draw, circle] (g3) at (11.5,-1.8) {$g$};
            \node[draw, circle] (s13) at (8.5,-0.3) {$s_1$};
            \node[draw, circle] (s23) at (10,-.3) {$s_2$};
            \node[draw, circle] (sn3) at (11.5,-.3) {$s_N$};
            \node[circle, fill, minimum size=4pt,inner sep=0pt, outer sep=0pt] (da13) at (9,-1.3) {};
            \node at (10.7, -0.3) {$\ldots$};
            \path[->] (da13) edge node[anchor=west, right] {$p= \frac{\epsilon}{A^N}$} (s13);
            \path[->] (s13) edge node[above] {$\bar{a}$} (s23);
            \path[->] (sn3) edge (g3);
            \path[->] (da13) edge node[below]{$1-p$} (g3);
            \draw[double,->] (s03) --node[right, yshift=-5pt] {$\{a_i\}$} (da13);
            \path[->] (s13) edge[loop above] node[right]{$\{a_i\} \setminus \bar{a}$} (s13);
        \end{scope}

        \node[font=\normalsize] at (0.05\textwidth, -2.9) {(a)};
        \node[font=\normalsize] at (0.4\textwidth, -2.9) {(b)};
        \node[font=\normalsize] at (0.75\textwidth, -2.9) {(c)};
    }
    \caption{(a) hard instance (simplified for proof sketch) in \pref{thm:lower.bound.eps} when $\cmin>0$.
    (b) hard instance in \pref{thm:lower.bound.eps} when $\cmin=0$.
    (c) hard instance in \pref{thm:lower.bound.BPI}.
    Here, $c$ represents the cost of an action, while $p$ represents the transition probability.
    }
    \label{fig:fig}
\end{figure}



In this section, we derive lower bounds on PAC-SSP in various cases.

\subsection{Lower Bound for $\epsilon$-optimality}\label{ssec:lower.bound.general}


We first establish the sample complexity lower bound of any $(\epsilon,\delta)$-correct learning algorithm when no prior knowledge is available.

\begin{theorem}
\label{thm:lower.bound.eps}
    For any $S\geq 3$, $A\geq 3$, $\cmin > 0$, $B\geq 2$, $T_0\geq\max\{B,\log_AS+1\}$, $\epsilon\in(0, \frac{1}{32})$, and $\delta\in(0, \frac{1}{2e^4})$ such that $T_0\leq B/\cmin$, there exists an MDP with $S$ states, $A$ actions, minimum cost $\cmin$, $\B=\Theta(B)$, and $\T=\Theta(T_0)$, such that any $(\epsilon, \delta)$-correct algorithm has sample complexity $\lowO{\frac{\Tc\B^2SA}{\epsilon^2}\ln\frac{1}{\delta}}$.\footnote{Formally, for any $n\geq 0$, we say that an algorithm has sample complexity $\lowO{n}$ on an SSP instance $\calM$ if $\mathbb{P}_{\calM}(\calT \leq n, {\hatpi}\text{ is $\epsilon$-optimal in $\calM$}) \geq 1-\delta$.}
    There also exists an MDP with $\cmin=0$, $\T=1$, $\uT=\infty$, and $\B=1$ in which every $(\epsilon, \delta)$-correct algorithm with $\epsilon\in(0, \frac{1}{2})$ and $\delta\in(0, \frac{1}{16})$ has infinite sample complexity.
\end{theorem}
Details are deferred to \pref{app:lower.bound.eps.T}.
We first remark that the lower bound qualitatively matches known PAC bounds for the discounted and finite-horizon settings in terms of its dependency on the size of the state-action space and on the inverse of the squared accuracy $\epsilon$. As for the dependency on $\B$ and $\cmin$, it can be conveniently split in two terms: \textbf{1)} a term $\B^2$ and \textbf{2)} a factor $\Tc=\B/c_{\min}$. 

\paragraph{Dependency on $\B^2$} This term is connected to the range of the optimal value function $V^{\star}$.
Interestingly, in finite-horizon and discounted settings $H$ and $1/(1-\gamma)$ bound the range of the value function of \textit{any} policy, whereas in SSP a more refined analysis is required to avoid dependencies on, e.g., $\max_\pi V^\pi$, which can be unbounded whenever an improper policy exists.

\paragraph{Dependency on $\Tc$}
While $\Tc$ is an upper bound on the hitting time of the optimal policy, in the construction of the lower bound $\T$ is strictly smaller than $\Tc$.
For the case $c_{\min}>0$, this shows that the algorithm proposed by~\citet{tarbouriech2021sample} has an optimal dependency in $\B$ and $\Tc$.
On the other hand, this reveals that in certain SSP instances no algorithm can return an $\epsilon$-optimal policy after collecting a finite number of samples. \emph{This is the first evidence that learning in SSPs is strictly harder than the finite-horizon and discounted settings, where the sample complexity is always bounded.}
This is also in striking contrast with results in regret minimization in SSP, where the regret is bounded even for $\cmin=0$ and no prior knowledge about $\B$ or $\T$ is provided.
This is due to the fact that the regret measures performance in the cost dimension and the algorithm is allowed to change policies within and across episodes.
On the other hand, in learning with a generative model the performance is evaluated in terms of the number of samples needed to confidently commit to a policy with performance $\epsilon$-close to the optimal policy.
This requires to distinguish between proper and improper policies, which can become arbitrarily hard in certain SSPs where $\cmin=0$.

\paragraph{Proof Sketch} In order to provide more insights about our result, here we present the main idea of our hard instances construction.
We consider two cases separately: 1) $\cmin>0$ and 2) $\cmin=0$.
When $\cmin>0$, our construction is a variant of that in \citep[Theorem 1]{mannor2004sample}; see an illustration in \pref{fig:fig} (a).
Let's consider an MDP $\calM$ with a multi-arm bandit structure: it has a single state $s_0$ and $N+1$ actions $\calA=\{0,1,\ldots,N\}$ (in the general case this corresponds to $N+1$ state-action pairs).
Taking action $0$ incurs a cost $\frac{B}{T_0}$ and transits to the goal state with probability $\frac{1+\overline{\epsilon}/2}{T_0}$ (stays in $s_0$ otherwise), where $\overline{\epsilon}=\frac{32\epsilon}{B}$.
For each $i \in [N]$, taking action $i$ incurs a cost $\frac{B}{T_1}$ and transits to the goal state with probability $\frac{1}{T_1}$.
Note that in $\calM$ the optimal action (deterministic policy) is $0$, with $\B=\Theta(B)$, $\T=\Theta(T_0)$, $\Tc=\Theta(T_1)$, whereas all other actions are more than $\epsilon$ suboptimal. 
Also note that it takes $\lowo{\frac{T_1\B^2}{\epsilon^2}}$ samples to estimate the expected cost of action $i\in[N]$ with accuracy $\epsilon$.
If an algorithm $\frA$ spends $o(\frac{T_1\B^2}{\epsilon^2})$ samples on some action $i'$, then we can consider an alternative MDP $\calM'$, whose only difference compared to $\calM$ is that taking action $i'$ transits to the goal state with probability $\frac{1+\overline{\epsilon}}{T_1}$.
Note that in $\calM'$ the only $\epsilon$-optimal action is $i'$.
However, algorithm $\frA$ cannot distinguish between $\calM$ and $\calM'$ since it does not have enough samples on action $i'$, and thus has a high probability on outputting the wrong action in either $\calM$ or $\calM'$.
Applying this argument to each arm $i\in[N]$, we conclude that an $(\epsilon,\delta)$-correct algorithm needs at least $\lowo{\frac{NT_1\B^2}{\epsilon^2}}=\lowo{\frac{\Tc\B^2SA}{\epsilon^2}}$ samples.

We emphasize that in our construction, $\T$ (whose proxy is $T_0$) can be arbitrarily smaller than $\Tc$ (whose proxy is $T_1$) in $\calM$.
However, the learner still needs $\lowo{\frac{\Tc\B^2}{\epsilon^2}}$ samples to exclude the alternative $\calM'$ in which $\T=\Tc$.
A natural question to ask is what if we have prior knowledge on $\T$, which could potentially reduce the space of alternative MDPs.
We answer this in \pref{sec:lower.bound.prior}.

When $\cmin=0$, we consider a much simpler MDP $\calM$ with two states $\{s_0, s_1\}$ and two actions $\{a_0, a_g\}$; see an illustration in \pref{fig:fig} (b).
At $s_0$, taking $a_0$ transits to $s_0$ with cost $0$ and taking $a_g$ transits to $g$ with cost $\frac{1}{2}$.
At $s_1$, taking both actions transits to $g$ with cost $1$.
Clearly $\cmin=0$, $\B=\T=1$, $\optV(s_0)=\frac{1}{2}$, and both actions in $s_0$ are $\epsilon$-optimal in $\calM$.
Now consider any algorithm $\frA$ with sample complexity $n<\infty$ on $\calM$, and without loss of generality, assume that $\frA$ outputs a deterministic policy $\hatpi$.
We consider two cases: 1) $\hatpi(s_0)=a_0$ and 2) $\hatpi(s_0)=a_g$.
In the first case, consider an alternative MDP $\calM_+$, whose only difference compared to $\calM$ is that taking $a_0$ at $s_0$ transits to $s_1$ with probability $\frac{1}{n}$, and to $s_0$ otherwise.
Note that the optimal action at $s_0$ is $a_g$ in $\calM_+$.
Since $\frA$ uses at most $n$ samples, with high probability it never observes transition $(s_0, a_0, s_1)$ and is unable to distinguish between $\calM$ and $\calM_+$.
Thus, it still outputs $\hatpi$ with $\hatpi(s_0)=a_0$ in $\calM_+$.
This gives $V^{\hatpi}(s_0)-\optV(s_0)=1-\frac{1}{2}$ and $\hatpi$ is not $\epsilon$-optimal for any $\epsilon\in(0,\frac{1}{2})$.
In the second case, consider another alternative MDP $\calM_-$, whose only difference compared to $\calM$ is that taking $a_0$ at $s_0$ transits to $g$ with probability $\frac{1}{n}$, and to $s_0$ otherwise.
The optimal action at $s_0$ is $a_0$ in $\calM_-$.
Again, algorithm $\frA$ cannot distinguish between $\calM$ and $\calM_-$ and still output $\hatpi$ with $\hatpi(s_0)=a_g$ in $\calM_-$.
This gives $V^{\hatpi}(s_0)-\optV(s_0)=\frac{1}{2}$ and $\hatpi$ is not $\epsilon$-optimal for any $\epsilon\in(0,\frac{1}{2})$.
Combining these two cases, we have that any $(\epsilon,\delta)$-correct algorithm with $\epsilon\in(0,\frac{1}{2})$ cannot have finite sample complexity.

\paragraph{Remark} Our construction reveals that the potentially infinite horizon in SSP does bring hardness into learning when $\cmin=0$.
Indeed, we can treat $\calM$ as an infinite-horizon MDP due to the presence of the self-loop at $s_0$.
Any algorithm that uses finite number of samples cannot identify all proper policies in $\calM$, that is, it can never be sure whether $(s_0, a_0)$ has non-zero probability of reaching states other than $s_0$.


\subsection{Lower Bound for $\epsilon$-optimality with Prior Knowledge on $\T$}\label{sec:lower.bound.prior}


Now we consider the case where the learning algorithm has some prior knowledge $\uT \geq \T$ on the hitting time of the optimal proper policy. Intuitively, we expect the algorithm to exploit the knowledge of parameter $\uT$ to focus on the set of policies $\{\pi:\norm{T^\pi}_{\infty}\leq \uT\}$ with bounded hitting time.\footnote{Notice that $\{\pi:\norm{T^\pi}_{\infty}\leq \uT\}$ includes the optimal policy by definition since $\uT \geq \T$.}

\begin{theorem}
    \label{thm:lower.bound.eps.T}
    For any $S\geq 3$, $A\geq 3$, $\cmin\geq 0$, $B\geq 2$, $T_0\geq \max\{B,\log_AS+1\}$, $\uT\geq 0$, $\epsilon\in(0, \frac{1}{32})$, and $\delta\in(0, \frac{1}{2e^4})$ such that $T_0 \leq \min\{\uT/2, B/\cmin\} < \infty$, there exist an MDP with $S$ states, $A$ actions, minimum cost $\cmin$, $\B=\Theta(B)$, and $\T=\Theta(T_0)\leq \uT$, such that any $(\epsilon, \delta)$-correct algorithm has sample complexity $\lowO{\min\cbr{\Tc, \uT}\frac{\B^2SA}{\epsilon^2}\ln\frac{1}{\delta}}$.
\end{theorem}
Details are deferred to \pref{app:lower.bound.eps.T}.
The main idea of proving \pref{thm:lower.bound.eps.T} still follows from that of \pref{thm:lower.bound.eps}.
Also note that the bound in \pref{thm:lower.bound.eps.T} subsumes that of \pref{thm:lower.bound.eps} since we let $\uT=\infty$ when such knowledge is unavailable.

\paragraph{Dependency on $\min\{\Tc, \uT\}$}
We distinguish two regimes: \textbf{1)} When $\uT \leq \Tc$ the bound reduces to $\lowo{\frac{\uT \B^2SA}{\epsilon^2}}$ with no dependency on $c_{\min}$.
In this case, an algorithm may benefit from its prior knowledge to effectively prune any policy with hitting time larger than $T$, thus reducing the sample complexity of the problem and avoiding infinite sample complexity when $\cmin=0$.
\textbf{2)} When $\uT > \Tc$, we recover the bound $\lowo{\frac{\Tc\B^2SA}{\epsilon^2}}$.
In this case, an algorithm does not pay the price of a loose upper bound on $\T$.
%
Again, in our construction it is possible that $\T< \min\{\Tc, \uT\}$.
This concludes that it is impossible to adapt to $\T$ for computing $\epsilon$-optimal policies in SSPs.

\subsection{Lower Bound for $(\epsilon,T)$-optimality}
Knowing that we cannot solve for an $\epsilon$-optimal policy when $\min\{\Tc,\uT\}=\infty$, that is, $\cmin=0$ and $\uT=\infty$, we now consider a restricted optimality criterion where we only seek $\epsilon$-optimality w.r.t.\ a set of proper policies.

\begin{definition}[Restricted $(\epsilon,T)$-Optimality] \label{def:restricted.optimality}
    For any $T\in [1,\uT]$, we define the set $\Pi_T=\{\pi\in\Pi: \norm{T^{\pi}}_{\infty}\leq T\}$.
    Also define $\optpi_{T,s}=\argmin_{\pi\in\Pi_T}V^{\pi}(s)$, $V^{\star,T}(s)=V^{\optpi_{T,s}}(s)$, and $B_{\star,T}=\max_sV^{\star,T}(s)$.
    We say that a policy $\pi$ is $(\epsilon,T)$-optimal if $V^{\pi}(s)-V^{\star,T}(s)\leq\epsilon$ for all $s\in\calS$. We define $V^{\star,T}(s)=\infty$ for all $s$ when $\Pi_T=\emptyset$.\footnote{\cite{tarbouriech2021sample} consider a slightly different notion of restricted optimality, where they let $T=\theta D$ with $\theta\in[1,\infty)$ as input to the algorithm, and $D$ is unknown.}
\end{definition}
When $T\geq\T$, we have $\optpi_{T,s}=\optpi$ for all $s$.
When $D\leq T<\T$, the policy $\optpi_{T,s}$ exists and may vary for different starting state $s$ due to the hitting time constraint.
It can even be stochastic from the literature of constrained MDPs~\citep{altman1999constrained}.
\footnote{Consider an MDP with one state and two actions. Taking action one suffers cost 1 and directly transits to the goal state. Taking action 2 suffers cost 0 and transits to the goal state with probability $1/3$. Now consider $T=2$. Then the optimal constrained policy should take action 2 with probability $3/4$.}
When $T<D$, we have $\Pi_T=\varnothing$, and $V^{\star,T}(s)=\infty$ for all $s$.
Clearly, $V^{\star,T}(s)\geq \optV(s)$ for any $s$ and $T$.
%
\begin{definition}[$(\epsilon,\delta,T)$-Correctness]
    \label{def:generative.T}
    Let $\calT$ be the random stopping time by when an algorithm terminates its interaction with the environment and returns a policy $\hatpi$. We say that an algorithm is $(\epsilon,\delta,T)$-correct with sample complexity $n(\calM)$ if $\mathbb{P}_{\calM}(\calT\leq n(\calM), {\hatpi}\text{ is $(\epsilon,T)$-optimal in $\calM$})\geq1-\delta$ for any SSP instance $\calM$, where $n(\calM)$ is a deterministic function of the characteristic parameters of the problem (e.g., number of states and actions, inverse of the accuracy $\epsilon$).
\end{definition}

\ale{This may be seen a bit as ``cheating'' since you may exploit the fact that you ``go out'' of the constraint to recover $(\epsilon, T)$ optimality.}
Note that $\pi$ being $(\epsilon,T)$-optimal does not require $\pi\in\Pi_T$. For example, $\optpi$ is $(\epsilon,T)$-optimal for any $T\geq 1$. Similarly, policy output by an $(\epsilon, \delta, T)$-correct algorithm is not required to be in $\Pi_T$, and it is allowed return a better cost-oriented policy.
Now we establish a sample complexity lower bound of any $(\epsilon, \delta, T)$-correct algorithm when $\min\{\Tc,\uT\}=\infty$ (see \pref{app:lower.bound.T} for details).

\begin{theorem}
    \label{thm:lower.bound.T}
    For any $S\geq 6$, $A\geq 8$, $\B\geq 2$, $T\geq 6(\log_{A-1}(S/2)+1)$, $B_T\geq 2$, $\epsilon\in(0, \frac{1}{32})$, and $\delta\in(0, \frac{1}{8e^4})$ such that $\B \leq B_T \leq \B(A-1)^{S/2-1}/4$ and $B_T\leq T/6$, and for any $(\epsilon,\delta, T)$-correct algorithm, there exist an MDP with $B_{\star,T} = \Theta(B_T)$, $\cmin=0$, and parameters $S$, $A$, $\B$, such that with a generative model, the algorithm has sample complexity $\lowO{\frac{TB_{\star,T}^2SA}{\epsilon^2}\ln\frac{1}{\delta}}$.
\end{theorem}
Note that when $T\geq\T$, the lower bound reduces to $\frac{T\B^2SA}{\epsilon^2}$, which coincides with that of \pref{thm:lower.bound.eps.T}.
On the other hand, the sample complexity lower bound for computing $(\epsilon, T)$-optimal policy when $\min\{\Tc,\uT\}<\infty$ is still unknown and it is an interesting open problem.

\paragraph{Proof Sketch} We consider an MDP $\calM$ with state space $\calS=\calS_T\cup\calS_{\star}$.
The learner can reach the goal state either through states in $\calS_T$ or $\calS_{\star}$, where in the first case the learner aims at learning an $(\epsilon,T)$-optimal policy, and in the second case the learner aims at learning an $\epsilon$-optimal policy.
In $\calS_T$, we follow the construction in \pref{thm:lower.bound.eps} so that learning an $\epsilon$-optimal policy on the sub-MDP restricted on $\calS_T$ takes $\lowo{\frac{TB_{\star,T}^2SA}{\epsilon^2}}$ samples.
In $\calS_{\star}$, we consider a sub-MDP that forms a chain similar to \citep[Figure 1]{strens2000bayesian}, where the optimal policy suffer $\B$ cost but a bad policy could suffer $\lowo{\B A^S}$ cost.
For each state $s$ in $\calS_{\star}$, we make the probability of transiting back to $s$ by taking any action large enough, so that learner with sample complexity of order $\tilo{\frac{TB_{\star,T}^2SA}{\epsilon^2}}$ hardly receive any learning signals in $\calS_{\star}$.
Therefore, any algorithm with $\tilo{\frac{TB_{\star,T}^2SA}{\epsilon^2}}$ sample complexity should focus on learning the sub-MDP restricted on $\calS_T$. 
This proves the statement.


\section{Algorithm with a Generative Model}
\label{sec:alg.generative}
\DontPrintSemicolon
\setcounter{AlgoLine}{0}
\begin{algorithm}[H]
    \caption{Search Horizon}
    \small
    \label{alg:sh}
    \textbf{Input:} hitting time bound $T$ ($T=\uT$ with prior knowledge), accuracy $\epsilon\in(0,1)$, and probability $\delta\in(0,1)$.

    \textbf{Initialize:} $i\leftarrow 1$. 
 
    \nl Let $B_i=2^i$, $H_i=4\min\{B_i/\cmin, T\}\ln(48B_i/\epsilon)$, $c_{f,i}(s)=0.6B_i\Ind\{s\neq g\}$, $\delta_i=\delta/(40i^2)$, \label{line:define BH}
    $N^{\star}_i= N^{\star}(B_i,H_i,\frac{\epsilon}{2},\delta_i)$ and $N_i=\hatN(B_i,H_i,0.1B_i,\delta_i)$, where $N^{\star}$, $\widehat{N}$ are defined in \pref{lem:fine bound} and \pref{lem:coarse bound} respectively.
	
	{\color{azure(colorwheel)}\small\textsc{// Estimate $B_{\star,T}$}}
	
	\While{True}{
            Reset counter $\N$, and then draw $N_i$ samples for each $(s, a)$ to update $\N$.
 
		\nl $\pi^i, V^i = \text{LCBVI}(H_i, \N, B_i, c_{f,i}, \delta_i)$ (refer to \pref{alg:LCBVI}).\label{line:estimate Vi}
		
		\nl \lIf{$\norm{V^i_1}_{\infty}\leq 0.1B_i$ \textbf{and} $\max_{h\in[H+1]}\norm{V^i_h}_{\infty}\leq 0.7B_i$}{\textbf{break}.}\label{line:stop Vi}

            $i\leftarrow i+1$.

            \nl\lIf{$B_i>40T$}{\textbf{output} $\hatpi=\varnothing$. {\color{azure(colorwheel)}\small\textsc{\hspace{.8in}// i.e., $T<D$ (every policy is $(\epsilon,T)$-optimal)}} }\label{line:T<D}
	}
	{\color{azure(colorwheel)}\small\textsc{// Compute $\epsilon$-optimal policy}}
	
        Reset counter $\N$, and then draw $N^{\star}_i$ samples for each $(s, a)$ to update $\N$.

	\nl $\hatpi, \hatV =\text{LCBVI}(H_i, \N, B_i, c_{f,i}, \delta_i)$ (refer to \pref{alg:LCBVI}).\label{line:compute hatpi}

	\textbf{Output:} policy $\hatpi$ extended to infinite horizon.
\end{algorithm}

In this section, we present an algorithm whose sample complexity matches all the lower bounds introduced in \pref{sec:lower.bounds}.
We notice that the horizon (or hitting time) of the optimal policy plays an important role in the lower bounds.
Thus, a natural algorithmic idea is to explicitly determine and control the horizon of the output policy.
This leads us to the idea of finite-horizon reduction, which is frequently applied in the previous works on SSP~\citep[e.g.,][]{chen2021improved,chen2021minimax,cohen2020near}.

Now we formally describe the finite-horizon reduction scheme.
Given an SSP $\calM$, let $\calM_{H,c_f}$ be a time-homogeneous finite-horizon MDP with horizon $H$ and terminal cost $c_f\in [0,\infty)^{\calS_+}$, which has the same state space, action space, cost function, and transition function as $\calM$.
When interacting with $\calM_{H,c_f}$, the learner starts in some initial state and stage $h$, it observes state $s_h$, takes action $a_h$, incurs cost $c(s_h, a_h)$, and transits to the next state $s_{h+1}\sim P_{s_h,a_h}$.
It also suffers cost $c_f(s_{H+1})$ before ending the interaction.
When the finite-horizon MDP $\calM_{H,c_f}$ is clear from the context, we define $V^{\pi}_h(s)$ as the expected cost of following policy $\pi$ starting from state $s$ and stage $h$ in $\calM_{H,c_f}$.

Although the finite-horizon reduction has become a common technique in regret minimization for SSP, it is not straightforward to apply it in our setting.
Indeed, even if we solve a near-optimal policy in the finite-horizon MDP, it is unclear how to apply the finite-horizon policy in SSP, where any trajectory may be much longer than $H$.
Our key result is a lemma that resolve this.
It turns out that when the terminal cost in the finite-horizon MDP is large enough, all we need for applying the finite-horizon policy to SSP is to repeat it periodically.
Specifically, given a finite-horizon policy $\pi\in(\Delta_{\calA})^{\calS\times[H]}$, we abuse the notation and define $\pi\in(\Delta_{\calA})^{\calS\times\fN_+}$ as an infinite-horizon non-stationary policy, such that $\pi(a|s,h+iH)=\pi(a|s,h), \forall i\in \mathbb{N}_+$.
The following lemma relates the performance of $\pi$ in $\calM$ to its performance in $\calM_{H,c_f}$ (see \pref{app:extend} for details).

\begin{lemma}
    \label{lem:extend}
    For any SSP $\calM$, horizon $H$, and terminal cost function $c_f$, suppose $\pi$ is a policy in $\calM_{H,c_f}$ and $V^{\pi}_1(s)\leq c_f(s)$ for all $s\in\calS_+$.
    Then $V^{\pi}(s)\leq V^{\pi}_1(s)$.
\end{lemma}
Thanks to the lemma above, for a given horizon $T$, we can first learn an $\epsilon$-optimal policy $\hatpi$ in $\calM_{H,c_f}$ with $H=\tilo{T}$ and $c_f(s)=\bigo{B_{\star,T}\Ind\{s\neq g\}}$, and then extend it to an SSP policy with performance $V^{\hatpi}(s)\leq V^{\hatpi}_1(s) \approx \optV_1(s)\approx \optV(s)$, where $\optV_1$ is the optimal value function of stage $1$ in $\calM_{H,c_f}$, and the last step is by the fact that $H$ is sufficiently large compared to $T$. As a result, $\hatpi$ would then be $(\epsilon,T)$-optimal policy in the original SSP problem.
\pref{alg:sh} builds on this idea. It takes a hitting time upper bound $T$ as input, and aims at computing an $(\epsilon, T)$-optimal policy.
The main idea is to search the range of $B_{\star,T}$ and $\min\{\Tc, T\}$ via a doubling trick on estimators $B_i$ and $H_i$ (\pref{line:define BH}).
\footnote{Note that $\norm{T^{\optpi_{T,s}}}_{\infty}\leq\Tc$ for any $T\geq D$ and state $s$ since $\optpi_{T,s}=\optpi$ when $T\geq\Tc\geq\T$.}
For each possible value of $B_i$ and $H_i$, we compute an optimal value function estimate with $0.1B_i$ accuracy using $SAN_i\lesssim S^2AH_i$ samples (\pref{line:estimate Vi}), and stop if $B_i$ becomes a proper upper bound on the estimated value function (\pref{line:stop Vi}).
Here we need different conditions bounding $V^i_1$ and $V^i_h$ as the terminal cost $c_f$ should be negligible starting from stage $1$ but not for any stage.
Once we determine their range, we compute an $\epsilon$-optimal finite-horizon policy with final values of $B_i$ and $H_i$ using $SAN^{\star}_i\lesssim \frac{H_iB_i^2SA}{\epsilon^2}$ samples (\pref{line:compute hatpi}).
On the other hand, if $B_i$ becomes unreasonably large, then the algorithm claims that $T<D$ (\pref{line:T<D}), in which case $V^{\star,T}(s)=\infty$ for any $s$ (see \pref{def:restricted.optimality}), and any policy is $(\epsilon,T)$-optimal by definition.

\ale{People may wonder why using an optimistic estimate instead of just the ``raw'' estimate.\lc{It could work for PAC-SSP, but for BPI we do need optimistic algorithm. Here we share the estimate algorithm for both settings.}} 
In the procedure described above, we need to repeatedly compute a near-optimal policy with various accuracy and horizon.
We use a simple variant of the UCBVI algorithm~\citep{azar2017minimax,zhang2020reinforcement} to achieve this (see \pref{alg:LCBVI} in \pref{app:LCBVI}).
The main idea is to compute an optimistic value function estimate by incorporating a Bernstein-style bonus (\pref{line:update rule}).


We state the guarantee of \pref{alg:sh} in the following theorem (see \pref{app:bound.algo.generative} for details).
\begin{theorem}\label{thm:bound.algo.generative}
For any given $T\geq 1$, $\epsilon\in(0,1)$, and $\delta\in(0,1)$, with probability at least $1-\delta$, \pref{alg:sh} either uses $\tilO{S^2AT}$ samples to confirm that $T<D$,
or uses $\tilO{\min\left\{ \Tc,T\right\}\frac{B_{\star,T}^2SA}{\epsilon^2}}$ samples to output an $(\epsilon, T)$-optimal policy (ignoring lower order terms).
\end{theorem}

When prior knowledge $\infty>\uT\geq\T$ is available, we simply set $T=\uT$. 
In this case, we have that $(\epsilon,T)$-optimality is equivalent to $\epsilon$-optimality, $B_{\star,\uT}=\B$  and \pref{alg:sh} matches the lower bound in \pref{thm:lower.bound.eps.T}.
When $\min\{\Tc, \uT\}=\infty$, \pref{alg:sh} computes an $(\epsilon,T)$-optimal policy with $\tilo{\frac{TB_{\star,T}^2SA}{\epsilon^2}}$ samples, which matches the lower bound in \pref{thm:lower.bound.T}.
Thus, our algorithm is minimax optimal in all cases considered in \pref{sec:lower.bounds}.




When comparing with the results in~\citep{tarbouriech2021sample}, in terms of computing an $\epsilon$-optimal policy, we remove a $\Gamma$ factor and improve the dependency of $\Tc$ to $\min\{\Tc,\uT\}$, that is, our algorithm is able to leverage a given bound on $\T$ to improve sample efficiency while theirs cannot.
In terms of computing an $(\epsilon, T)$-optimal policy, we greatly improve over their result by removing a $\frac{B_{\star,T}\Gamma}{\epsilon}$ factor and improving the dependency of $T$ to $\min\{\Tc, T\}$, that is, our algorithm automatically adapts to a smaller hitting time upper bound of the optimal policy.


Finally, it is interesting to notice that even though the $(\epsilon, T)$-optimal policy is possibly stochastic, the policy output by \pref{alg:sh} is always deterministic, and it does not necessarily have hitting time bounded by $T$. In fact, \pref{alg:sh} puts no constraint on the hitting time of the output policy, except that the horizon for the reduction is $\tilO{T}$.
Nevertheless, as shown in \pref{thm:bound.algo.generative}, we can still prove that the policy is $(\epsilon,T)$-optimal since the requirement only evaluates the expected cost and not the constraint on the hitting time.
\ale{We could mention as a future work to derive an algorithm for the constrained case.}

\section{Lower Bounds without a Generative Model}
\label{sec:lower.bounds.BPI}
In this section, we consider the best policy identification problem in SSP, where the learner collects samples by interacting with the environment.
This is a more challenging setting compared to having access to a generative model since the learner cannot ``teleport'' to any arbitrary state-action pair but it only observes trajectories obtained by playing policies from some initial state.
Yet, this setting is more practical, and it naturally generalizes beyond tabular MDPs.
Surprisingly, we find that BPI with polynomial number of samples is impossible in general.

\begin{theorem}
    \label{thm:lower.bound.BPI}
    For any $S\geq 4$, $A\geq 4$, $\B\geq 1$, $\cmin\geq 0$, $\epsilon\in(0, \frac{1}{4})$, and $\delta\in(0, \frac{1}{16})$, and any $(\epsilon, \delta)$-correct algorithm, there exists an MDP with parameters $S$, $A$, $\B$, and $\cmin$, such that without a generative model, the sample complexity of the algorithm is $\lowO{\frac{A^{\min\{\floor{\B},S-3\}}}{\epsilon}}$.
\end{theorem}
Details are deferred to \pref{app:lower.bound.BPI}.
The exponential dependency $A^{\lowo{\min\{\B,S\}}}$ implies that sample efficient BPI is impossible.
The intuition of our construction is that if there are $N$ unvisited states where the learner has no samples, then the learner may suffer $\lowo{A^N}$ cost on visiting these states.
Therefore, the learner needs to spend $\lowo{A^N/\epsilon}$ samples on estimating the transition distribution to guarantee $\epsilon$-optimality when there are $N$ hardly reachable states; see an illustration in \pref{fig:fig} (c).

To enable sample efficient BPI, we need to avoid the extreme event described above.
One natural idea is to allow the learner to get out of ``unfamiliar'' states by paying a fixed cost. 
This assumption also appears in \citep[Section I.2]{tarbouriech2020no} in the context of non-communicating SSPs.
\begin{assumption}
    \label{assum:terminal}
    There is an action $\da\in\calA$ with $c(s,\da)=J$ and $P(g|s,\da)=1$ for all $s\in\calS$.
\end{assumption}
%
This assumption guarantees that there is a proper policy $\pi$ with $\|V^\pi\|_{\infty} = J$ and $\|T^\pi\|_{\infty} = 1$.
We show in \pref{sec:alg.BPI} that under \pref{assum:terminal}, efficient BPI is indeed possible.
To better understand the difficulty of BPI under this assumption, we also establish the following lower bound.

\begin{theorem}
    \label{thm:lower.bound.terminal}
    Under \pref{assum:terminal}, for any $S\geq 8$, $A\geq 5$, $\cmin\geq 0$, $B\geq \max\{2, (\log_{A-1}S+1)\cmin\}$, $\uT\geq 2\max\{B,\log_{A-1}S+1\}$, $J\geq 3B$, $\epsilon\in(0, \frac{1}{32})$, and $\delta\in(0, \frac{1}{2e^4})$ such that $\min\{\uT/2, B/\cmin\} < \infty$ and $J\leq\frac{1}{2}(A-1)^N$ with $N=\min\{\floor{B}, S-3\}$, and for any $(\epsilon, \delta)$-correct algorithm,
    there exist an MDP with parameters $S$, $A$, $\cmin$, $\uT$, and $\B=\Theta(B)$, such that the algorithm has sample complexity $\lowO{\min\cbr{\Tc, \uT}\frac{\B^2SA}{\epsilon^2}\ln\frac{1}{\delta} + \frac{J}{\epsilon}}$.
\end{theorem}
Details are deferred to \pref{app:lower.bound.terminal}.
Compared to the lower bound in \pref{thm:lower.bound.eps.T}, it has an extra $\frac{J}{\epsilon}$ term, showing that BPI is harder even with the extra assumption.
The first term in the lower bound implies that within easily reachable states, the sample complexity of BPI is similar to having access to a generative model.
Compared to the lower bound in \pref{thm:lower.bound.BPI}, we replace the exponential factor $A^S$ by $J$, which reflects the worst case cost of encountering ``unfamiliar'' states.
Whether the dependency on $J$ is minimax optimal remains an important future direction.

\section{Algorithm without a Generative Model}
\label{sec:alg.BPI}

\setcounter{AlgoLine}{0}
\begin{algorithm}[t]
    \caption{BPI-SSP}
    \label{alg:BPI-SSP}
    \small
    \textbf{Define:} $\calN=\{2^j\}_{j\geq 0}$, $H=\frac{32J}{\cmin}\ln\frac{8J}{\epsilon}$, $c_f(s)=J\Ind\{s\neq g\}$.
    
    \textbf{Initialize:} $B\leftarrow 1$, $m\leftarrow 1$,
%
    $\N(s, a)\leftarrow 0$ and $\N(s, a, s')\leftarrow 0$ for any $(s, a, s')\in\SA\times\calS_+$.

    \nl\For{$r=1,\ldots$}{\label{line:round}

        \While{True}{
            \nl $\pi^r, V^r \leftarrow \LCBVI(H,\N,2J,c_f,\frac{\delta}{2B})$.\label{line:hat compute}
    
            \lIf{$B\geq\max_{h\leq H/2+1,s}V^r_h(s)$}{
            \textbf{break}.}
            
            \nl $B\leftarrow 2\cdot\max_{h\leq H/2+1,s}V^r_h(s)$.\label{line:B}
        }

        \nl $\hattau\leftarrow0$, $\lambda\leftarrow N_{\dev}(B ,\frac{\epsilon}{4}, \frac{\delta}{2r^2})$ (defined in \pref{lem:N dev}).\label{line:lambda}
        
        \For{$m'=1,\ldots,\lambda$}{
            \For{$h=1,\ldots,H$}{
                Observe $s^m_h$, take action $a^m_h=\pi^r(s^m_h, h)$, and transit to $s^m_{h+1}$.

                \nl $\N(s^m_h, a^m_h)\overset{+}{\leftarrow}1$, $\N(s^m_h, a^m_h, s^m_{h+1})\overset{+}{\leftarrow} 1$, $\hattau\overset{+}{\leftarrow}c(s^m_h, a^m_h)/\lambda$. \label{line:behavior}

                \nl \If{$\N(s^m_h, a^m_h) \in \calN$}{\label{line:update}
                    \lIf{$s^m_{h+1}\neq g$}{take action $\da$ to reach $g$, and suffer cost $J$.}
                    
                    Return to \pref{line:round} (skip round).
                }

                \lIf{$s^m_{h+1}=g$}{\textbf{break}.}
                
                \lIf{$h=H$}{take action $\da$ to reach $g$, and suffer cost $J$.}
            }

            \nl \lIf{$\hattau> V^r(\sinit) + \frac{\epsilon}{2}$}{return to \pref{line:round} (failure round).}\label{line:failure}

            $m\overset{+}{\leftarrow}1$.
        }

        \nl \textbf{Return} policy $\hatpi=\pi^r$ (success round).\label{line:success}
    }
    
\end{algorithm}

In this section, we develop an efficient algorithm for BPI under \pref{assum:terminal}.
It is actually unclear how to design such an algorithm at first glance.
The main difficulty lies in deciding when to invoke the action $\da$.
In fact, if we simply apply the commonly used optimism based algorithm, then $\da$ may never be selected since $J\geq\B$.
Intuitively, we want to involve $\da$ for states with large uncertainty, which conflicts with the principle of optimism in the face of uncertainty.
Therefore, we need a more carefully designed scheme to balance exploration and exploitation.

It turns out that we can obtain a naive sample complexity bound by reducing BPI in SSP to BPI in a specific finite-horizon MDP.
Consider a finite-horizon MDP $\calM_{H,c_f}$ with $H=\tilo{\frac{J}{\cmin}}$ and $c_f(s)=J\Ind\{s\neq g\}$.
Executing policy $\pi$ in $\calM_{H,c_f}$ corresponds to following policy $\pi$ in $\calM$ for $H$ steps and then taking action $\da$ if the goal state is not reached.
Thus, any policy $\pi$ in $\calM_{H,c_f}$ can directly extend to $\calM$ by defining $\pi(s,H+1)=\da$, and we have $V^{\pi}=V^{\pi}_1$.
Moreover, by the choice of $H$ and $c_f$, the optimal policy in $\calM_{H,c_f}$ is also near-optimal in $\calM$.
Applying any minimax-optimal finite-horizon BPI algorithm (for example, a variant of \citep[Algorithm 1]{tarbouriech2022adaptive}), we can solve an $\epsilon$-optimal policy with $\tilo{\frac{H\norm{\optV_{\cdot}}_{\infty}^2SA}{\epsilon^2}}=\tilo{\frac{J^3SA}{\cmin\epsilon^2}}$ samples, where $\optV_h$ is the optimal value function in $\calM_{H,c_f}$ and $\norm{\optV_{\cdot}}_{\infty}=\max_{h\in[H+1]}\norm{\optV_h}_{\infty}$.



The sample complexity bound above is undesirable as $J$ appears in the dominating term, which is not the case in our established lower bound (\pref{thm:lower.bound.terminal}).
The main issue is that in $\calM_{H,c_f}$, the range of optimal value function $\norm{\optV_{\cdot}}_{\infty}=\bigo{J}$.
Thus, simply applying finite-horizon BPI algorithm with minimax rate is insufficient to adapt to $\norm{\optV}_{\infty}=\B$.
Now we present a BPI algorithm that resolves this issue and achieves minimax optimal sample complexity in the dominating term when there is no prior knowledge.
The pseudo-code is shown in \pref{alg:BPI-SSP}.

The main structure of \pref{alg:BPI-SSP} follows that of \citep{lim2012autonomous,cai2022near} for autonomous exploration.
The learning procedure is divided into rounds (\pref{line:round}) of three types: skip round, failure round, and success round.
In each round, the learner follows a behavior policy for at most $\lambda$ episodes to collect samples and estimate its empirical performance (\pref{line:behavior}).
If in the current round the number of visits to some state-action pair is doubled, then the current round is classified as a skip round (\pref{line:update}).
If the empirical performance of the current behavior policy is not close enough to its estimated performance, then the current round is classified as a failure round (\pref{line:failure}).
In both cases, we start a new round and the behavior policy is updated.
Otherwise, the current round is a success round, and the algorithm returns an $\epsilon$-optimal policy (\pref{line:success}).

To adapt to $\B$ instead of $J$, we maintain an estimator $B$ that is an upper bound of the estimated value function in the first $H/2+1$ steps (\pref{line:B}).
Intuitively, $\norm{\optV_h}_{\infty}\approx\norm{\optV}_{\infty}$ when $h\leq \frac{H}{2}+1$.
The estimator $B$ is used to determine the number of episodes needed to obtain an accurate empirical performance estimate of current behavior policy (\pref{line:lambda}), where $\lambda\lesssim \frac{B^2}{\epsilon^2}$ (see \pref{lem:N dev}).
The sample complexity of \pref{alg:BPI-SSP} is summarized in the following theorem.

\begin{theorem}
    \label{thm:algo.BPI} 
    Under \pref{assum:terminal} and assuming $\cmin>0$, for any $\epsilon\in(0,1)$ and $\delta\in(0, 1)$, \pref{alg:BPI-SSP} is $(\epsilon, \delta)$-correct with sample complexity $\tilo{\frac{\Tc\B^2SA}{\epsilon^2} + \frac{\B J^4S^2A^2}{\cmin^3\epsilon}}$.
\end{theorem}


Details are deferred to \pref{app:alg.BPI}.
The achieved sample complexity has no $J$ dependency in the dominating term, and the dominating term is minimax-optimal when there is no prior knowledge, that is, $\uT=\infty$.
On the other hand, the lower order term might be sub-optimal and the algorithm only works for strictly positive costs.
Resolving these two issues is an important open question.



\section{Conclusion}
In this work, we study the sample complexity of the SSP problem. 
We provide an almost complete characterization of the minimax sample complexity with a generative model, and initiate the study of BPI in SSP.
We derived two important negative results: 1) an $\epsilon$-optimal policy may not be learnable in SSP even with a generative model; 2) best policy identification in SSP requires an exponential number of samples in general. We complemented the study of sample complexity with lower bounds for learnable settings with and without a generative model, and matching upper bounds.
Many interesting problems remain open, such as the minimax optimal sample complexity of computing an $(\epsilon,T)$-optimal policy when $\min\{\Tc, \uT\}<\infty$, and the minimax optimal sample complexity of BPI under \pref{assum:terminal}. Furthermore, an important direction is to study BPI under weaker conditions than \pref{assum:terminal} (e.g., communicating SSP). We believe that similar results can be obtained, with a more complicated analysis, when a reset action to the initial state is available in every state, a common assumption in the literature.



\bibliography{ref}

\begin{thebibliography}{47}
\providecommand{\natexlab}[1]{#1}
\providecommand{\url}[1]{\texttt{#1}}
\expandafter\ifx\csname urlstyle\endcsname\relax
  \providecommand{\doi}[1]{doi: #1}\else
  \providecommand{\doi}{doi: \begingroup \urlstyle{rm}\Url}\fi

\bibitem[Altman(1999)]{altman1999constrained}
Eitan Altman.
\newblock \emph{Constrained Markov decision processes}, volume~7.
\newblock CRC Press, 1999.

\bibitem[Azar et~al.(2017)Azar, Osband, and Munos]{azar2017minimax}
Mohammad~Gheshlaghi Azar, Ian Osband, and R{\'e}mi Munos.
\newblock Minimax regret bounds for reinforcement learning.
\newblock In \emph{International Conference on Machine Learning}, pages
  263--272. PMLR, 2017.

\bibitem[Bertsekas(2012)]{bertsekas2012dynamicvol2}
Dimitri Bertsekas.
\newblock \emph{Dynamic programming and optimal control: Volume II}, volume~2.
\newblock Athena scientific, 2012.

\bibitem[Bertsekas and Tsitsiklis(1991)]{bertsekas1991analysis}
Dimitri~P Bertsekas and John~N Tsitsiklis.
\newblock An analysis of stochastic shortest path problems.
\newblock \emph{Mathematics of Operations Research}, 16\penalty0 (3):\penalty0
  580--595, 1991.

\bibitem[Bertsekas and Yu(2013)]{bertsekas2013stochastic}
Dimitri~P Bertsekas and Huizhen Yu.
\newblock Stochastic shortest path problems under weak conditions.
\newblock \emph{Lab. for Information and Decision Systems Report LIDS-P-2909,
  MIT}, 2013.

\bibitem[Cai et~al.(2022)Cai, Ma, and Du]{cai2022near}
Haoyuan Cai, Tengyu Ma, and Simon Du.
\newblock Near-optimal algorithms for autonomous exploration and multi-goal
  stochastic shortest path.
\newblock \emph{arXiv preprint arXiv:2205.10729}, 2022.

\bibitem[Chen and Luo(2021)]{chen2021finding}
Liyu Chen and Haipeng Luo.
\newblock Finding the stochastic shortest path with low regret: The adversarial
  cost and unknown transition case.
\newblock In \emph{International Conference on Machine Learning}, 2021.

\bibitem[Chen and Luo(2022)]{chen2022near}
Liyu Chen and Haipeng Luo.
\newblock Near-optimal goal-oriented reinforcement learning in non-stationary
  environments.
\newblock \emph{arXiv preprint arXiv:2205.13044}, 2022.

\bibitem[Chen et~al.(2021{\natexlab{a}})Chen, Jafarnia-Jahromi, Jain, and
  Luo]{chen2021implicit}
Liyu Chen, Mehdi Jafarnia-Jahromi, Rahul Jain, and Haipeng Luo.
\newblock Implicit finite-horizon approximation and efficient optimal
  algorithms for stochastic shortest path.
\newblock \emph{Advances in Neural Information Processing Systems},
  2021{\natexlab{a}}.

\bibitem[Chen et~al.(2021{\natexlab{b}})Chen, Jain, and Luo]{chen2021improved}
Liyu Chen, Rahul Jain, and Haipeng Luo.
\newblock Improved no-regret algorithms for stochastic shortest path with
  linear mdp.
\newblock \emph{arXiv preprint arXiv:2112.09859}, 2021{\natexlab{b}}.

\bibitem[Chen et~al.(2021{\natexlab{c}})Chen, Luo, and Wei]{chen2021minimax}
Liyu Chen, Haipeng Luo, and Chen-Yu Wei.
\newblock Minimax regret for stochastic shortest path with adversarial costs
  and known transition.
\newblock In \emph{Conference on Learning Theory}, pages 1180--1215. PMLR,
  2021{\natexlab{c}}.

\bibitem[Chen et~al.(2022)Chen, Luo, and Rosenberg]{chen2022policy}
Liyu Chen, Haipeng Luo, and Aviv Rosenberg.
\newblock Policy optimization for stochastic shortest path.
\newblock \emph{arXiv preprint arXiv:2202.03334}, 2022.

\bibitem[Cohen et~al.(2020)Cohen, Kaplan, Mansour, and
  Rosenberg]{cohen2020near}
Alon Cohen, Haim Kaplan, Yishay Mansour, and Aviv Rosenberg.
\newblock Near-optimal regret bounds for stochastic shortest path.
\newblock In \emph{Proceedings of the 37th International Conference on Machine
  Learning}, volume 119, pages 8210--8219. PMLR, 2020.

\bibitem[Cohen et~al.(2021)Cohen, Efroni, Mansour, and
  Rosenberg]{cohen2021minimax}
Alon Cohen, Yonathan Efroni, Yishay Mansour, and Aviv Rosenberg.
\newblock Minimax regret for stochastic shortest path.
\newblock \emph{Advances in Neural Information Processing Systems}, 2021.

\bibitem[Dann et~al.(2017)Dann, Lattimore, and Brunskill]{dann2017unifying}
Christoph Dann, Tor Lattimore, and Emma Brunskill.
\newblock Unifying pac and regret: Uniform pac bounds for episodic
  reinforcement learning.
\newblock \emph{Advances in Neural Information Processing Systems}, 30, 2017.

\bibitem[Dann et~al.(2019)Dann, Li, Wei, and Brunskill]{dann2019policy}
Christoph Dann, Lihong Li, Wei Wei, and Emma Brunskill.
\newblock Policy certificates: Towards accountable reinforcement learning.
\newblock In \emph{International Conference on Machine Learning}, pages
  1507--1516. PMLR, 2019.

\bibitem[Jafarnia-Jahromi et~al.(2021)Jafarnia-Jahromi, Chen, Jain, and
  Luo]{jafarnia2021online}
Mehdi Jafarnia-Jahromi, Liyu Chen, Rahul Jain, and Haipeng Luo.
\newblock Online learning for stochastic shortest path model via posterior
  sampling.
\newblock \emph{arXiv preprint arXiv:2106.05335}, 2021.

\bibitem[Jin et~al.(2018)Jin, Allen-Zhu, Bubeck, and Jordan]{jin2018q}
Chi Jin, Zeyuan Allen-Zhu, Sebastien Bubeck, and Michael~I Jordan.
\newblock Is {Q}-learning provably efficient?
\newblock In \emph{Advances in Neural Information Processing Systems}, pages
  4863--4873, 2018.

\bibitem[Jin et~al.(2020)Jin, Krishnamurthy, Simchowitz, and Yu]{jin2020reward}
Chi Jin, Akshay Krishnamurthy, Max Simchowitz, and Tiancheng Yu.
\newblock Reward-free exploration for reinforcement learning.
\newblock In \emph{International Conference on Machine Learning}, pages
  4870--4879. PMLR, 2020.

\bibitem[Kearns and Singh(1998)]{kearns1998finite}
Michael Kearns and Satinder Singh.
\newblock Finite-sample convergence rates for q-learning and indirect
  algorithms.
\newblock \emph{Advances in neural information processing systems}, 11, 1998.

\bibitem[Li et~al.(2020)Li, Wei, Chi, Gu, and Chen]{li2020breaking}
Gen Li, Yuting Wei, Yuejie Chi, Yuantao Gu, and Yuxin Chen.
\newblock Breaking the sample size barrier in model-based reinforcement
  learning with a generative model.
\newblock \emph{Advances in neural information processing systems},
  33:\penalty0 12861--12872, 2020.

\bibitem[Lim and Auer(2012)]{lim2012autonomous}
Shiau~Hong Lim and Peter Auer.
\newblock Autonomous exploration for navigating in {MDP}s.
\newblock In \emph{Conference on Learning Theory}, pages 40--1. JMLR Workshop
  and Conference Proceedings, 2012.

\bibitem[Mannor and Tsitsiklis(2004)]{mannor2004sample}
Shie Mannor and John~N Tsitsiklis.
\newblock The sample complexity of exploration in the multi-armed bandit
  problem.
\newblock \emph{Journal of Machine Learning Research}, 5\penalty0
  (Jun):\penalty0 623--648, 2004.

\bibitem[M{\'e}nard et~al.(2021)M{\'e}nard, Domingues, Jonsson, Kaufmann,
  Leurent, and Valko]{menard2021fast}
Pierre M{\'e}nard, Omar~Darwiche Domingues, Anders Jonsson, Emilie Kaufmann,
  Edouard Leurent, and Michal Valko.
\newblock Fast active learning for pure exploration in reinforcement learning.
\newblock In \emph{International Conference on Machine Learning}, pages
  7599--7608. PMLR, 2021.

\bibitem[Min et~al.(2021)Min, He, Wang, and Gu]{min2021learning}
Yifei Min, Jiafan He, Tianhao Wang, and Quanquan Gu.
\newblock Learning stochastic shortest path with linear function approximation.
\newblock \emph{arXiv preprint arXiv:2110.12727}, 2021.

\bibitem[Rashidinejad et~al.(2021)Rashidinejad, Zhu, Ma, Jiao, and
  Russell]{rashidinejad2021bridging}
Paria Rashidinejad, Banghua Zhu, Cong Ma, Jiantao Jiao, and Stuart Russell.
\newblock Bridging offline reinforcement learning and imitation learning: A
  tale of pessimism.
\newblock \emph{Advances in Neural Information Processing Systems},
  34:\penalty0 11702--11716, 2021.

\bibitem[Ren et~al.(2021)Ren, Li, Dai, Du, and Sanghavi]{ren2021nearly}
Tongzheng Ren, Jialian Li, Bo~Dai, Simon~S Du, and Sujay Sanghavi.
\newblock Nearly horizon-free offline reinforcement learning.
\newblock \emph{Advances in neural information processing systems},
  34:\penalty0 15621--15634, 2021.

\bibitem[Rosenberg and Mansour(2020)]{rosenberg2020adversarial}
Aviv Rosenberg and Yishay Mansour.
\newblock Stochastic shortest path with adversarially changing costs.
\newblock \emph{arXiv preprint arXiv:2006.11561}, 2020.

\bibitem[Sidford et~al.(2018{\natexlab{a}})Sidford, Wang, Wu, Yang, and
  Ye]{sidford2018near}
Aaron Sidford, Mengdi Wang, Xian Wu, Lin Yang, and Yinyu Ye.
\newblock Near-optimal time and sample complexities for solving markov decision
  processes with a generative model.
\newblock \emph{Advances in Neural Information Processing Systems}, 31,
  2018{\natexlab{a}}.

\bibitem[Sidford et~al.(2018{\natexlab{b}})Sidford, Wang, Wu, and
  Ye]{sidford2018variance}
Aaron Sidford, Mengdi Wang, Xian Wu, and Yinyu Ye.
\newblock Variance reduced value iteration and faster algorithms for solving
  markov decision processes.
\newblock In \emph{Proceedings of the Twenty-Ninth Annual ACM-SIAM Symposium on
  Discrete Algorithms}, pages 770--787. SIAM, 2018{\natexlab{b}}.

\bibitem[Strens(2000)]{strens2000bayesian}
Malcolm Strens.
\newblock A bayesian framework for reinforcement learning.
\newblock In \emph{ICML}, volume 2000, pages 943--950, 2000.

\bibitem[Tarbouriech et~al.(2020{\natexlab{a}})Tarbouriech, Garcelon, Valko,
  Pirotta, and Lazaric]{tarbouriech2020no}
Jean Tarbouriech, Evrard Garcelon, Michal Valko, Matteo Pirotta, and Alessandro
  Lazaric.
\newblock No-regret exploration in goal-oriented reinforcement learning.
\newblock In \emph{International Conference on Machine Learning}, pages
  9428--9437. PMLR, 2020{\natexlab{a}}.

\bibitem[Tarbouriech et~al.(2020{\natexlab{b}})Tarbouriech, Pirotta, Valko, and
  Lazaric]{tarbouriech2020improved}
Jean Tarbouriech, Matteo Pirotta, Michal Valko, and Alessandro Lazaric.
\newblock Improved sample complexity for incremental autonomous exploration in
  {MDP}s.
\newblock In \emph{Advances in Neural Information Processing Systems},
  volume~33, pages 11273--11284. Curran Associates, Inc., 2020{\natexlab{b}}.

\bibitem[Tarbouriech et~al.(2021{\natexlab{a}})Tarbouriech, Pirotta, Valko, and
  Lazaric]{tarbouriech2021provably}
Jean Tarbouriech, Matteo Pirotta, Michal Valko, and Alessandro Lazaric.
\newblock A provably efficient sample collection strategy for reinforcement
  learning.
\newblock \emph{Advances in Neural Information Processing Systems},
  34:\penalty0 7611--7624, 2021{\natexlab{a}}.

\bibitem[Tarbouriech et~al.(2021{\natexlab{b}})Tarbouriech, Pirotta, Valko, and
  Lazaric]{tarbouriech2021sample}
Jean Tarbouriech, Matteo Pirotta, Michal Valko, and Alessandro Lazaric.
\newblock Sample complexity bounds for stochastic shortest path with a
  generative model.
\newblock In \emph{Algorithmic Learning Theory}, pages 1157--1178. PMLR,
  2021{\natexlab{b}}.

\bibitem[Tarbouriech et~al.(2021{\natexlab{c}})Tarbouriech, Zhou, Du, Pirotta,
  Valko, and Lazaric]{tarbouriech2021stochastic}
Jean Tarbouriech, Runlong Zhou, Simon~S Du, Matteo Pirotta, Michal Valko, and
  Alessandro Lazaric.
\newblock Stochastic shortest path: Minimax, parameter-free and towards
  horizon-free regret.
\newblock \emph{Advances in Neural Information Processing Systems},
  2021{\natexlab{c}}.

\bibitem[Tarbouriech et~al.(2022)Tarbouriech, Domingues, M{\'e}nard, Pirotta,
  Valko, and Lazaric]{tarbouriech2022adaptive}
Jean Tarbouriech, Omar~Darwiche Domingues, Pierre M{\'e}nard, Matteo Pirotta,
  Michal Valko, and Alessandro Lazaric.
\newblock Adaptive multi-goal exploration.
\newblock In \emph{International Conference on Artificial Intelligence and
  Statistics}, pages 7349--7383. PMLR, 2022.

\bibitem[Vial et~al.(2021)Vial, Parulekar, Shakkottai, and
  Srikant]{vial2021regret}
Daniel Vial, Advait Parulekar, Sanjay Shakkottai, and R~Srikant.
\newblock Regret bounds for stochastic shortest path problems with linear
  function approximation.
\newblock \emph{arXiv preprint arXiv:2105.01593}, 2021.

\bibitem[Wang(2017)]{wang2017randomized}
Mengdi Wang.
\newblock Randomized linear programming solves the discounted markov decision
  problem in nearly-linear (sometimes sublinear) running time.
\newblock \emph{arXiv preprint arXiv:1704.01869}, 2017.

\bibitem[Xie et~al.(2021)Xie, Jiang, Wang, Xiong, and Bai]{xie2021policy}
Tengyang Xie, Nan Jiang, Huan Wang, Caiming Xiong, and Yu~Bai.
\newblock Policy finetuning: Bridging sample-efficient offline and online
  reinforcement learning.
\newblock \emph{Advances in neural information processing systems},
  34:\penalty0 27395--27407, 2021.

\bibitem[Yin and Wang(2021)]{yin2021towards}
Ming Yin and Yu-Xiang Wang.
\newblock Towards instance-optimal offline reinforcement learning with
  pessimism.
\newblock \emph{Advances in neural information processing systems},
  34:\penalty0 4065--4078, 2021.

\bibitem[Yin et~al.(2021)Yin, Bai, and Wang]{yin2021near}
Ming Yin, Yu~Bai, and Yu-Xiang Wang.
\newblock Near-optimal provable uniform convergence in offline policy
  evaluation for reinforcement learning.
\newblock In \emph{International Conference on Artificial Intelligence and
  Statistics}, pages 1567--1575. PMLR, 2021.

\bibitem[Yin et~al.(2022)Yin, Chen, Wang, and Wang]{yin2022offline}
Ming Yin, Wenjing Chen, Mengdi Wang, and Yu-Xiang Wang.
\newblock Offline stochastic shortest path: Learning, evaluation and towards
  optimality.
\newblock In \emph{The 38th Conference on Uncertainty in Artificial
  Intelligence}, 2022.

\bibitem[Zhang et~al.(2020{\natexlab{a}})Zhang, Du, and Ji]{zhang2020nearly}
Zihan Zhang, Simon~S Du, and Xiangyang Ji.
\newblock Nearly minimax optimal reward-free reinforcement learning.
\newblock \emph{arXiv preprint arXiv:2010.05901}, 2020{\natexlab{a}}.

\bibitem[Zhang et~al.(2020{\natexlab{b}})Zhang, Ji, and
  Du]{zhang2020reinforcement}
Zihan Zhang, Xiangyang Ji, and Simon~S Du.
\newblock Is reinforcement learning more difficult than bandits? a near-optimal
  algorithm escaping the curse of horizon.
\newblock In \emph{Conference On Learning Theory}, 2020{\natexlab{b}}.

\bibitem[Zhang et~al.(2022)Zhang, Ji, and Du]{zhang2022horizon}
Zihan Zhang, Xiangyang Ji, and Simon Du.
\newblock Horizon-free reinforcement learning in polynomial time: the power of
  stationary policies.
\newblock In \emph{Conference on Learning Theory}, pages 3858--3904. PMLR,
  2022.

\bibitem[Zhao et~al.(2022)Zhao, Li, and Zhou]{zhao2022dynamic}
Peng Zhao, Long-Fei Li, and Zhi-Hua Zhou.
\newblock Dynamic regret of online markov decision processes.
\newblock In \emph{International Conference on Machine Learning}, pages
  26865--26894. PMLR, 2022.

\end{thebibliography}

\appendix

\renewcommand\ptctitle{Contents of Appendix}
\part{}
\parttoc
\newpage



\section{Related Work}\label{app:related}
Sample complexity (with or without a generative model) is a well-studied topic in finite-horizon MDPs~\citep{dann2017unifying,sidford2018near,sidford2018variance,dann2019policy} and discounted MDPs~\citep{kearns1998finite,sidford2018near,sidford2018variance,wang2017randomized,li2020breaking}.
Apart from computing $\epsilon$-optimal policy for a given cost function, researchers also study obtaining $\epsilon$-optimal policies for an arbitrary sequence of cost functions after interacting with the environment, known as reward-free exploration~\citep{jin2020reward,menard2021fast,zhang2020nearly}.

Instead of reaching a single goal state, another line of research considers exploration problems of discovering reachable states~\citep{lim2012autonomous,tarbouriech2020improved,tarbouriech2021provably,tarbouriech2022adaptive,cai2022near}.
Sample complexity of SSP is a building block for solving these problems, and existing results only consider strictly positive costs, that is, $\cmin>0$.

Another active research area is learning $\epsilon$-optimal policy purely from an offline dataset, known as offline reinforcement learning~\citep{ren2021nearly,rashidinejad2021bridging,xie2021policy,yin2021towards,yin2021near,yin2022offline}.
Offline SSP has been recently studied by~\citet{yin2022offline} where they provide a minimax optimal offline algorithm. While both offline RL and our setting aim to recover an $\epsilon$-optimal policy, there are important differences. In offline SSP~\citep{yin2022offline} , the samples are collected by a behavior policy with bounded coverage (i.e., maximum ratio between the state-action distribution of the optimal and behavior policy) while in sample complexity the algorithm is responsible of deciding the sample collection strategy. Furthermore, the analysis in~\citep{yin2022offline} is limited to positive costs (i.e., $\cmin>0$) and their sample complexity is measured in terms of number of trajectories and coverage bound. These terms hide both the dependence in terms of action space and, most importantly, the horizon. Our analysis provides a much more comprehensive understanding of sample complexity in SSP.

\section{Preliminaries}\label{app:pre}

\paragraph{Extra Notations} For any distribution $P\in\Delta_{\calS_+}$ and function $V\in\fR^{\calS_+}$, define $PV=\E_{S\sim P}[V(S)]$ and $\fV(P, V)=\var_{S\sim P}[V(S)]$ as the expectation and the variance of $V(S)$ with $S$ sampled from $P$ respectively.


\section{Omitted Details in \pref{sec:lower.bounds}}

In this section we provide omitted proofs and discussions in \pref{sec:lower.bounds}.

\subsection{\pfref{thm:lower.bound.eps} and \pref{thm:lower.bound.eps.T}}
\label{app:lower.bound.eps.T}

It suffices to prove \pref{thm:lower.bound.eps.T} and the second statement in \pref{thm:lower.bound.eps}, since \pref{thm:lower.bound.eps.T} subsumes the first statement of \pref{thm:lower.bound.eps}.
We decompose the proof into two cases: 1) $\min\{\Tc,\uT\}<\infty$, and 2) $\min\{\Tc,\uT\}=\infty$, and we prove each case in a separate theorem.

\subsubsection{Lower Bound for $\min\cbr{\Tc, \uT} < \infty$}

In case there is a finite upper bound on the hitting time of optimal policy, we construct a hard instance adapted from \citep{mannor2004sample}.



\begin{proof}[of \pref{thm:lower.bound.eps.T}]
Without loss of generality, we assume $S=\frac{A^l-1}{A-1}$ for some $l \geq 0$.
It is clear that $l\leq \log_AS +1$.
We construct an MDP $\calM_0$ of full $A$-ary tree structure: the root node is $s_0$, each action at a non-leaf node transits to one of its children with cost $\cmin$, and we denote the set of leaf nodes by $\calS'$.
Since $A\geq 3$, we have $|\calS'|\geq \frac{S}{2}$.
The action space is $\calA=[A]$, and we partition the state-action pairs in $\calS'$ into two parts: $\Lambda_0=\calS'\times[1]$ and $\Lambda=\calS'\times\{2,\ldots,A\}=\{(s_1,a_1),\ldots,(s_N,a_N)\}$, where $N=|\calS'|(A-1)$ (note that here we index state-action pair instead of state, so $s_i$, $s_j$ with $i\neq j$ may refer to the same state in $\calS'$).
Now define $T_1=\min\{\uT/2, B/\cmin\}$.
The cost function satisfies $c(s, 1)=\frac{B}{T_0}$ for $s\in\calS'$, and $c(s_i, a_i)=\frac{B}{T_1}$ for $i\in[N]$.
The transition function satisfies $P(g|s, 1)=\frac{1}{T_0}+\frac{T_1\alpha}{2T_0}$, $P(s|s, 1)=1-P(g|s, 1)$ for $s\in\calS'$, and $P(g|s_i, a_i)=\frac{1}{T_1}$, $P(s_i|s_i, a_i)=1-P(g|s_i, a_i)$ for $i\in[N]$, where $\alpha=\frac{32\epsilon}{T_1B}$.

Now we consider a class of alternative MDPs $\{\calM_i\}_{i=1}^N$.
The only difference between $\calM_0$ and $\calM_i$ is that the transition of $\calM_i$ at $(s_i, a_i)$ satisfies $P(g|s_i, a_i)=\frac{1}{T_1}+\alpha$ and $P(s_i|s_i, a_i)=1-P(g|s_i, a_i)$, that is, $(s_i, a_i)$ is a ``good'' state-action pair at $\calM_i$.
Denote by $\B^i$ and $\T^i$ the value of $\B$ and $\T$ in $\calM_i$ respectively.
For $i\in\{0,\ldots,N\}$, we have $\frac{B}{2} \leq \B^i \leq \cmin\cdot l + B \leq 2B$ by $\alpha\leq \frac{1}{2T_1}$ and $B\geq 2$; $\frac{T_0}{2} \leq \T^i \leq l + T_1 \leq \uT$; and $\cmin$ is indeed the minimum cost by $\cmin\leq \frac{B}{T_1}$.
It is not hard to see that at $s_0$, the optimal behavior in $\calM_0$ is to reach any leaf node and then take action $1$ until $g$ is reached; while in $\calM_i$ for $i\in\{1,\ldots,N\}$, the optimal behavior is to reach $s_i$ and then take $(s_i, a_i)$ until $g$ is reached.
Thus, $\T^0=\Theta(T_0)$ and $\T^i=\Theta(T_1)$ for $i\in[N]$.

Without loss of generality, we consider learning algorithms that output a deterministic policy, which can be represented by $\hatv\in\calS'\times\calA$ the unique state-action pair in $\calS'$ reachable by following the output policy starting from $s_0$.
Define event $\calE_1=\{ \hatv \in\Lambda_0 \}$.
Below we fix a $z\in[N]$.
Let $\hatT$ be the number of times the learner samples $(s_z, a_z)$, and $K_t$ be the number of times the agent observes $(s_z, a_z, g)$ among the first $t$ samples of $(s_z, a_z)$.
We introduce event
$$\calE_2=\cbr{ \max_{1\leq t\leq \tstar}|pt - K_t| \leq \varepsilon },$$ where $\varepsilon=\sqrt{2p(1-p)\tstar\ln\frac{d}{\theta}}$, $p=\frac{1}{T_1}$, $d=e^4$, $\theta=\exp(-d'\alpha^2\tstar/(p(1-p)))<1$ for some $\tstar>0$ to be specified later, and $d'=128$.
Also define events $\calE_3=\{\hatT \leq \tstar\}$ and $\calE=\calE_1\cap\calE_2\cap\calE_3$.
For each $i\in\{0,\ldots,N\}$, we denote by $P_i$ and $\E_i$ the probability and expectation w.r.t $\calM_i$ respectively.

Below we introduce two lemmas that characterize the behavior of the learner if it gathers insufficient samples on $(s_z, a_z)$.

\begin{lemma}
    \label{lem:e23}
    If $P_0(\calE_3)\geq\frac{7}{8}$, then $P_0(\calE_2\cap\calE_3)\geq\frac{3}{4}$.
\end{lemma}
\begin{proof}
    Note that in $\calM_0$, the probability of observing $(s_z, a_z, g)$ is $p$ for each sample of $(s_z, a_z)$.
    Thus, $pt-K_t$ is a sum of i.i.d random variables, and the variance of $pt-K_t$ for $t=\tstar$ is $\tstar p(1-p)$.
    By Kolmogorov's inequality, we have
    \begin{align*}
        P_0(\calE_2) = P_0\rbr{\max_{1\leq t\leq \tstar}|pt-K_t| \leq \varepsilon } \geq 1 - \frac{\tstar p(1-p)}{2p(1-p)\tstar\ln\frac{d}{\theta}} = 1 - \frac{1}{2\ln\frac{d}{\theta}} \geq \frac{7}{8}.
    \end{align*}
    Thus, $P_0(\calE_2\cap\calE_3)= P_0(\calE_2) + P_0(\calE_3) - P_0(\calE_2\cup\calE_3)\geq \frac{3}{4}$.
\end{proof}

\begin{lemma}
    \label{lem:prob of wrong}
     If $P_0(\calE_3)\geq\frac{7}{8}$ and $P_0(\calE_1)\geq 1 - \frac{\theta}{2d}$, then $P_z(\calE_1)\geq\frac{\theta}{2d}$.
\end{lemma}
\begin{proof}
    The range of $\epsilon$ ensures that $\alpha\leq\frac{p}{2}\leq\frac{1-p}{2}$.
    By the assumptions of this lemma and \pref{lem:e23}, we have $P_0(\calE_1)\geq 1 - \frac{1}{2d}\geq\frac{7}{8}$ by $d\leq\frac{1}{16}$, and thus $P_0(\calE)\geq \frac{1}{2}$.
    Let $W$ be the interaction history of the learner and the generative model, and $L_j(w)=P_j(W=w)$ for $j\in\{0,\ldots,N\}$.
    Note that the next-state distribution is identical in $\calM_0$ and $\calM_z$ unless $(s_z, a_z)$ is sampled.
    Define $K=K_{\hatT}$.
    We have
    \begin{align*}
        \frac{L_z(W)}{L_0(W)} &= \frac{(p+\alpha)^K(1-p-\alpha)^{\hatT-K}}{p^K(1-p)^{\hatT-K}} = \rbr{1 + \frac{\alpha}{p}}^K\rbr{1 - \frac{\alpha}{1-p}}^{\hatT-K}\\
        &= \rbr{1 + \frac{\alpha}{p}}^K\rbr{1 - \frac{\alpha}{1-p}}^{K(\frac{1}{p}-1)}\rbr{1 - \frac{\alpha}{1-p}}^{\hatT - \frac{K}{p}}.
    \end{align*}
    By $1-u\geq e^{-u-u^2}$ for $u\in[0,\frac{1}{2}]$, $e^{-u}\geq 1-u$, and $\alpha\leq\frac{1-p}{2}$, we have
    \begin{align*}
        \rbr{1 - \frac{\alpha}{1-p}}^{\frac{1}{p}-1} &\geq \exp\rbr{ \frac{1-p}{p}\rbr{ -\frac{\alpha}{1-p} - \rbr{\frac{\alpha}{1-p}}^2 } }=\exp\rbr{-\frac{\alpha}{p}}\exp\rbr{-\frac{\alpha^2}{p(1-p)}}\\
        &\geq \rbr{1 - \frac{\alpha}{p}}\rbr{1 - \frac{\alpha^2}{p(1-p)}}.
    \end{align*}
    Therefore, conditioned on $\calE$, we have
    \begin{align*}
        \frac{L_z(W)}{L_0(W)}\Ind_{\calE} &\geq \rbr{1 - \frac{\alpha^2}{p^2}}^K\rbr{1 - \frac{\alpha^2}{p(1-p)}}^K\rbr{1 - \frac{\alpha}{1-p}}^{\hatT - \frac{K}{p}}\Ind_{\calE}\\
        &\geq \rbr{1 - \frac{\alpha^2}{p^2}}^{p\hatT + \varepsilon}\rbr{1 - \frac{\alpha^2}{p(1-p)}}^{p\hatT + \varepsilon}\rbr{1-\frac{\alpha}{1-p}}^{\frac{\varepsilon}{p}}\Ind_{\calE},
    \end{align*}
    where in the last inequality we apply $|p\hatT - K|\leq \varepsilon$ by $\calE_2$ and $\calE_3$.
    Then by $1-u\geq e^{-2u}$ for $u\in[0,\frac{1}{2}]$ and $\alpha\leq\frac{p}{2}\leq\frac{1-p}{2}$, we have
    \begin{align*}
        \frac{L_z(W)}{L_0(W)}\Ind_{\calE} &\geq \exp\rbr{-2\rbr{\frac{\alpha^2}{p^2}(p\hatT+\varepsilon) + \frac{\alpha^2}{p(1-p)}(p\hatT+\varepsilon) + \frac{\alpha}{1-p}\frac{\varepsilon}{p} }} \\
        &\geq \exp\rbr{-2\rbr{\frac{\hatT\alpha^2}{p} + \frac{\hatT\alpha^2}{1-p} + \frac{\alpha^2\varepsilon}{p^2} + \frac{\alpha^2\varepsilon}{p(1-p)} + \frac{\alpha\varepsilon}{p(1-p)} } }\Ind_{\calE}\\
        &\geq \exp\rbr{-2\rbr{ \frac{1}{d'}\ln\frac{1}{\theta} + \frac{3\alpha\varepsilon}{p(1-p)} } }\Ind_{\calE} \tag{$\hatT \leq \tstar$, $\alpha^2\tstar=\frac{p(1-p)}{d'}\ln\frac{1}{\theta}$, and $\alpha<p$}\\
        &\overset{\text{(i)}}{\geq} \exp\rbr{-2\rbr{\frac{1}{d'}\ln\frac{1}{\theta} + 3\sqrt{\frac{2}{d'}}\ln\frac{d}{\theta}}}\Ind_{\calE} \geq \rbr{\frac{\theta}{d}}^{2\rbr{1/d'+3\sqrt{2/d'}}}\Ind_{\calE}\geq \frac{\theta}{d}\Ind_{\calE},
    \end{align*}
    where in (i) we apply the definition of $\varepsilon$ and $\alpha^2\tstar=\frac{p(1-p)}{d'}\ln\frac{1}{\theta}$ to have
    \begin{align*}
    		\frac{\alpha\varepsilon}{p(1-p)} = \sqrt{\frac{2\tstar\alpha^2}{p(1-p)}\ln\frac{d}{\theta}} = \sqrt{\frac{2}{d'}\ln\frac{1}{\theta}\ln\frac{d}{\theta}} \leq \sqrt{\frac{2}{d'}}\ln\frac{d}{\theta}.
    \end{align*}
    Then we have
    \begin{align*}
        P_z(\calE_1) \geq P_z(\calE) = \E_z[\Ind_{\calE}(W)] = \E_0\sbr{\frac{L_z(W)}{L_0(W)}\Ind_{\calE}(W)} \geq \frac{\theta}{d}P_0(\calE)\geq \frac{\theta}{2d}.
    \end{align*}
    This completes the proof.
\end{proof}
Now for any $\delta\in(0, \frac{1}{2d})$, let $\tstar=\frac{B^2(T_1-1)}{32^2d'\epsilon^2}\ln\frac{1}{2d\delta}$.
This gives $\frac{\theta}{2d}=\delta$.

\begin{lemma}
    \label{lem:bound t}
    An $(\epsilon, \delta)$-correct algorithm with $\epsilon\in(0,\frac{1}{32})$ and $\delta\in(0,\frac{1}{2e^4})$ must have $P_0(\calE_3)<\frac{7}{8}$.
\end{lemma}
\begin{proof}
    Denote by $\pi_0$ a deterministic policy such that when following $\pi_0$ starting from $s_0$, it reaches some state in $\calS'$ and then takes action $1$ until reaching the goal state $g$.
    For any $j\in[N]$, denote by $\pi_j$ a deterministic policy such that when following $\pi_j$ starting from $s_0$, it reaches $s_j$ and then takes action $a_j$ until reaching the goal state $g$.
    It is not hard to see that $\pi_j$ is an optimal policy in $\calM_j$ for $j\in\{0,\ldots,N\}$ starting from $s_0$.
    Denote by $V^j_i$ the value function of $\pi_j$ in $\calM_i$ and $\optV_i$ the optimal value function of $\calM_i$.
    By the choice of $\alpha$, we have $V^j_0(s_0) - \optV_0(s_0) = B - \frac{B}{1 + T_1\alpha/2} > \epsilon$ for any $j\in[N]$.
    Thus, all $\epsilon$-optimal deterministic policies in $\calM_0$ have the same behavior as $\pi_0$ starting from $s_0$.
    On the other hand, $V^0_z(s_0)-\optV_z(s_0) = B(\frac{1}{1+T_1\alpha/2} - \frac{1}{1+T_1\alpha}) > \epsilon$.
    Thus, all $\epsilon$-optimal deterministic policies in $\calM_z$ have the same behavior as $\pi_z$ starting from $s_0$.
    Therefore, an $(\epsilon, \delta)$-correct algorithm should guarantee $P_0(\calE_1)\geq 1-\delta$ and $P_z(\calE_1)<\delta$.
    When $P_0(\calE_3)\geq\frac{7}{8}$,
    this leads to a contradiction by \pref{lem:prob of wrong} and the choice of $\tstar$.
    Thus, $P_0(\calE_3)<\frac{7}{8}$.
\end{proof}

We are now ready to prove the main statement of \pref{thm:lower.bound.eps.T}.
Note that in $\calM_0$, an $(\epsilon, \delta)$-correct algorithm should guarantee $P_0(\hatT_z\leq \tstar)<\frac{7}{8}$ for any $z\in[N]$ by \pref{lem:bound t}, where $\hatT_z$ is the number of times the learner samples $(s_z, a_z)$.
Define $\calN=\sum_z\Ind\{\hatT_z\leq\tstar\}$.
Clearly, we have $\calN\leq N$ and $\E_0[\calN] < \frac{7N}{8}$.
Moreover, $P_0(\calN\geq \frac{8N}{9}) \leq \frac{63}{64}$ by Markov's inequality.
This implies that with probability at least $\frac{1}{64}>\frac{1}{2e^4}$, we have $|\{z\in[N]: \hatT_z>\tstar\}|>\frac{N}{9}$ and thus the total number of samples used in $\calM_0$ is at least $\frac{N\tstar}{9}$.
To conclude, there is no $(\epsilon,\delta)$-correct algorithm with $\epsilon\in(0,\frac{1}{32})$, $\delta\in(0,\frac{1}{2e^4})$, and sample complexity $\frac{N\tstar}{9}=\lowo{\frac{NB^2T_1}{\epsilon^2}\ln\frac{1}{\delta}}=\lowo{ \min\{\frac{\B}{\cmin}, \uT\}\frac{\B^2SA}{\epsilon^2}\ln\frac{1}{\delta} }$ on $\calM_0$ by $\B=\Theta(B)$ and the definition of $T_1$.
This completes the proof.

\end{proof}

\begin{remark}
    Note that $\uT$ is both a parameter of the environment and the knowledge given to the learner.
    In fact, $\uT$ constrains the hitting time of the optimal policy in all the possible alternative MDPs $\{\calM_j\}_{j\in[N]}$, which affects the final lower bound.
    Also note that the lower bound holds even if the learner has access to an upper bound of $\B$ (which is $2B$ in the proof above).
\end{remark}

\paragraph{Why a faster rate is impossible with $\uT\geq\T$?} This result may seem unintuitive because when we have knowledge of $\uT \geq \T$, a finite-horizon reduction with horizon $\tilo{\uT}$ ensures that the estimation error shrinks at rate $\B\sqrt{\T/n}$~\citep[Figure 1]{yin2021towards}, where $n$ is the number of samples for each state-action pair.
Then it seems that it might be possible to obtain a sample complexity of order $\frac{\T\B^2SA}{\epsilon^2}$.
However, our lower bound indicates that the sample complexity should scale with $\uT$ instead of $\T$.
An intuitive explanation is that even if the estimation error shrinks with rate $\T$ in hindsight, since the learner doesn't know the exact value of $\T$, it can only set $n$ w.r.t $\uT$ so that the output policy is $\epsilon$-optimal even in the worst case of $\uT=\T$.


\subsubsection{Lower Bound for $\min\cbr{\Tc, \uT} = \infty$}
Now we show that when there is no finite upper bound on $\T$, it really takes infinite number of samples to learn in the worst scenario.

\begin{theorem}[Second statemnt of \pref{thm:lower.bound.eps}]
    \label{thm:lower.bound.infinity}
    There exist an SSP instance with $\cmin=0$, $\T=1$, and $\B=1$ in which every $(\epsilon, \delta)$-correct algorithm with $\epsilon\in(0, \frac{1}{2})$ and $\delta\in(0, \frac{1}{16})$ has infinite sample complexity.
\end{theorem}
\begin{proof}
    Consider an SSP $\calM_0$ with $\calS=\{s_0, s_1\}$ and $\calA=\{a_0,a_g\}$.
    The cost function satisfies $c(s_0, a_0)=0$, $c(s_0, a_g)=\frac{1}{2}$, and $c(s_1, a)=1$ for all $a$.
    The transition function satisfies $P(g|s_0, a_g)=1$, $P(s_0|s_0,a_0)=1$, and $P(g|s_1,a)=1$ for all $a$; see \pref{fig:fig} (b).
    Clearly $\cmin=0$, $\B=\T=1$, and $\optV(s_0)=\frac{1}{2}$ in $\calM_0$.
    Without loss of generality, we consider learning algorithm that outputs deterministic policy $\hatpi$ and define events $\calE_1=\{\hatpi(s_0)=a_0\}$ and $\calE'_1=\{\hatpi(s_0)=a_g\}$.

    If a learning algorithm is $(\epsilon,\delta)$-correct with $\delta\in(0,\frac{1}{8})$ and has sample complexity $n\in[2, \infty)$ on $\calM_0$, then consider two alternative MDPs $\calM_+$ and $\calM_-$.
    MDP $\calM_+$ is the same as $\calM_0$ except that $P(s_1|s_0, a_0)=\frac{1}{n}$ and $P(s_0|s_0,a_0)=1-\frac{1}{n}$.
    MDP $\calM_-$ is the same as $\calM_0$ except that $P(g|s_0, a_0)=\frac{1}{n}$ and $P(s_0|s_0,a_0)=1-\frac{1}{n}$.    
    Note that in $\calM_+$, the optimal proper policy takes $a_g$ at $s_0$, and $\optV(s_0)=\frac{1}{2}$; while in $\calM_-$, the optimal proper policy takes $a_0$ at $s_0$, and $\optV(s_0)=0$.
    Let $W$ be the interaction history between the learner and the generative model, and define $L_j(w)=P_j(W=w)$ for $j\in\{0,+,-\}$, where $P_j$ is the probability w.r.t $\calM_j$.
    Also let $\hatT$ be the number of times the learner samples $(s_0, a_0)$ before outputting $\hatpi$, and $\gamma(w)=\Ind\{L_0(w)>0\}$.
    Define $\calE_2=\{\hatT\leq n\}$, $\calE=\calE_1\cap\calE_2$ and $\calE'=\calE'_1\cap\calE_2$.
    For any $j\in\{+,-\}$, we have $\frac{L_j(W)}{L_0(W)}\Ind_{\calE}(W)\gamma(W)=(1-\frac{1}{n})^{\hatT}\Ind_{\calE}(W)\gamma(W)\geq(1-\frac{1}{n})^n\Ind_{\calE}(W)\gamma(W)\geq\frac{\Ind_{\calE}(W)\gamma(W)}{4}$.
    Thus,
    \begin{align*}
        P_j(\calE) = \E_j[\Ind_{\calE}(W)] \geq \E_j[\Ind_{\calE}(W)\gamma(W)] = E_0\sbr{\frac{L_j(W)}{L_0(W)}\Ind_{\calE}(W)\gamma(W)} \geq \frac{P_0(\calE)}{4}.
    \end{align*}
    By a similar arguments, we also have $P_j(\calE')\geq P_0(\calE')/4$ for $j\in\{+, -\}$.
    Now note that $P_0(\calE_2)\geq\frac{7}{8}$ by the sample complexity of the learner.
    Since $\calE\cup\calE'=\calE_2$ and $\calE\cap\calE'=\varnothing$, we have $P_0(\calE)\geq\frac{7}{16}$ or $P_0(\calE')\geq\frac{7}{16}$.
    Combining with $P_j(\calE)\geq P_0(\calE)/4$ and $P_j(\calE')\geq P_0(\calE')/4$, we have either $P_j(\calE)\geq\frac{7}{64}$ for $j\in\{+, -\}$, or $P_j(\calE')\geq\frac{7}{64}$ for $j\in\{+, -\}$.
    In the first case, in $\calM_+$, we have $V^{\hatpi}(s_0)-\optV(s_0)=1-\frac{1}{2}=\frac{1}{2}$ with probability at least $\frac{7}{64}$.
    In the second case, in $\calM_-$, we have $V^{\hatpi}(s_0)-\optV(s_0)=\frac{1}{2}-0=\frac{1}{2}$ with probability at least $\frac{7}{64}$.
    Therefore, for any $\epsilon\in(0, \frac{1}{2})$ and $\delta\in(0,\frac{1}{16})$, there is a contradiction in both cases if the learner is $(\epsilon, \delta)$-correct and has finite sample complexity on $\calM_0$.
    This completes the proof.
\end{proof}

\begin{remark}
    Note that although $\T=1$ in $\calM_0$, the key of the analysis is that $\T$ can be arbitrarily large in the alternative MDPs.
    Indeed, if we have a finite upper bound $\uT$ of $\T$, then the learning algorithm only requires finite number of samples as shown in \pref{thm:lower.bound.eps.T}.
\end{remark}

\subsection{\pfref{thm:lower.bound.T}}
\label{app:lower.bound.T}

\begin{figure}
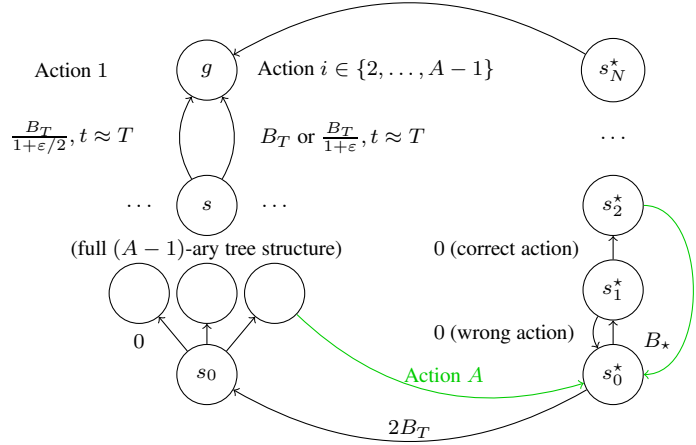

    \centering
    \tikz[font=\scriptsize, scale=0.9]{
        \begin{scope}
            \node[draw, circle, minimum width=0.8cm] (g) at (0,0) {$g$};
            \node[draw, circle, minimum width=0.8cm] (s0) at (0,-4.5) {$s_0$};
            \node[draw, circle, minimum width=0.8cm] (s1) at (-1, -3.3) {};
            \node[draw, circle, minimum width=0.8cm] (s2) at (0, -3.3) {};
            \node[draw, circle, minimum width=0.8cm] (s3) at (1, -3.3) {};
            \node at (-1.0, -4) {$0$};
            \path[->] (s0) edge[] (s1);
            \path[->] (s0) edge[] (s2);
            \path[->] (s0) edge[] (s3);
            \node at (0, -2.65) {(full $(A-1)$-ary tree structure)};
            \node at (-1,-2) {$\ldots$};
            \node[draw, circle, minimum width=0.8cm] (s) at (0,-2) {$s$};
            \node at (1,-2) {$\ldots$};
            \path[->] (s) edge[bend left] (g);
            \path[->] (s) edge[bend right] (g);
            \node at (-2, 0) {Action $1$};
            \node at (2.5, 0) {Action $i\in\{2,\ldots,A-1\}$};
            \node at (-2, -1) {$\frac{B_T}{1+\varepsilon/2}, t\approx T$};
            \node at (2, -1) {$B_T$ or $\frac{B_T}{1+\varepsilon}, t\approx T$};

            \node[draw, circle, minimum width=0.8cm] (ss0) at (6, -4.5) {$\sstar_0$};
            \path[->] (ss0) edge[bend left] (s0);
            \path[->] (s3) edge[bend right, green!80!black] (ss0);
            \node[green!80!black] at (3.5, -4.5) {Action $A$};
            \node at (3, -5.3) {$2B_T$};
            \node[draw, circle, minimum width=0.8cm] (ss1) at (6, -3.25) {$\sstar_1$};
            \path[->] (ss0) edge[] (ss1);
            \node[anchor=west] at (6.3, -4) {$\B$};
            \node[draw, circle, minimum width=0.8cm] (ss2) at (6, -2) {$\sstar_2$};
            \path[->] (ss1) edge[] (ss2);
            \path[->] (ss1) edge[bend right] (ss0);
            \path[->] (ss2) edge[bend left=90, green!80!black] (ss0);
            \node[anchor=west] at (3.2, -2.65) {$0$ (correct action)};
            \node[anchor=west] at (3.2, -3.9) {$0$ (wrong action)};
            \node at (6, -1) {$\ldots$};
            \node[draw, circle, minimum width=0.8cm] (ssn) at (6, 0) {$\sstar_N$};
            \path[->] (ssn) edge[bend right] (g);
        \end{scope}
    }
    \caption{
        Hard instance in \pref{thm:lower.bound.T}.
        Each arrow represents a possible transition of a state-action pair, and the value on the side is the expected cost of taking this state-action pair until the transition happens.
        Value $t$ represents the expected number of steps needed for the transition to happen.
    }
    \label{fig:lb T}
\end{figure}

\begin{proof}
Without loss of generality, assume that $S=\frac{2((A-1)^l-1)}{A-2}$ for some $l\geq 0$.
Consider a family of MDPs $\{\calM_{i,j}\}_{i\in\{0,\ldots,N'\},j\in[A-1]^N}$ with state space $\calS=\calS_T\cup\calS_{\star}$ where $|\calS_T|=|\calS_{\star}|=N+1$, $N'=(A-2)(A-1)^l$, and action space $\calA=[A]$.
States in $\calS_T$ forms a full $(A-1)$-ary tree on action subset $[A-1]$ as in \pref{thm:lower.bound.eps.T} with root $s_0$, $\uT=T/3$, $T_0=T/6$, $B=B_T$, and $\cmin=0$.
It is clear that $N'=|\Lambda|$ (defined in \pref{thm:lower.bound.eps.T}) in the tree formed by $\calS_T$.
The transition of $\calM_{i,j}$ in $\calS_T$ corresponds to $\calM_i$ in \pref{thm:lower.bound.eps.T}.
We denote $\calS_{\star}=\{\sstar_0,\ldots,\sstar_N\}$, and for each state in $\calS_T$, the remaining unspecified action transits to $\sstar_0$ with cost $0$.


Consider another set of MDPs $\{\calM'_i\}_{i\in\{0,\ldots,N'\}}$ with state space $\calS_T$.
The transition and cost functions of $\calM'_i$ is the same as $\calM_i$ in $\calS_T$ except that its action space is restricted to $[A-1]$.
\pref{thm:lower.bound.eps.T} implies that there exists constants $\alpha_1$, $\alpha_2$, such that any $(\epsilon',\delta')$-correct algorithm with $\epsilon'\in(0,\frac{1}{32})$, $\delta'\in(0,\frac{1}{2e^4})$ has sample complexity at least $C(\epsilon',\delta')=\frac{\alpha_1B_T^2TSA}{{\epsilon'}^2}\ln\frac{\alpha_2}{\delta'}$ on $\{\calM'_i\}_i$ (note that in \pref{thm:lower.bound.eps.T} we only show the sample complexity lower bound in $\calM'_0$, but it not hard to show a similar bound for other $\calM'_i$ following similar arguments).
Now we specify the transition and cost functions in $\calS_{\star}$ for each $\calM_{i,j}$ such that learning $\{\calM_{i,j}\}_{i,j}$ is as hard as learning $\{\calM'_i\}_i$.
At $\sstar_0$, taking any action suffers cost $1$; 
taking any action in $[A-1]$ transits to $\sstar_1$ with probability $\frac{1}{\B}$ and stays at $\sstar_0$ otherwise;
taking action $A$ transits to $s_0$ with probability $\frac{1}{2B_T}$ and stays at $\sstar_0$ otherwise.
At $\sstar_n$ for $n\in[N]$, taking any action suffers cost $0$; 
taking action $j_n$ (recall that $j\in[A-1]^N$) transits to $\sstar_{n+1}$ (define $\sstar_{N+1}=g$) with probability $p=\min\{\frac{1}{2T}, \frac{\delta}{4C(\epsilon,4\delta)}\}$ and stays at $\sstar_n$ otherwise; taking any other action in $[A-1]$ transits to $\sstar_0$ with probability $p$ and stays at $\sstar_n$ otherwise; taking action $A$ directly transits to $\sstar_0$;
see illustration in \pref{fig:lb T}.
Note that any $\calM_{i,j}$ has parameters $\B$ (transiting to $\sstar_0$ from any state and then reaching $g$ through $\calS_{\star}$) and satisfies $B_{\star,T}\in[\frac{B_T}{2}, 3B_T]$ (transiting from $\sstar_0$ to $s_0$ and then reaching $g$ through $\calS_T$).
From now on we fix the learner as an $(\epsilon,\delta, T)$-correct algorithm with sample complexity $C(\epsilon,4\delta)-1$ on $\{\calM_{i,j}\}_{i,j}$.
Define $\calE_1$ as the event that the first $C(\epsilon,4\delta)$ samples drawn by the learner from any $(\sstar_n, a)$ with $n\in[N]$ and $a\in[A-1]$ transits to $\sstar_n$, and denote by $P_{i,j}$ the probability distribution w.r.t $\calM_{i,j}$.
By $1+x\geq e^{\frac{x}{1+x}}$ for $x\geq -1$ and $e^x\geq 1+x$, we have for any $i,j$,
\begin{align*}
    P_{i,j}(\calE_1)=(1-p)^{C(\epsilon,4\delta)} \geq e^{\frac{-pC(\epsilon,4\delta)}{1-p}} \geq e^{-2pC(\epsilon,4\delta)} \geq e^{-\frac{\delta}{2}} \geq 1 - \frac{\delta}{2}.
\end{align*}
Also define $\calE_2$ as the event that the learner uses at most $C(\epsilon,4\delta)-1$ samples, and $\calE=\calE_1\cap\calE_2$.
We have $P_{i,j}(\calE_2)\geq 1-\delta$ by the sample complexity of the learner, and thus $P_{i,j}(\calE)\geq 1-\frac{3}{2}\delta$ for any $i,j$.
We first bound the expected cost of the learner in $\calS_{\star}$ conditioned on $\calE$.
Denote by $V^{\pi}_{\calM}$ the value function of policy $\pi$ in $\calM$.

\begin{lemma}
    \label{lem:Sstar}
    Given any policy distribution $\rho$, there exists $\jstar$ such that $\E_{\pi\sim\rho}[\Ind\{V^{\pi}_{\calM_{i,\jstar}}(\sstar_0) \geq 2B_T\}] \geq \frac{1}{2}$ for any $i$.
\end{lemma}
\begin{proof}
    Below we fix an $i\in[N']$.
    For any policy $\pi$ and $j\in[A-1]^N$, define $x^{\pi}_j=\prod_{n=1}^Np^{\pi}(j_n|\sstar_n)$ and $y^{\pi}=\pi(A|\sstar_0)$, where $p^{\pi}(a|\sstar_n)$ is the probability that when following policy $\pi$ starting from $\sstar_n$, the last action taken before leaving $\sstar_n$ is $a$.
    It is not hard to see that in our construction, $p^{\pi}$ is independent of $j$.
    Also denote by $V^{\pi}_j$ the value function of policy $\pi$ in $\calM_{i,j}$.
    Note that
    \begin{align*}
        V^{\pi}_j(\sstar_0) &\geq 1 + \frac{1-y^{\pi}}{\B}V^{\pi}_j(\sstar_1) + \rbr{1-\frac{y^{\pi}}{2B_T} - \frac{1-y^{\pi}}{\B}}V^{\pi}_j(\sstar_0)\\
        &= 1 + \frac{(1-y^{\pi})(1-x^{\pi}_j)}{\B}V^{\pi}_j(\sstar_0) + \rbr{1-\frac{y^{\pi}}{2B_T} - \frac{1-y^{\pi}}{\B}}V^{\pi}_j(\sstar_0)\\
        &= 1 + \rbr{1 - \frac{y^{\pi}}{2B_T} - \frac{(1-y^{\pi})x^{\pi}_j}{\B} }V^{\pi}_j(\sstar_0).
    \end{align*}
    Reorganizing terms gives $V^{\pi}_j(\sstar_0)\geq \frac{1}{y^{\pi}/(2B_T) + (1-y^{\pi})x^{\pi}_j/\B}$.
    Now if $V^{\pi}_j(\sstar_0) < 2B_T$, then we have $y^{\pi} + (1-y^{\pi})\frac{2B_Tx^{\pi}_j}{\B} > 1$, which gives $x^{\pi}_j>\frac{\B}{2B_T}$.
    Let $\calX^{\pi}$ be the set of $j\in[A-1]^N$ such that $V^{\pi}_j(\sstar_0)<2B_T$.
    By $\frac{\B}{2B_T}|\calX^{\pi}|\leq\sum_jx^{\pi}_j\leq 1$, we have $|\calX^{\pi}| \leq \frac{2B_T}{\B}$.
    Define $z^{\pi}(j)=\Ind\{j\in\calX^{\pi}\}$.
    We have $\sum_jz^{\pi}(j)=|\calX^{\pi}|\leq\frac{2B_T}{\B}$ for any $\pi$, and thus $\sum_j\int_{\pi} z^{\pi}(j)\rho(\pi)d\pi \leq \frac{2B_T}{\B}$.
    Therefore, there exists $\jstar$ such that $\int_{\pi}z^{\pi}(\jstar)\rho(\pi)d\pi\leq\frac{2B_T}{\B(A-1)^N}$, which implies that
    \begin{align*}
        \E_{\pi\sim\rho}[\Ind\{V^{\pi}_{\calM_{i,\jstar}}(\sstar_0) \geq 2B_T\}] = 1 - \int_{\pi} z^{\pi}(\jstar)\rho(\pi)d\pi \geq 1 - \frac{2B_T}{\B(A-1)^N} \geq \frac{1}{2}.
    \end{align*}
    The proof is completed by noting that for the picked $\jstar$, the bound above holds for any $i$, since the lower bound on $V^{\pi}_j(\sstar_0)$ we applied above is independent of $i$.
\end{proof}
Now consider another set of MDPs $\{\calM''_i\}_{i\in\{0,\ldots,N'\}}$ with state space $\calS_T$.
The transition and cost functions of $\calM''_i$ is the same as $\calM_{i,j}$ restricted on $\calS_T$ for any $j$, except that taking action $A$ at any state directly transits to $g$ with cost $2B_T$.
We show that any $(\epsilon',\delta')$-correct algorithm with $\epsilon'\in(0,\frac{1}{32})$, $\delta'\in(0,\frac{1}{2e^4})$ has sample complexity at least $C(\epsilon', \delta')$ on $\{\calM''_i\}_i$
Given any policy $\pi$ on $\calM''_i$, define $g_{\pi}$ as a policy on $\calM'_i$ and $\calM''_i$ such that $g_{\pi}(a|s)\propto \pi(a|s)$ and $\sum_{a=1}^{A-1}g_{\pi}(a|s)=1$.
It is straightforward to see that $V^{g_{\pi}}_{\calM'_i}(s)=V^{g_{\pi}}_{\calM''_i}(s)\leq V^{\pi}_{\calM''_i}(s)$ and $\optV_{\calM'_i}(s)=\optV_{\calM''_i}(s)$, where $\optV_{\calM}$ is the optimal value function in $\calM$.
Thus, if there exists an algorithm $\frA$ that is $(\epsilon',\delta')$-correct with sample complexity less than $C(\epsilon',\delta')$ on $\{\calM''_i\}_i$, then we can obtain an $(\epsilon',\delta')$-correct algorithm on $\{\calM'_i\}_i$ with sample complexity less than $C(\epsilon',\delta')$ as follows: executing $\frA$ on $\{\calM''_i\}_i$ to obtain policy $\hatpi$, and then output $g_{\hatpi}$.
This leads to a contradiction to the definition of $C(\cdot,\cdot)$, and thus any $(\epsilon',\delta')$-correct algorithm with $\epsilon'\in(0,\frac{1}{32})$, $\delta'\in(0,\frac{1}{2e^4})$ has sample complexity at least $C(\epsilon', \delta')$ on $\{\calM''_i\}_i$.

Since we assume that the learner has sample complexity less than $C(\epsilon,4\delta)$ on $\calM_{i,j}$, for a fixed $j_0$, there exists $\istar$ such that $P_{\istar,j_0}(\calE_3)>4\delta$, where $\calE_3=\{\exists s: V^{\hatpi}_{\calM''_{\istar}}(s) - \optV_{\calM''_{\istar}}(s) > \epsilon\}$ (note that $\hatpi$ is computed on $\calM_{\istar,j_0}$, but we can apply $\hatpi$ restricted on $\calS_T$ to $\calM''_{\istar}$).
This also implies that $P_{\istar,j}(\calE\cap\calE_3) = P_{\istar,j_0}(\calE\cap\calE_3) \geq\frac{5\delta}{2}$ for any $j$, since the value of $P_{i,j}(\omega)$ is independent of $j$ when $\omega\in\calE$.
Define $\calE_4=\{ \exists s: V^{\hatpi}_{\calM}(s)-V^{\star,T}_{\calM}(s)>\epsilon \}$.
By \pref{lem:Sstar}, there exists $\jstar$ such that
\begin{align*}
    P_{\istar,\jstar}(\calE_4|\calE\cap\calE_3) \geq \E_{\hatpi\sim P_{\istar}(\cdot|\calE\cap\calE_3)}[\Ind\{V^{\hatpi}_{\calM_{\istar,\jstar}}(\sstar_0) \geq 2B_T\}] \geq \frac{1}{2},
\end{align*} 
since the distribution of $\hatpi$ is independent of $j$ under $\calE\cap\calE_3$, $V^{\star,T}_{\calM_{\istar,j}}(s)=\optV_{\calM''_{\istar}}(s)$ for any $j$ and $s\in\calS_T$, and $V^{\hatpi}_{\calM_{\istar,j}}(s) \geq V^{\hatpi}_{\calM''_{\istar}}(s)$ for $s\in\calS_T$ when $V^{\hatpi}_{\calM_{\istar,j}}(\sstar_0)\geq 2B_T$.
Putting everything together, we have $P_{\istar,\jstar}(\calE_4)\geq P_{\istar,\jstar}(\calE_4\cap\calE\cap\calE_3)>\delta$, a contradiction.
Therefore, there is no $(\epsilon,\delta, T)$-correct algorithm with sample complexity less than $C(\epsilon, 4\delta)$ on $\{\calM_{i,j}\}_{i,j}$.
In other words, for any $(\epsilon,\delta, T)$-correct algorithm, there exists $\calM\in\{\calM_{i,j}\}_{i,j}$ such that this algorithm has sample complexity at least $C(\epsilon, 4\delta)$ on $\calM$.
This completes the proof.
\end{proof}

\section{Omitted Details in \pref{sec:alg.generative}}
In this section, we present the omitted proofs of \pref{lem:extend} and \pref{thm:bound.algo.generative}.
To prove \pref{thm:bound.algo.generative}, we first discuss the guarantee of the finite-horizon algorithm in \pref{app:LCBVI}.
Then, we bound the sample complexity of \pref{alg:sh} in \pref{app:bound.algo.generative}.

\subsection{\pfref{lem:extend}}
\label{app:extend}
\begin{proof}
    Let $V^{\pi}_{1,h}$ be the value function $V^{\pi}_1$ in $\calM_{h,c_f}$.
    For any $n \geq 0$, we have
    \begin{align*}
        V^{\pi}_{1,(n+1)H}(s) &= \E_{\pi}\sbr{\left.\sum_{i=1}^{nH}c(s_i, a_i) + V^{\pi}_{1,H}(s_{nH+1}) \right|s_1=s}\\
        &\leq \E_{\pi}\sbr{\left.\sum_{i=1}^{nH}c(s_i, a_i) + c_f(s_{nH+1}) \right|s_1=s} = V^{\pi}_{1,nH}(s).
    \end{align*}
    Therefore, $V^{\pi}(s)\leq\lim_{n\rightarrow\infty}V^{\pi}_{1,nH}(s)\leq V^{\pi}_{1,H}(s)$ and this completes the proof.
    Note that the first inequality may be strict.
    Indeed, $V^{\pi} = \lim_{H\to\infty} V^\pi_{1,H}$ in $\calM_{H,0}$.
    Consider an improper policy $\pi$ behaving in a loop with zero cost. Then, $V^{\pi}=0$ but $\lim_{H\rightarrow\infty}V^{\pi}_{1,H}=c_f$ in $\calM_{H,c_f}$. 
\end{proof}

\subsection{Guarantee of the Finite-Horizon Algorithm \LCBVI}
\label{app:LCBVI}

\setcounter{AlgoLine}{0}
\begin{algorithm}[t]
	\caption{LCBVI ($H,\N, B, c_f, \delta$)}
    \label{alg:LCBVI}
	\textbf{Input:} horizon $H$, counter $\N$, optimal value function upper bound $B$, terminal cost $c_f$, failure probability $\delta$, and cost function $c$,
	
	\textbf{Define:} $\P_{s,a}(s')=\frac{\N(s, a, s')}{\Np(s, a)}$ and $b(s, a, V) = \max\cbr{7\sqrt{\frac{\fV(\P_{s, a}, V)\iota}{\Np(s, a)}}, \frac{49B\iota}{\Np(s, a)} }$, where $\iota=\ln\frac{2SAHn}{\delta}$, $n=\sum_{s,a}\N(s,a)$, and $\Np(s,a)=\max\{1,\N(s,a)\}$.
	
	\textbf{Initialize:} $\hatV_{H+1}=c_f$.
	
	\For{$h=H,\ldots,1$}{
		\nl $\hatQ_h(s, a) = \rbr{ c(s, a) + \P_{s, a}\hatV_{h+1} - b(s, a, \hatV_{h+1})}_+$.\label{line:update rule}
		
		$\hatV_h(s)=\min_a\hatQ_h(s, a)$.
	}
	
	\textbf{Output:} $(\hatpi, \hatV)$ with $\hatpi(s, h)=\argmin_a\hatQ_h(s, a)$.
\end{algorithm}

In this section, we discuss and prove the guarantee of \pref{alg:LCBVI}.


\paragraph{Notations} Within this section, $H$, $\N$, $B$, $c_f$, $\delta$ are inputs of \pref{alg:LCBVI}, and $\hatpi$, $\hatQ$, $\hatV$, $\P_{s, a}$, $\N$, $\Np$, $\iota$, and $b$ are defined in \pref{alg:LCBVI}.
Value function $V^{\pi}_h$ is w.r.t MDP $\calM_{H,c_f}$, and we denote by $\optV_h$, $\optQ_h$ the optimal value function and action-value function, such that $\optV_h(s)=\argmin_{\pi}V^{\pi}_h(s)$ and $\optQ_h(s, a)=c(s,a)+P_{s,a}\optV_{h+1}$ for $(s, a, h)\in\SA\times[H]$.
We also define $V^{\pi}_{H+1}=\optV_{H+1}=c_f$ for any policy $\pi$, and $B^{\star}_H = \max_{h\in [H+1]}\norm{\optV_h}_{\infty}$.
For any $(\bars,\barh)\in\calS\times[H]$ and $(s,a,h)\in\SA\times[H]$, denote by $q_{\pi,(\bars,\barh)}(s, a, h)$ the probability of visiting $(s, a)$ in stage $h$ if the learner starts in state $\bars$ in stage $\barh$ and follows policy $\pi$ afterwards.
For any value function $V\in\fR^{\SA\times[H+1]}$, define $\norm{V_{\cdot}}_{\infty}=\max_{h\in[H+1]}\norm{V_h}_{\infty}$.

We first prove optimism of the estimated value functions.
\begin{lemma}
    \label{lem:opt}
    When $B\geq B^{\star}_H$, we have $\hatQ_h(s, a) \leq \optQ_h(s, a)$ for any $(s, a, h)\in\SA\times[H]$ with probability at least $1-\delta$.
\end{lemma}
\begin{proof}
	We prove this by induction.
	The case of $h=H+1$ is clearly true.
	For $h \leq H$, note that
	\begin{align*}
		&c(s, a) + \P_{s, a}\hatV_{h+1} - b(s, a, \hatV_{h+1}) \leq c(s, a) + \P_{s, a}\optV_{h+1} - b(s, a, \optV_{h+1}) \tag{\pref{lem:mvp}}\\
        &=c(s, a) + P_{s,a}\optV_{h+1} + (\P_{s,a}-P_{s,a})\optV_{h+1} - \max\cbr{7\sqrt{\frac{\fV(\P_{s, a},\optV_{h+1})\iota}{\Np(s, a)}}, \frac{49B\iota}{\Np(s, a)}}\\
		&\leq c(s, a) + P_{s, a}\optV_{h+1} + (2\sqrt{2}-3)\sqrt{\frac{\fV(\P_{s,a}, \optV_{h+1})\iota}{\Np(s, a)}} + (19 - 24)\frac{B\iota}{\Np(s, a)} \tag{\pref{lem:anytime bernstein} and $\max\{a, b\} \geq \frac{a+b}{2}$}\\
        &\leq c(s, a) + P_{s, a}\optV_{h+1} = \optQ_h(s, a).
	\end{align*}
	The proof is then completed by the definition of $\hatQ$.
\end{proof}

\begin{lemma}
    \label{lem:d hatV}
    For any state $\bars\in\calS$ and $\barh\in[H]$, we have
    \begin{align*}
        \abr{V^{\hatpi}_{\barh}(\bars) - \hatV_{\barh}(\bars)} \leq \sum_{s, a, h}q_{\hatpi,(\bars,\barh)}(s, a, h)\rbr{\abr{(P_{s, a} - \P_{s, a})\hatV_{h+1}} + b(s, a, \hatV_{h+1})}.
    \end{align*}
\end{lemma}
\begin{proof}
    First note that
    \begin{align*}
        V^{\hatpi}_{\barh}(\bars) - \hatV_{\barh}(\bars) &\leq P_{\bars,\hatpi(\bars,\barh)}V^{\hatpi}_{\barh+1} - \P_{\bars,\hatpi(\bars,\barh)}\hatV_{\barh+1} + b(\bars,\hatpi(\bars,\barh), \hatV_{\barh+1}) \tag{definition of $\hatpi$ and $\hatV$}\\
        &= P_{\bars,\hatpi(\bars,\barh)}(V^{\hatpi}_{\barh+1} - \hatV_{\barh+1}) + (P_{\bars,\hatpi(\bars,\barh)}-\P_{\bars,\hatpi(\bars,\barh)})\hatV_{\barh+1} + b(\bars,\hatpi(\bars,\barh), \hatV_{\barh+1})\\
        &= \sum_{s, a, h}q_{\hatpi,(\bars,\barh)}(s, a, h)\rbr{(P_{s, a} - \P_{s, a})\hatV_{h+1} + b(s, a, \hatV_{h+1})}. \tag{expand $P_{\bars,\hatpi(\bars,\barh)}(V^{\hatpi}_{\barh+1} - \hatV_{\barh+1})$ recursively and $V^{\hatpi}_{H+1}=\hatV_{H+1}$}
    \end{align*}
    For the other direction,
    \begin{align*}
        &(\hatV_{\barh}(\bars) - V^{\hatpi}_{\barh}(\bars))_+ \leq \rbr{\P_{\bars,\hatpi(\bars,\barh)}\hatV_{\barh+1} - P_{\bars,\hatpi(\bars,\barh)}V^{\hatpi}_{\barh+1} }_+ \tag{$(a)_+ - (b)_+ \leq (a-b)_+$}\\
        &\leq P_{\bars,\hatpi(\bars,\barh)}(\hatV_{\barh+1} - V^{\hatpi}_{\barh+1})_+ + \abr{(P_{\bars,\hatpi(\bars,\barh)}-\P_{\bars,\hatpi(\bars,\barh)})\hatV_{\barh+1}} \tag{$(a+b)_+\leq (a)_+ + (b)_+$}\\
        &\leq \sum_{s, a, h}q_{\hatpi,(\bars,\barh)}(s, a, h)\abr{(P_{s, a} - \P_{s, a})\hatV_{h+1}}. \tag{expand $P_{\bars,\hatpi(\bars,\barh)}(\hatV_{\barh+1} - V^{\hatpi}_{\barh+1})_+$ recursively}
    \end{align*}
    Combining both directions completes the proof.
\end{proof}
\begin{remark}
    Note that the inequality in \pref{lem:d hatV} holds even if optimism (\pref{lem:opt}) does not hold, which is very important for estimating $\B$.
\end{remark}

\begin{lemma}
    \label{lem:fine bound}
    There exists a function $N^{\star}(B',H',\epsilon',\delta') \lesssim \frac{{B'}^2H'}{{\epsilon'}^2} + \frac{SB'H'}{\epsilon'} + S{H'}^2$ such that when $B\geq B^{\star}_H$ and $\N(s, a)=N\geq N^{\star}(B,H,\epsilon,\delta)$ for all $s, a$ for some integer $N$, we have $\norm{V^{\hatpi}_{\cdot} - \optV_{\cdot}}_{\infty}\leq\epsilon$ with probability at least $1-\delta$.
\end{lemma}
\begin{proof}
    Below we assume that $B\geq B^{\star}_H$.
    Fix any state $\bars\in\calS$ and $\barh\in[H]$, we write $q_{\hatpi,(\bars,\barh)}$ as $q_{\hatpi}$ for simplicity.
    We have with probability at least $1-4\delta$,
	\begin{align*}
		&V^{\hatpi}_{\barh}(\bars) - \hatV_{\barh}(\bars)\\
        &\leq \sum_{s, a, h}q_{\hatpi}(s, a, h)\rbr{ \abr{(P_{s, a} - \P_{s, a})\optV_{h+1}} + \abr{(P_{s, a} - \P_{s, a})(\hatV_{h+1}-\optV_{h+1})} + b(s, a, \hatV_{h+1})} \tag{$|a+b|\leq |a| + |b|$ and \pref{lem:d hatV}}\\
        &\lesssim \sum_{s,a,h}q_{\hatpi}(s, a, h)\rbr{ \sqrt{\frac{\fV(P_{s, a}, \optV_{h+1})}{N}} + \frac{SB}{N} + \sqrt{\frac{S\fV(P_{s, a}, \hatV_{h+1}-\optV_{h+1})}{N}} + \sqrt{\frac{\fV(\P_{s, a}, \hatV_{h+1})}{N}} }  \tag{\pref{lem:dPV} and $\max\{a,b\}\leq (a)_+ + (b)_+$}\\
		&\lesssim \sqrt{\frac{H}{N}\sum_{s, a, h}q_{\hatpi}(s, a, h)\fV(P_{s, a}, \optV_{h+1})} + \sqrt{\frac{SH}{N}\sum_{s, a, h}q_{\hatpi}(s, a, h)\fV(P_{s, a}, \hatV_{h+1} - \optV_{h+1}) } + \frac{SBH}{N},
	\end{align*}
	where in the last inequality we apply $\var[X+Y]\leq 2(\var[X] + \var[Y])$, Cauchy-Schwarz inequality, \pref{lem:barPV to PV}, and $\sum_{s,a,h}q_{\hatpi}(s,a,h)\leq H$.
	Now note that:
	\begin{align*}
		&\sum_{s, a, h}q_{\hatpi}(s, a, h)\fV(P_{s, a}, \optV_{h+1}) = \E_{\hatpi}\sbr{\left.\sum_{h=\barh}^H\fV(P_{s_h, a_h}, \optV_{h+1}) \right|s_{\barh}=\bars }\\
        &= \E_{\hatpi}\sbr{\left.\sum_{h=\barh}^H\rbr{P_{s_h,a_h}(\optV_{h+1})^2 - (P_{s_h,a_h}\optV_{h+1})^2} \right|s_{\barh}=\bars } \\
		&= \E_{\hatpi}\sbr{\left. \sum_{h=\barh}^H\rbr{\optV_{h+1}(s_{h+1})^2 - \optV_h(s_h)^2} + \sum_{h=\barh}^H\rbr{\optV_h(s_h)^2 - (P_{s_h, a_h}\optV_{h+1})^2}\right| s_{\barh}=\bars }\\
		&\leq B^2 + 3B\E_{\hatpi}\sbr{\sum_{h=\barh}^H\rbr{\optQ_h(s_h, a_h) - P_{s_h, a_h}\optV_{h+1}}_+ } \tag{$a^2-b^2\leq(a+b)(a-b)_+$ for $a,b>0$ and $\optV(s_h)\leq\optQ_h(s_h, a_h)$}\\
        &= B^2 + 3B\E_{\hatpi}\sbr{\left.\sum_{h=\barh}^Hc(s_h, a_h)\right|s_{\barh}=\bars} = B^2 + 3BV^{\hatpi}_{\barh}(\bars).
	\end{align*}
	Plugging this back and by $\fV(P_{s,a}, \hatV_{h+1}-\optV_{h+1})\leq\norm{\hatV_{h+1}-\optV_{h+1}}_{\infty}^2$, we have with probability at least $1-\delta$,
	\begin{align*}
		0 \leq V^{\hatpi}_{\barh}(\bars) - \hatV_{\barh}(\bars) &\lesssim B\sqrt{\frac{H}{N}} + \sqrt{\frac{BHV^{\hatpi}_{\barh}(\bars)}{N}} + \sqrt{\frac{SH^2}{N}}\norm{\hatV_{\cdot} - \optV_{\cdot}}_{\infty} + \frac{SBH}{N} \tag{\pref{lem:opt}}\\
        &\lesssim B\sqrt{\frac{H}{N}} + \sqrt{\frac{BH(V^{\hatpi}_{\barh}(\bars)-\hatV_{\barh}(\bars))}{N}} + \sqrt{\frac{SH^2}{N}}\norm{\hatV_{\cdot} - V^{\hatpi}_{\cdot}}_{\infty} + \frac{SBH}{N},
	\end{align*}
        where in the last step we apply $\hatV_{\barh}(\bars)\leq B$ and $\norm{\hatV_{\cdot}-\optV_{\cdot}}_{\infty}\leq\norm{\hatV_{\cdot}-V^{\hatpi}_{\cdot}}_{\infty}$ since $\hatV(s)\leq\optV(s)\leq V^{\hatpi}(s)$ for all $s\in\calS$ by \pref{lem:opt}.
    Solving a quadratic inequality w.r.t $V^{\hatpi}_{\barh}(\bars)-\hatV_{\barh}(\bars)$, we have
    \begin{align*}
        V^{\hatpi}_{\barh}(\bars)-\hatV_{\barh}(\bars) \lesssim B\sqrt{\frac{H}{N}} + \sqrt{\frac{SH^2}{N}}\norm{\hatV_{\cdot} - V^{\hatpi}_{\cdot}}_{\infty} + \frac{SBH}{N}.
    \end{align*}
    The inequality above implies that there exist quantity $\overline{N}^{\star}\lesssim SH^2$, such that when $N\geq \overline{N}^{\star}$, we have
    \begin{align*}
        V^{\hatpi}_{\barh}(\bars) - \hatV_{\barh}(\bars) \lesssim B\sqrt{\frac{H}{N}} + \frac{1}{2}\norm{\hatV_{\cdot} - V^{\hatpi}_{\cdot}}_{\infty} + \frac{SBH}{N},
    \end{align*}
    for any $(\bars, \barh)$.
    Taking maximum of the left-hand-side over $(\bars,\barh)$, reorganizing terms and by \pref{lem:opt}, we obtain
    \begin{equation}
        \label{eq:Nstar}
        \norm{V^{\hatpi}_{\cdot} - \optV_{\cdot}}_{\infty}\leq\norm{V^{\hatpi}_{\cdot}-\hatV_{\cdot}}_{\infty}\lesssim B\sqrt{\frac{H}{N}} + \frac{SBH}{N}.
    \end{equation}
    Now define $n^{\star}=\overline{N}^{\star} + \inf_n\{\text{right-hand-side of \pref{eq:Nstar} }\leq \epsilon\text{ when }N=n\}$.
    We have $n^{\star}\lesssim \frac{B^2H}{\epsilon^2} + \frac{SBH}{\epsilon} + SH^2$.
    This implies that when $B\geq B^{\star}_H$ and $\N(s, a)=N\geq n^{\star}$ for all $s, a$, we have $\norm{V^{\hatpi}_{\cdot}-\optV_{\cdot}}_{\infty}\leq\epsilon$ with probability at least $1-5\delta$.
    The proof is then completed by treating $n^{\star}$ as a function with input $B$, $H$, $\epsilon$, $\delta$ and replace $\delta$ by $\delta/5$ in the arguments above.
\end{proof}

\begin{lemma}
    \label{lem:coarse bound}
    There exists functions $\hatN(B', H', \epsilon', \delta') \lesssim \frac{{B'}^2SH'}{{\epsilon'}^2} + \frac{SB'H'}{\epsilon'}$ such that when $\N(s, a)=N\geq \hatN(B, H, \epsilon, \delta)$ for all $s, a$ for some $N$ and $\norm{\hatV_{\cdot}}_{\infty} \leq B$, we have $\norm{V^{\hatpi}_{\cdot} - \hatV_{\cdot}}_{\infty}\leq\epsilon$ with probability at least $1-\delta$.
\end{lemma}
\begin{proof}
    Below we assume that $\norm{\hatV_{\cdot}}_{\infty} \leq B$.
    For any state fixed $\bars\in\calS$ and $\barh\in[H]$, we write $q_{\hatpi,(\bars,\barh)}$ as $q_{\hatpi}$ for simplicity.
    Note that with probability at least $1-2\delta$,
	\begin{align*}
		&\abr{V^{\hatpi}_{\barh}(\bars) - \hatV_{\barh}(\bars)} \leq \sum_{s, a, h}q_{\hatpi}(s, a, h)\rbr{\abr{(P_{s, a} - \P_{s, a})\hatV_{h+1}} + b(s, a, \hatV_{h+1})} \tag{\pref{lem:d hatV}}\\
		&\lesssim \sum_{s, a, h}q_{\hatpi}(s, a, h)\rbr{\sqrt{\frac{S\fV(P_{s, a}, \hatV_{h+1})}{N}} + \sqrt{\frac{\fV(\P_{s, a}, \hatV_{h+1})}{N}} + \frac{SB}{N}} \tag{\pref{lem:dPV} and $\max\{a, b\}\leq (a)_+ + (b)_+$}\\
		&\lesssim \sum_{s, a, h}q_{\hatpi}(s, a, h)\rbr{\sqrt{\frac{S\fV(P_{s, a}, \hatV_{h+1})}{N}} + \frac{SB}{N}} \tag{\pref{lem:barPV to PV}}\\
		&\lesssim \sqrt{\frac{SH}{N}\sum_{s, a, h}q_{\hatpi}(s, a, h)\fV(P_{s, a}, \hatV_{h+1})} + \frac{SBH}{N}. \tag{Cauchy-Schwarz inequality and $\sum_{s,a,h}q_{\hatpi}(s,a,h)\leq H$}
	\end{align*}
	Now note that with probability at least $1-\delta$,
	\begin{align*}
		&\sum_{s, a, h}q_{\hatpi}(s, a, h)\fV(P_{s, a}, \hatV_{h+1}) = \E_{\hatpi}\sbr{\left.\sum_{h=\barh}^H\fV(P_{s_h, a_h}, \hatV_{h+1})\right|s_{\barh}=\bars }\\
		&= \E_{\hatpi}\sbr{\left. \sum_{h=\barh}^H\rbr{\hatV_{h+1}(s_{h+1})^2 - \hatV_h(s_h)^2} + \sum_{h=\barh}^H\rbr{\hatV_h(s_h)^2 - (P_{s_h, a_h}\hatV_{h+1})^2}\right|s_{\barh}=\bars }\\
		&\leq B^2 + 3B\E_{\hatpi}\sbr{\left.\sum_{h=\barh}^H\rbr{\hatQ_h(s_h, a_h) - P_{s_h, a_h}\hatV_{h+1}}_+\right| s_{\barh}=\bars } \tag{$a^2-b^2\leq(a+b)(a-b)_+$ for $a,b>0$ and $\hatV_h(s_h)=\hatQ_h(s_h,a_h)$}\\
		&\leq B^2 + 3B\E_{\hatpi}\sbr{\left. \sum_{h=\barh}^H\rbr{ c(s_h, a_h) + (\P_{s_h, a_h} - P_{s_h, a_h})\hatV_{h+1} }_+ \right|s_{\barh}=\bars}\tag{definition of $\hatQ_h$ and $(a)_+-(b)_+\leq(a-b)_+$}\\
		&\lesssim B^2 + BV^{\hatpi}_{\barh}(\bars) + B\sqrt{\frac{SH}{N}\sum_{s, a, h}q_{\hatpi}(s, a, h)\fV(P_{s, a}, \hatV_{h+1})} + \frac{SB^2H}{N},
	\end{align*}
    where the last step is by $(a+b)_+\leq(a)_+ + (b)_+$, the definition of $V^{\hatpi}_{\barh}(\bars)$ , and
    \begin{align*}
        &\E_{\hatpi}\sbr{\left. \sum_{h=\barh}^H\rbr{ (\P_{s_h, a_h} - P_{s_h, a_h})\hatV_{h+1} }_+ \right|s_{\barh}=\bars}\\
        &\lesssim \E_{\hatpi}\sbr{\left. \sum_{h=\barh}^H\rbr{\sqrt{\frac{S\fV(P_{s_h,a_h}, \hatV_{h+1})}{N}} + \frac{SB}{N}} \right|s_{\barh}=\bars} \tag{\pref{lem:dPV} }\\
        &=\sum_{s,a,h}q_{\hatpi}(s,a,h)\rbr{\sqrt{\frac{S\fV(P_{s, a}, \hatV_{h+1})}{N}} + \frac{SB}{N}} \leq \sqrt{\frac{SH}{N}\sum_{s,a,h}q_{\hatpi}(s,a,h)\fV(P_{s,a}, \hatV_{h+1})} + \frac{SBH}{N}. \tag{Cauchy-Schwarz inequality and $\sum_{s,a,h}q_{\hatpi}(s,a,h)\leq H$}
    \end{align*}
	Solving a quadratic inequality w.r.t $\sum_{s, a, h}q_{\hatpi}(s, a, h)\fV(P_{s, a}, \hatV_{h+1})$, we have
	\begin{align*}
		\sum_{s, a, h}q_{\hatpi}(s, a, h)\fV(P_{s, a}, \hatV_{h+1}) \lesssim B^2 + BV^{\hatpi}_{\barh}(\bars) + \frac{SB^2H}{N}.
	\end{align*}
	Plugging this back, we have
	\begin{align*}
		\abr{V^{\hatpi}_{\barh}(\bars) - \hatV_{\barh}(\bars)} &\lesssim B\sqrt{\frac{SH}{N}} + \sqrt{\frac{BSHV^{\hatpi}_{\barh}(\bars)}{N}} + \frac{SBH}{N}\\
        &\lesssim B\sqrt{\frac{SH}{N}} + \sqrt{\frac{BSH|V^{\hatpi}_{\barh}(\bars) - \hatV_{\barh}(\bars)|}{N}} + \frac{SBH}{N}. \tag{$\hatV_{\barh}(\bars)\leq B$}
	\end{align*}
	Again solving a quadratic inequality w.r.t $|V^{\hatpi}_{\barh}(\bars) - \hatV_{\barh}(\bars)|$ and taking maximum over $(\bars, \barh)$ on the left-hand-side, we have
    \begin{equation}
        \label{eq:hatN}
        \norm{V^{\hatpi}_{\cdot} - \hatV_{\cdot}}_{\infty} \lesssim B\sqrt{\frac{SH}{N}} + \frac{SBH}{N}.
    \end{equation}
    Now define $\widehat{n}= \inf_n\{\text{right-hand-side of \pref{eq:hatN} }\leq \epsilon\text{ when }N=n\}$.
    We have $\widehat{n}\lesssim \frac{B^2SH}{\epsilon^2} + \frac{SBH}{\epsilon}$.
    This implies that when $\N(s, a)=N\geq \hatn$ for all $s, a$ and $\norm{\hatV_{\cdot}}_{\infty}\leq B$, we have $\norm{V^{\hatpi}_{\cdot}-\hatV_{\cdot}}_{\infty}\leq\epsilon$ with probability at least $1-4\delta$.
    The proof is then completed by treating $\hatn$ as a function with input $B$, $H$, $\epsilon$, $\delta$ and replace $\delta$ by $\delta/4$ in the arguments above.
\end{proof}

\begin{lemma}
    \label{lem:barPV to PV}
    For any $(s, a)\in\SA$ and $V\in [-B,B]^{\calS_+}$ for some $B>0$, with probability at least $1-\delta$, we have $\fV(\P_{s, a}, V) \lesssim \fV(P_{s, a}, V) + \frac{SB^2}{\Np(s, a)}$ for all $(s, a)$, where $\Np(s, a)=\max\{1,\N(s, a)\}$.
\end{lemma}
\begin{proof}
    Note that
    \begin{align*}
        \fV(\P_{s, a}, V) &\leq \P_{s, a}(V - P_{s, a}V)^2 \tag{$\frac{\sum_ip_ix_i}{\sum_ip_i}=\argmin_z\sum_ip_i(x_i-z)^2$}\\
        &= \fV(P_{s, a}, V) + (\P_{s,a}-P_{s,a})(V - P_{s,a}V)^2\\
        &\lesssim \fV(P_{s, a}, V) +  B\sqrt{\frac{S\fV(P_{s, a}, V)}{\Np(s, a)}} + \frac{SB^2}{\Np(s, a)} \tag{\pref{lem:dPV}}\\
        &\lesssim \fV(P_{s, a}, V) + \frac{SB^2}{\Np(s, a)}. \tag{AM-GM inequality}
    \end{align*}
    This completes the proof.
\end{proof}

\begin{lemma}
    \label{lem:dPV}
    Given any value function $V\in[-B,B]^{\calS_+}$, with probability at least $1-\delta$, $|(P_{s,a}-\P_{s,a})V|\lesssim \sqrt{\frac{S\fV(P_{s,a},V)}{\Np(s, a)}} + \frac{SB}{\Np(s, a)}$ for any $(s, a)\in\SA$, where $\Np(s, a)=\max\{1,\N(s, a)\}$.
\end{lemma}
\begin{proof}
    For any $(s,a)\in\SA$, by \pref{lem:anytime bernstein}, with probability at least $1-\frac{\delta}{SA}$, we have
    \begin{align*}
        &|(P_{s,a}-\P_{s,a})V| \leq \sum_{s'}|P_{s,a}(s')-\P_{s,a}(s')||V(s')-P_{s,a}V| \tag{$\sum_{s'}(P_{s,a}(s')-\P_{s,a}(s'))=0$}\\
        &\lesssim \sum_{s'}\rbr{\sqrt{\frac{P_{s,a}(s')}{\Np(s,a)}} + \frac{1}{\Np(s,a)}}\abr{ V(s')-P_{s,a}V } \lesssim \sqrt{\frac{S\fV(P_{s,a}, V)}{\Np(s,a)}} + \frac{SB}{\Np(s, a)},
    \end{align*}
    where the last step is by Cauchy-Schwarz inequality.
    Taking a union bound over $(s, a)$ completes the proof.
\end{proof}

\subsection{\pfref{thm:bound.algo.generative}}
\label{app:bound.algo.generative}
\begin{proof}
    For each index $i$, define finite-horizon MDP $\calM_i=\calM_{H_i,c_{f,i}}$.
    Also define $V^{\pi}_{h,i}$ and $\optV_{h,i}$ as value function $V^{\pi}_h$ and optimal value function $\optV_h$ in $\calM_i$ respectively.
    We first assume that $T\geq D$ such that $B_{\star,T}<\infty$.
    In this case, we have $T^{\optpi_{T,s}}(s) \leq \min\{\Tc,T\}$ for any $s$ by $\optpi_{T,s}=\optpi$ when $T\geq\Tc\geq\T$.
    Note that when $B_i\in [20B_{\star,T}, 40B_{\star,T}]$, by $T^{\optpi_{T,s}}(s) \leq \min\{\Tc,T\}$ for any $s$, definition of $H_i$, and \pref{lem:hitting}, we have $\optV_{1,i}(s) \leq V^{\optpi_{T,s}}_{1,i}(s) \leq V^{\star,T}(s) + 0.6B_i\cdot\frac{\epsilon}{24B_i} \leq 0.1B_i$ and $\optV_{h,i}(s)\leq V^{\optpi_{T,s}}_{h,i}(s) \leq V^{\star,T}(s) + 0.6B_i \leq 0.7B_i$ for any $s\in\calS$ and $h\in[H]$, where applying stationary policy $\optpi_{T,s}$ in $\calM_i$ means executing $\optpi_{T,s}$ in each step $h\in[H]$.
    This implies $B_i \geq \norm{\optV_{\cdot,i}}_{\infty}$.
    Then according to \pref{line:estimate Vi} and by \pref{lem:opt}, with probability at least $1-\delta_i$, we have $\norm{V^i_1}_{\infty} \leq \norm{\optV_{1,i}}_{\infty}\leq 0.1B_i$, $\norm{V^i_{\cdot}}_{\infty} \leq \norm{\optV_{\cdot,i}}_{\infty}\leq 0.7B_i$, and the while loop should break (\pref{line:stop Vi}).
    Let $\istar$ be the value of $i$ when the while loop breaks, we thus have $B_{\istar}\leq 40B_{\star, T}$.
    Moreover, by \pref{lem:coarse bound} and the definition of $N_i$, with probability at least $1-\delta_{\istar}$, we have $V^{\pi^{\istar}}_{1,\istar}(s) \leq (V^{\istar}_1(s) + 0.1B_{\istar})\Ind\{s\neq g\} \leq c_{f,\istar}(s)$ for any $s\in\calS_+$.
    Thus by \pref{lem:extend}, we have $\optV(s) \leq V^{\pi^{\istar}}(s) \leq V^{\pi^{\istar}}_{1,\istar}(s) \leq V^{\istar}_1(s) + 0.1B_{\istar} \leq B_{\istar}$ for any $s\in\calS$.
    This gives $B_{\istar}\geq\B$.
    If $T<\Tc\leq\frac{B_{\istar}}{\cmin}$, then $H_{\istar}\lesssim\min\{\frac{B_{\istar}}{\cmin},T\}= \min\{\Tc,T\}$.
    Otherwise, $T\geq\Tc$, $B_{\star,T}=\B$, and $H_{\istar}\lesssim\min\{\frac{B_{\istar}}{\cmin},T\}\lesssim \min\{\Tc, T\}$ by $B_{\istar}\lesssim B_{\star,T}$.
    Therefore, $H_{\istar}\lesssim\min\{\Tc,T\}$.
    By \pref{lem:coarse bound}, $\norm{V^{\istar}_1}_{\infty}\leq 0.1B_{\istar}$, $\norm{V^{\istar}_{\cdot}}_{\infty}\leq 0.7B_{\istar}$ (breaking condition of the while loop), and the definition of $N_i$, we have with probability at least $1-\delta_{\istar}$, $\norm{\optV_{1,\istar}}_{\infty}\leq\norm{V^{\pi^{\istar}}_{1,\istar}}_{\infty}\leq \norm{V^{\istar}_1}_{\infty} + \norm{V^{\pi^{\istar}}_{1,\istar} - V^{\istar}_1}_{\infty} \leq 0.2B_{\istar}$ and $\norm{\optV_{\cdot,\istar}}_{\infty}\leq \norm{V^{\pi^{\istar}}_{\cdot,\istar}}_{\infty} \leq \norm{V^{\istar}_{\cdot}}_{\infty} + \norm{ V^{\pi^{\istar}}_{\cdot,\istar} - V^{\istar}_{\cdot} }_{\infty} \leq 0.8B_{\istar}$.
    Therefore, by \pref{lem:fine bound} and the definition of $N^{\star}_{\istar}$, we have with probability at least $1-\delta_{\istar}$, $\norm{V^{\hatpi}_{\cdot,\istar} - \optV_{\cdot,\istar}}_{\infty} \leq \frac{\epsilon}{2}$.
    Moreover, by \pref{lem:extend} and $V^{\hatpi}_{1,\istar}(s) \leq (\optV_{1,\istar}(s) + \frac{\epsilon}{2})\Ind\{s\neq g\} \leq (0.2B_{\istar} + \frac{1}{2})\Ind\{s\neq g\} \leq c_{f,\istar}(s)$ for all $s\in\calS_+$ since $B_{\istar}\geq 2$, we have $V^{\hatpi}(s)\leq V^{\hatpi}_{1,\istar}(s)$ for all $s$.
    Thus, $V^{\hatpi}(s) \leq V^{\hatpi}_{1,\istar}(s) \leq \optV_{1,\istar}(s) + \frac{\epsilon}{2} \leq V^{\star,T}(s) + 0.6B_{\istar}\frac{\epsilon}{24B_{\istar}} + \frac{\epsilon}{2}\leq V^{\star,T}(s) + \epsilon$ by the definition of $H_{\istar}$ and \pref{lem:hitting} for any $s\in\calS$.
    Finally, by the definition of $N_i$ and $N^{\star}_i$, the total number of samples spent is of order
    \begin{align*}
        &\tilO{SA(N_{\istar} + N^{\star}_{\istar})} = \tilO{\frac{H_{\istar}B_{\istar}^2SA}{\epsilon^2} + \frac{H_{\istar}B_{\istar}S^2A}{\epsilon} + H_{\istar}^2S^2A }\\ 
        &= \tilO{ \min\cbr{\Tc, T }\frac{B_{\star,T}^2SA}{\epsilon^2} + \min\cbr{\Tc, T}\frac{B_{\star,T}S^2A}{\epsilon} + \min\cbr{\Tc, T}^2S^2A }.
    \end{align*}
    Moreover, the bound above holds with probability at least $1-\delta$ since $20\sum_i\delta_i \leq \delta$.
    Now we consider the case $T<D$.
    From the arguments above we know that $B_i\leq 40B_{\star,T}\leq 40T$ for all $i\leq \istar$ if $T\geq D$.
    Thus, we can conclude that $T<D$ if $B_i>40T$ for some $i$ still in the while loop, and the total number samples used is of order $\tilo{SAN_i}=\tilo{S^2AT}$ by the definition of $N_i$.
    This completes the proof.
\end{proof}

\section{Omitted Details in \pref{sec:lower.bounds.BPI}}

\subsection{\pfref{thm:lower.bound.BPI}}
\label{app:lower.bound.BPI}
\begin{proof}
Let $N=\min\{\floor{\B}, S-3\}$.
Consider a family of MDPs $\{\calM_j\}_{j \in [A]^N}$ with $\calS=\calS_N\cup\calS'$, $\calA=[A]$, and $\sinit=s_0$, where $\calS_N=\{s_0,s_1,\ldots,s_N\}$ and $\calS'=\{s_b,s_c,\ldots\}$.
Clearly, $|\calS_N|=N+1$ and $|\calS'|=S-N-1$.
For each $\calM_j$, the cost is $1$ for every state-action pair in $\calS_N$;
at $s_0$, taking any action transits to $g$ with probability $1-p$, and transits to state $s_1$ with probability $p$, where $p=\frac{4\epsilon}{A^N}$;
at $s_i$ for $i\in[N]$, taking action $j_i$ transits to $s_{i+1}$ (define $s_{N+1}=g$), while taking any other actions transits to $s_1$;
at $s_b$, taking any action suffers cost $1$ and transits to $g$ with probability $1/\B$ (stays at $s_b$ otherwise);
at $s_c$, taking any action suffers cost $\cmin$ and directly transits to $g$;
at any of the rest of states in $\calS'$, taking any action suffers cost $1$ and directly transits to $g$.
Note that all of these MDPs have parameters $S$, $A$, $\B$, $\cmin$, and all these parameters are known to the learner.
Also note that states in $\calS'$ are unreachable and does not affect the learner.
We include them simply to show that we can obtain a hard instance for any values of $S$, $A$, $\B$, and $\cmin$ using dummy states.

Consider a learner that is $(\epsilon,\delta)$-correct with $\epsilon\in(0,\frac{1}{4})$, $\delta\in(0,\frac{1}{16})$, and sample complexity $\frac{1}{p}$ on $\{\calM_j\}_j$. 
Denote by $\calE_1$ the event that the first $\frac{1}{p}$ steps from $s_0$ all transit to $g$, $\calE_2$ the event that the learner uses at most $\frac{1}{p}$ samples, and define $\calE=\calE_1\cap\calE_2$.
Also denote by $P_j$ the distribution w.r.t $\calM_j$.
Note that event $\calE$ is agnostic to $j$, that is, for any interaction history (including the randomness of the learner) $\omega\in\calE$, we have $P_j(\omega)=P(\omega)$ for all $j$. Also note that for any $j$, we have $P_j(\calE_1)=(1-p)^{1/p}\geq \frac{1}{4}$ and $P_j(\calE_2)\geq 1 - \delta \geq \frac{7}{8}$. 
Thus, $P(\calE) = P_j(\calE) \geq \frac{1}{8}$.
Now we show that the failure probability of such a learner is large.
Note that when $\calE$ is true, the learner outputs $\hatpi$ before visiting $s_1$.
Moreover, the distribution of $\hatpi$ under $\calE$ is identical for all $\{\calM_j\}_j$, that is, $P_j(\hatpi|\calE)=P(\hatpi|\calE)$.
This is because $P_j(\omega)=P(\omega)$ for any interaction history $\omega\in\calE$, and $\hatpi$ is a function of $\omega$.
Denote by $\calE'$ the bad event that $\hatpi$ is not $\epsilon$-optimal.
We show that there exists $j$ such that $P_j(\calE'|\calE)$ is sufficiently large.

First, for any given $j$ and any policy $\pi$, define $x^{\pi}_j=\prod_{i=1}^N\pi(j_i|s_i)$ and $V^{\pi}_j$ as the value function of $\pi$ in $\calM_j$.
Since the learner transits to $s_1$ if it does not follow the ``correct'' action sequence, we have $V^{\pi}_j(s_1)\geq Nx^{\pi}_j + (1-x^{\pi}_j)(1 + V^{\pi}_j(s_1))$, which gives $V^{\pi}_j(s_1)\geq N + \frac{1}{x^{\pi}_j} -1$.
Moreover, if $\pi$ is $\epsilon$-optimal in $\calM_j$, then we have $V^{\pi}_j(s_0)\leq 1 + pN + \epsilon$.
Combining with $V^{\pi}_j(s_0) = 1 + pV^{\pi}_j(s_1)$ gives $x^{\pi}_j\geq\frac{1}{1+\epsilon/p}\geq\frac{p}{2\epsilon}$ by $\epsilon/p\geq 1$.
Also note that $\sum_jx^{\pi}_j=1$.
Therefore, each policy $\pi$ can be $\epsilon$-optimal for at most $\frac{2\epsilon}{p}$ MDPs in $\{\calM_j\}_j$.

Denote by $y^{\pi}(j)$ the indicator of whether policy $\pi$ is $\epsilon$-optimal in $\calM_j$.
We have $\sum_jy^{\pi}(j)\leq \frac{2\epsilon}{p}$ for any $\pi$.
Therefore, $\sum_j\int_{\hatpi}P(\hatpi|\calE)y^{\hatpi}(j)d\hatpi \leq \frac{2\epsilon}{p}$, which implies that there exist $\jstar$ such that $\int_{\hatpi}P(\hatpi|\calE)y^{\hatpi}(\jstar)d\hatpi \leq \frac{2\epsilon}{pA^N}$.
Therefore,
\begin{align*}
	P_{\jstar}(\calE'|\calE) = 1 - \int_{\hatpi}P(\hatpi|\calE)y^{\hatpi}(\jstar)d\hatpi \geq 1 - \frac{2\epsilon}{pA^N} = \frac{1}{2}.
\end{align*}
The overall failure probability in $\calM_{\jstar}$ is thus $P_{\jstar}(\calE')\geq P(\calE)P_{\jstar}(\calE'|\calE)\geq \frac{1}{16}$, a contradiction.
Therefore, for any $(\epsilon,\delta)$-correct learner, there exists $\calM\in\{\calM_j\}_j$ such that the learner has sample complexity more than $\frac{1}{p}=\lowo{A^N/\epsilon}$ on $\calM$.
The proof is then completed by the definition of $N$.
\end{proof}

\subsection{\pfref{thm:lower.bound.terminal}}
\label{app:lower.bound.terminal}
\begin{proof}
If $\min\cbr{\Tc, \uT}\frac{\B^2SA}{\epsilon^2}\ln\frac{1}{\delta}>\frac{J}{\epsilon}$, then we simply construct a full $(A-1)$-ary tree following that of \pref{thm:lower.bound.eps.T}, with a remaining action $\da$.
By $J\geq 3B$, we can simply ignore action $\da$ and the sample complexity lower bound is $\lowo{\min\cbr{\Tc, \uT}\frac{\B^2SA}{\epsilon^2}\ln\frac{1}{\delta} }=\lowo{\min\cbr{\frac{\B}{\cmin}, \uT}\frac{\B^2SA}{\epsilon^2}\ln\frac{1}{\delta} + \frac{J}{\epsilon} }$.

Otherwise, we have $\min\cbr{\Tc, \uT}\frac{\B^2SA}{\epsilon^2}\ln\frac{1}{\delta}\leq\frac{J}{\epsilon}$, and our construction follows that of \pref{thm:lower.bound.BPI} except that we have an action $\da$ at every state.
Consider a family of MDPs $\{\calM_j\}_{j \in [A-1]^N}$ with $\calS=\calS_N\cup\calS'$, $\calA=[A-1]\cup\{\da\}$, and $\sinit=s_0$, where $\calS_N=\{s_0,s_1,\ldots,s_N\}$ and $\calS'=\{s_b,s_c,\ldots\}$.
For each $\calM_j$, $c(s, a)=1$ for all $(s,a)\in\calS_N\times[A-1]$;
at $s_0$, taking any action in $[A-1]$ transits to $g$ with probability $1-p$, and transits to state $s_1$ with probability $p$, where $p=\frac{4\epsilon}{J}$;
at $s_i$ for $i\in[N]$, taking action $j_i$ transits to $s_{i+1}$ (define $s_{N+1}=g$), while taking any other actions in $[A-1]$ transits to $s_1$;
at $s_b$, taking any action in $[A-1]$ suffers cost $1$ and transits to $g$ with probability $1/B$ (stay at $s_b$ otherwise);
at $s_c$, taking any action in $[A-1]$ suffers cost $\cmin$ and directly transits to $g$;
at any of the rest of states in $\calS'$, taking any action in $[A-1]$ suffers cost $1$ and directly transits to $g$.
Note that all of these MDPs have parameters $S$, $A$, $\cmin$, $\uT$ and satisfy $\B=B$.
Moreover, all these parameters are known to the learner.

Consider a learner that is $(\epsilon,\delta)$-correct with $\epsilon\in(0,\frac{1}{4})$, $\delta\in(0,\frac{1}{16})$, and sample complexity $\frac{1}{p}$ on $\{\calM_j\}_j$. 
Denote by $\calE_1$ the event that the first $\frac{1}{p}$ steps from $(s_0, a)$ for some $a\in[A-1]$ all transit to $g$, $\calE_2$ the event that the learner uses at most $\frac{1}{p}$ samples, and define $\calE=\calE_1\cap\calE_2$.
Also denote by $P_j$ the distribution w.r.t $\calM_j$.
Note that event $\calE$ is agnostic to $j$, that is, for any interaction history (including the randomness of the learner) $\omega\in\calE$, we have $P_j(\omega)=P(\omega)$ for all $j$. Also note that for any $j$, we have $P_j(\calE_1)=(1-p)^{1/p}\geq \frac{1}{4}$ and $P_j(\calE_2)\geq 1 - \delta \geq \frac{7}{8}$. 
Thus, $P(\calE) = P_j(\calE) \geq \frac{1}{8}$.
Now we show that the failure probability of such a learner is large.
Note that when $\calE$ is true, the learner outputs $\hatpi$ before ever visiting $s_1$.
Moreover, the distribution of $\hatpi$ under $\calE$ is identical for all $\{\calM_j\}_j$, that is, $P_j(\hatpi|\calE)=P(\hatpi|\calE)$.
This is because $P_j(\omega)=P(\omega)$ for any interaction history $\omega\in\calE$, and $\hatpi$ is a function of $\omega$.
Denote by $\calE'$ the bad event that $\hatpi$ is not $\epsilon$-optimal.
We show that there exists $j$ such that $P_j(\calE'|\calE)$ is sufficiently large.

First, for any given $j$ and any policy $\pi$, define $x^{\pi}_j=\prod_{i=1}^N\pi(j_i|s_i)$, $V^{\pi}_j$ as the value function of $\pi$ in $\calM_j$, and $y^{\pi}=\pi(\da|s_1)$.
If $\pi$ is $\epsilon$-optimal in $\calM_j$, then we have $V^{\pi}_j(s_0)\leq 1 + pN + \epsilon < J$.
Combining with $V^{\pi}_j(s_0)\geq \min\{J, 1 + pV^{\pi}_j(s_1)\}$, we have $V^{\pi}_j(s_1)\leq N+\frac{\epsilon}{p} = N + \frac{J}{4} < J - 1$.
Moreover, the learner suffers cost $N$ if it follows the ``correct'' action sequence, and suffers at least cost $J > 1 + V^{\pi}_j(s_1)$ if it ever takes action $\da$.
Thus, we have $V^{\pi}_j(s_1)\geq Nx^{\pi}_j + Jy^{\pi} + (1-x^{\pi}_j - y^{\pi})(1 + V^{\pi}_j(s_1))$, which gives $V^{\pi}_j(s_1)\geq \frac{Nx^{\pi}_j+Jy^{\pi}}{x^{\pi}_j+y^{\pi}} + \frac{1}{x^{\pi}_j+y^{\pi}} -1$.
Combining with $V^{\pi}_j(s_1)\leq N+\frac{\epsilon}{p}$, we have
\begin{equation}
    \label{eq:xy}
    (1+\epsilon/p)(x^{\pi}_j+y^{\pi})\geq 1 + (J-N)y^{\pi}.
\end{equation}
Now note that $\sum_jx^{\pi}_j\leq 1-y^{\pi}$ for any $\pi$.
Define $\calX^{\pi}$ as the set of $j\in[A-1]^N$ where $\pi$ is $\epsilon$-optimal in $\calM_j$.
Summing over $j\in\calX^{\pi}$ for \pref{eq:xy}, we have $(1+\epsilon/p)(1-y^{\pi}+|\calX^{\pi}|y^{\pi})\geq |\calX^{\pi}|+(J-N)y^{\pi}|\calX^{\pi}|$. 
Reorganizing terms and assuming $|\calX^{\pi}|\geq 1$ gives
\begin{align*}
	1 - \frac{1+\epsilon/p}{|\calX^{\pi}|} \leq y^{\pi}\rbr{\rbr{1-\frac{1}{|\calX^{\pi}|}}\rbr{1 + \frac{\epsilon}{p}} - (J-N)} \leq y^{\pi}\rbr{1 + \frac{\epsilon}{p} - (J-N)}.
\end{align*}
Note that $1+\frac{\epsilon}{p}\leq J-N$ by $p\geq\frac{\epsilon}{J-N-1}$.
Thus, the right-hand-side $\leq 0$, which gives $|\calX^{\pi}|\leq 1+ \epsilon/p\leq J-N$.
Therefore, each policy can be $\epsilon$-optimal for at most $J-N$ MDPs in $\{\calM_j\}_j$.
Denote by $z^{\pi}(j)$ the indicator of whether $\pi$ is $\epsilon$-optimal in $\calM_j$.
We have $\sum_jz^{\pi}(j)\leq J - N$ for any policy $\pi$.
Therefore, $\sum_j\int_{\hatpi}P(\hatpi|\calE)z^{\hatpi}(j)d\hatpi \leq J-N$, which implies that there exist $\jstar$ such that $\int_{\hatpi}P(\hatpi|\calE)z^{\hatpi}(\jstar)d\hatpi \leq \frac{J-N}{(A-1)^N}$.
Therefore,
\begin{align*}
	P_{\jstar}(\calE'|\calE) = 1 - \int_{\hatpi}P(\hatpi|\calE)z^{\hatpi}(\jstar)d\hatpi \geq 1 - \frac{J-N}{(A-1)^N} \geq \frac{1}{2}.
\end{align*}
The overall failure probability in $\calM_{\jstar}$ is thus $P_{\jstar}(\calE')\geq P(\calE)P_{\jstar}(\calE'|\calE)\geq \frac{1}{16}$, a contradiction.
Therefore, for any $(\epsilon,\delta)$-correct learner, there exists $\calM\in\{\calM_j\}_j$ such that the learner has sample complexity more than $\frac{1}{p}=\lowo{J/\epsilon}=\lowo{\min\cbr{\Tc, \uT}\frac{\B^2SA}{\epsilon^2}\ln\frac{1}{\delta} + J/\epsilon}$ on $\calM$.
This completes the proof.
\end{proof}

\section{Omitted Details in \pref{sec:alg.BPI}}
\label{app:alg.BPI}
In this section, we present the proof of \pref{thm:algo.BPI}.

\paragraph{Notations} Denote by $\optV_{\cdot}$, $\optQ_{\cdot}$ the optimal value function and action-value function of $\calM_{H,c_f}$, where $H=\frac{32J}{\cmin}\ln\frac{8J}{\epsilon}$ and $c_f(s)=J\Ind\{s\neq g\}$.
Clearly, we have $\optV_h(s)=\argmin_a\optQ_h(s, a)$, $\optQ_h(s, a)=c(s, a)+P_{s,a}\optV_{h+1}$ for any $(s,a,h)\in\SA\times[H]$, and $\optV_{H+1}=c_f$.
The whole learning process is divided into episodes indexed by $m$.
Define $H_m=\inf_h\{s^m_{h+1}=g\text{ or }a^m_{h+1}=\da\}$ as the length of episode $m$ and $R_{[m',M']}=\sum_{m=m'}^{M'}(\sumhmp c^m_h - V^m_1(s^m_1))$ the regret w.r.t estimated value functions of episodes in $[m',M']$, where $c^m_h=c(s^m_h, a^m_h)$ and $c^m_{H_m+1}=c_f(s^m_{H_m+1})$.
We also write $R_{[1,M']}$ as $R_{M'}$.
Define $s^m_h=g$ for all $h>H_m+1$.
Denote by $\P^m$, $Q^m$, $V^m$, $b^m_h$, $\N^m_h$ the value of $\P$, $\hatQ$, $\hatV$, $b(s^m_h, a^m_h, V^m_{h+1})$, $\Np(s^m_h, a^m_h)$ from \pref{alg:LCBVI} executed in \pref{line:hat compute} in episode $m$.
For any episode $m$ in round $r$, define $\pi_m=\pi^r$.
Define $P^m_h=P_{s^m_h, a^m_h}$ and $\P^m_h=\P^m_{s^m_h, a^m_h}$.
Define $C^m=\sumhmp c^m_h$ and $C_{M'}=\summp C^m$.
Denote by $\lambda_r$ the value of $\lambda$ in round $r$ (computed in \pref{line:lambda}) and $B_m$ the value of $B$ in episode $m$.

\begin{lemma}
    \label{lem:BPI opt}
    With probability at least $1-\delta$, $Q^m_h(s, a)\leq \optQ_h(s, a)$ for any $m\geq 1$ and $(s, a, h)\in\SA\times[H]$.
\end{lemma}
\begin{proof}
    This is simply by $\max_{h\in[H+1]}\norm{\optV_h}_{\infty} \leq \norm{\optV}_{\infty} + \norm{c_f}_{\infty} \leq 2J$ and \pref{lem:opt}.
\end{proof}
We are now ready to prove \pref{thm:algo.BPI}.
\begin{proof}[of \pref{thm:algo.BPI}]
Denote by $\calI_r$ the set of episodes in round $r$.
First note that in the last round (success round) $R$, we have with probability at least $1-35\delta$,
\begin{align*}
    V^{\hatpi}(\sinit) - \optV(\sinit) &\leq V^{\hatpi}(\sinit) - \optV_1(\sinit) + \frac{\epsilon}{4} \tag{definition of $H$ and \pref{lem:hitting}}\\
    &\leq \frac{1}{\lambda_R}\sum_{m\in\calI_R}\rbr{V^{\pi_m}_1(s^m_1) - C^m} + \frac{1}{\lambda_R}\sum_{m\in\calI_R}\rbr{C^m - V^m_1(s^m_1)} + \frac{\epsilon}{4} \tag{$\pi_m=\hatpi$ for $m\in\calI_R$, $s^m_1=\sinit$, and \pref{lem:BPI opt}}\\
    &\leq \frac{\epsilon}{4} + \frac{\epsilon}{2} + \frac{\epsilon}{4} \leq \epsilon,
\end{align*}
where the last step is by \pref{lem:N dev}, the definition of $\lambda$, the condition of success round, and $V^m_1=V^R$ for $m\in\calI_R$.
Thus, the output policy $\hatpi$ is indeed $\epsilon$-optimal. 
Now we bound the sample complexity of the algorithm.
Note that by \pref{lem:hitting} and the definition of $H$, for any $h\leq H/2+1$ and $s\in\calS$, we have $\optV(s) \leq \optV_h(s) \leq \optV(s) + \frac{\epsilon J}{2J}\leq \frac{3}{2}\B$.
Then by \pref{lem:BPI opt}, we have $B_{M'} \leq 3\B$ for any $M'\geq 1$ with probability at least $1-\delta$.
Note that there are at most $\bigo{SA\log_2(R\lambda_R)}$ skip rounds, and one success round.
Thus, it suffices to bound the number of failure rounds $R_f$.
In each failure round, we have at least $\frac{\lambda_R\epsilon}{2}$ regret by the condition in \pref{line:failure}.
Moreover, in each skip or success round $r$, we have with probability at least $1-35\delta$,
\begin{align*}
    \sum_{m\in\calI_r}\rbr{ C^m - V^m_1(s^m_1) } \geq \sum_{m\in\calI_r}\rbr{ C^m - V^{\pi_m}_1(s^m_1) } \gtrsim -\lambda_R\epsilon.
\end{align*}
where the last step is by \pref{eq:dev bound} and the value of $\lambda_r$.
Thus, the total regret is lower bounded as follows: $R_{M}\gtrsim (R_f-SA)\lambda_R\epsilon$.
Denote by $M$ the total number episodes.
By \pref{lem:reg BPI}, \pref{lem:fine bound BPI} and $M\leq R\lambda_R \lesssim (SA+R_f)\lambda_R$, we have with probability at least $1-38\delta$,
\begin{align*}
    R_M \lesssim B_M\sqrt{SAM} + J^2H^{1.25}S^2A \lesssim B_M\sqrt{SA(SA+R_f)\lambda_R} + J^2H^{1.25}S^2A.
\end{align*}
Solving a quadratic inequality w.r.t $R_f$, we have
\begin{align*}
    R_f \lesssim SA + \frac{B_MSA}{\epsilon\sqrt{\lambda_R}} + \frac{B_M^2SA}{\epsilon^2\lambda_R} + \frac{J^2H^{1.25}S^2A}{\lambda_R\epsilon}  \lesssim SA.
\end{align*}
Therefore, $M\leq R\lambda_R \lesssim \frac{\B^2SA}{\epsilon^2} + \frac{J^2H^2S^2A^2}{\epsilon}$ by $B_M\leq 3\B$.
Moreover, by \pref{eq:bound C} with $M'=M$ and $B_M\leq 3\B$, we know that $C_M\lesssim \B M + JS^2A$ and thus the total number of samples is of order
\begin{align*}
    \tilO{\frac{\B M}{\cmin} + \frac{JS^2A}{\cmin}} = \tilO{\frac{\Tc\B^2SA}{\epsilon^2} + \frac{\B J^4S^2A^2}{\cmin^3\epsilon}}.
\end{align*}
This completes the proof.
\end{proof}


\begin{lemma}
    \label{lem:N dev}
    There exists function $N_{\dev}(B, \epsilon, \delta) \lesssim \frac{B^2}{\epsilon^2} + \frac{(H^{1.5}+J^2H^{1.25})SA}{\epsilon}$, such that for any $m'\geq 1$ and $n\geq N_{\dev}(B, \epsilon, \delta)$, if $B\geq B_{m'+n-1}$, then $\frac{1}{n}\abr{\sum_{m=m'}^{m'+n-1}(V^{\pi_m}_1(s^m_1)-C^m)} \leq \epsilon$ with probability at least $1-35\delta$.
\end{lemma}
\begin{proof}
    By \pref{lem:dev BPI} and \pref{lem:fine bound BPI}, for any $m',n\geq 1$, we have with probability at least $1-35\delta$,
    \begin{align}
        \abr{\frac{1}{n}\sum_{m=m'}^{m'+n-1}\rbr{V^{\pi_m}_1(s^m_1)-C^m}} \lesssim B_{m'+n-1}\sqrt{\frac{1}{n}} + \frac{(H^{1.5}+J^2H^{1.25})SA}{n}.\label{eq:dev bound}
    \end{align}
    Thus, for any given $B$, $\epsilon$, $\delta$, there exists $N\lesssim \frac{B^2}{\epsilon^2} + \frac{(H^{1.5}+J^2H^{1.25})SA}{\epsilon}$ such that if $n\geq N$ and $B\geq B_{m'+n-1}$, then $\frac{1}{n}\abr{\sum_{m=m'}^{m'+n-1}(V^{\pi_m}_1(s^m_1)-C^m)}\leq\epsilon$ with probability at least $1-35\delta$.
    Treating $N$ as function of $(B, \epsilon, \delta)$ completes the proof.
\end{proof}

Below we state the regret guarantee of LCBVI.
The main idea is to bound the regret w.r.t $B$ instead of $\B$, which is useful in deciding the number of episodes needed.

\begin{lemma}
    \label{lem:reg BPI}
    For any $m'\geq 1$, with probability at least $1-9\delta$, we have for all $M'\geq m'$ simultaneously $$R_{[m',M']}\lesssim \sqrt{SA\sum_{m=m'}^{M'}\sumhm\fV(P^m_h, V^m_{h+1})} + JS^2A.$$
\end{lemma}
\begin{proof}
    Without loss of generality, we assume $m'=1$.
    Note that
    \begin{align*}
        R_{M'} &=\summp\rbr{\sumhmp c^m_h - V^m_1(s^m_1)} \overset{\text{(i)}}{\lesssim} \summp\sumhm\rbr{c^m_h + V^m_{h+1}(s^m_{h+1}) - V^m_h(s^m_h)} + JSA\\
        &\leq \summp\sumhm\rbr{ (\Ind_{s^m_{h+1}} - P^m_h)V^m_{h+1} + (P^m_h-\P^m_h)V^m_{h+1} + b^m_h } + JSA \tag{definition of $V^m_h$},
    \end{align*}
    where (i) is by the fact that $V^m_{H_m+1}(s^m_{H_m+1})\neq c^m_{H_m+1}\leq J$ only if $m$ is the last episode of a skip round, and there are at most $\tilo{SA}$ skip rounds.
    We bound the three sums above separately.
    For the first sum, by \pref{lem:anytime freedman}, we have with probability at least $1-\delta$,
    \begin{align*}
        \summp\sumhm(\Ind_{s^m_{h+1}} - P^m_h)V^m_{h+1} \lesssim \sqrt{\summp\sumhm\fV(P^m_h, V^m_{h+1})} + J.
    \end{align*}
    For the second sum, we have with probability at least $1-7\delta$,
    \begin{align*}
        &\summp\sumhm(P^m_h-\P^m_h)V^m_{h+1} = \summp\sumhm(P^m_h-\P^m_h)\optV_{h+1} + \summp\sumhm(P^m_h-\P^m_h)(V^m_{h+1}-\optV_{h+1})\\
        &\lesssim \summp\sumhm\rbr{\sqrt{\frac{\fV(P^m_h,\optV_{h+1})}{\N^m_h}} + \sqrt{\frac{S\fV(P^m_h, V^m_{h+1}-\optV_{h+1})}{\N^m_h}} + \frac{SJ}{\N^m_h} } \tag{\pref{lem:anytime bernstein} and \pref{lem:dPV}}\\
        &\lesssim \sqrt{SA\summp\sumhm\fV(P^m_h, \optV_{h+1})} + \sqrt{S^2A\summp\sumhm\fV(P^m_h, V^m_{h+1} - \optV_{h+1})} + JS^2A. \tag{Cauchy-Schwarz inequality and \pref{lem:sum N}}\\
        &\lesssim \sqrt{SA\summp\sumhm\fV(P^m_h, V^m_{h+1})} + \sqrt{S^2A\summp\sumhm\fV(P^m_h, V^m_{h+1}-\optV_{h+1})} + JS^2A. \tag{$\var[X+Y]\leq\var[X]+\var[Y]$}\\
        &\lesssim \sqrt{SA\summp\sumhm\fV(P^m_h, V^m_{h+1})} + JS^2A. \tag{\pref{lem:sum dV} and AM-GM inequality}
    \end{align*}
    Plugging these back and applying \pref{lem:sum b} to bound $\summp\sumhm b^m_h$ completes the proof.
\end{proof}

\begin{lemma}
    \label{lem:sum dV}
    For any $m'\geq 1$, with probability at least $1-5\delta$, we have
    $\sum_{m=m'}^{M'}\sumhm\fV(P^m_h, \optV_{h+1} - V^m_{h+1})\lesssim J\sqrt{SA\sum_{m=m'}^{M'}\sumhm \fV(P^m_h, V^m_{h+1})} + J^2S^2A$ for all $M'\geq m'$.
\end{lemma}
\begin{proof}
    Without loss of generality, we assume $m'=1$.
    First note that with probability at least $1-\delta$,
    \begin{align*}
        &\summp\sumhm\rbr{(\optV_h(s^m_h) - V^m_h(s^m_h))^2 - (P^m_h(\optV_{h+1}-V^m_{h+1}))^2}\\
        &\lesssim J\summp\sumhm(\optV_h(s^m_h) - V^m_h(s^m_h) - P^m_h\optV_{h+1} + P^m_hV^m_{h+1})_+ \tag{\pref{lem:BPI opt} and $a^2-b^2\leq (a+b)(a-b)_+$ for $a,b>0$}\\
        &\lesssim J\summp\sumhm(c^m_h + P^m_hV^m_{h+1} - V^m_h(s^m_h))_+ \tag{$\optV_h(s^m_h)\leq \optQ_h(s^m_h, a^m_h) = c^m_h + P^m_h\optV_{h+1}$}.
    \end{align*}
    By the definition of $V^m_h$ and $(a)_+-(b)_+\leq(a-b)_+$, with probability at least $1-3\delta$, we continue with
    \begin{align*}
        &\lesssim J\summp\sumhm( (P^m_h-\P^m_h)\optV_{h+1} + (P^m_h-\P^m_h)(V^m_{h+1} - \optV_{h+1}) + b^m_h )_+\\
        &\lesssim J\summp\sumhm\rbr{\sqrt{\frac{\fV(P^m_h, \optV_{h+1})}{\N^m_h}} + \sqrt{\frac{S\fV(P^m_h, V^m_{h+1}-\optV_{h+1})}{\N^m_h}} + \frac{SJ}{\N^m_h} + b^m_h} \tag{\pref{lem:anytime bernstein} and Cauchy-Schwarz inequality}\\
        &\lesssim J\rbr{\sqrt{SA\summp\sumhm\fV(P^m_h, V^m_{h+1})} + \sqrt{S^2A\summp\sumhm\fV(P^m_h, V^m_{h+1} - \optV_{h+1})} } + J^2S^2A,
    \end{align*}
    where in the last step we apply $\var[X+Y]\leq\var[X]+\var[Y]$, Cauchy-Schwarz inequality, \pref{lem:sum N} and \pref{lem:sum b}.
    Then applying \pref{lem:sum var} with $\norm{\optV_{h+1}-V^m_h}_{\infty}\leq J$ and solving a quadratic inequality w.r.t $\summp\sumhm\fV(P^m_h, \optV_{h+1}-V^m_{h+1})$, we have with probability at least $1-\delta$,
    \begin{align*}
        &\summp\sumhm\fV(P^m_h, \optV_{h+1} - V^m_{h+1})\\
        &\lesssim \summp (\optV_{H_m+1}(s^m_{H_m+1})-V^m_{H_m+1}(s^m_{H_m+1}))^2 + J\sqrt{SA\summp\sumhm \fV(P^m_h, V^m_{h+1})} + J^2S^2A.
    \end{align*}
    Now note that $\summp(\optV_{H_m+1}(s^m_{H_m+1})-V^m_{H_m+1}(s^m_{H_m+1}))^2 \lesssim J^2SA$ since $\optV_{H_m+1}(s^m_{H_m+1})\neq V^m_{H_m+1}(s^m_{H_m+1})$ only when $m$ is the last episode of a skip round, and the number of skip rounds is of order $\tilo{SA}$.
    Plugging this back completes the proof.
\end{proof}

\begin{lemma}
    \label{lem:dev BPI}
    For any $m'\geq 1$, with probability at least $1-6\delta$, we have for all $M'\geq m'$ simultaneously,
    $$\abr{\sum_{m=m'}^{M'}(C^m - V^{\pi_m}_1(s^m_1))} \lesssim \sqrt{\sum_{m=m'}^{M'}\sumhm\fV(P^m_h, V^m_{h+1})} + H^{1.5}SA.$$
\end{lemma}
\begin{proof}
    Without loss of generality, we assume $m'=1$.
    Denote by $\{(\tils^m_h, \tila^m_h, \tils^m_{h+1})\}_{h=1}^H$ the visited state-action-next-state triplets in episode $m$ if the learner follows $\pi^m$ till the end, that is, it does not stop the current episode immediately if the number of visits to some state-action pair is doubled.
    Note that $(\tils^m_h, \tila^m_h)\neq (s^m_h, a^m_h)$ only if $m$ is the last episode of a skip round and $H_m<h$.
    Also define $\tilC^m=\sumh c(\tils^m_h, \tila^m_h) + c_f(\tils^m_{H+1})$ the corresponding total cost in episode $m$, and $\tilP^m_h=P_{\tils^m_h, \tila^m_h}$.
    By \citep[Lemma 26]{chen2022policy} and the fact that $\pi^m$ is a deterministic policy for any $m$, we have $\var_{\pi_m}[\tilC^m]=\E_{\pi^m}[\sumh\fV(\tilP^m_h, V^{\pi_m}_{h+1})]$.
    Moreover, there are at most $\tilo{SA}$ skip rounds.
    Then with probability at least $1-2\delta$,
    \begin{align*}
        &\abr{\summp(C^m - V^{\pi_m}_1(s^m_1))} \lesssim \abr{\summp(\tilC^m-V^{\pi_m}_1(s^m_1))} + HSA \lesssim \sqrt{\summp \var_{\pi_m}[\tilC^m]} + HSA \tag{\pref{lem:anytime freedman}}\\ 
        &\lesssim \sqrt{\summp\sumh\fV(\tilP^m_h, V^{\pi_m}_{h+1})} + H^{1.5}SA \lesssim \sqrt{\summp\sumhm\fV(P^m_h, V^{\pi_m}_{h+1})} + H^{1.5}SA \tag{\pref{lem:e2r}, and $\sumh\fV(\tilP^m_h, V^{\pi_m}_{h+1}) \leq H^3$}\\
    	&\lesssim \sqrt{\summp\sumhm\fV(P^m_h, V^m_{h+1}) } + \sqrt{\summp\sumhm\fV(P^m_h, V^{\pi_m}_{h+1} - V^m_{h+1})} + H^{1.5}SA. \tag{$\var[X+Y]\leq 2(\var[X]+\var[Y])$}\\
    \end{align*}
    For the second term above, note that with probability at least $1-3\delta$,
    \begin{align*}
        &\summp\sumhm\rbr{(V^{\pi_m}_h(s^m_h) - V^m_h(s^m_h))^2 - (P^m_h(V^{\pi_m}_{h+1}-V^m_{h+1}))^2}\\
        &\lesssim H\summp\sumhm\rbr{V^{\pi_m}_h(s^m_h) - V^m_h(s^m_h) - P^m_hV^{\pi_m}_{h+1} + P^m_hV^m_{h+1}}_+ \tag{\pref{lem:BPI opt} and $a^2-b^2\leq(a+b)(a-b)_+$ for $a,b>0$}\\
    	&\lesssim H\summp\sumhm\rbr{c^m_h + P^m_hV^m_{h+1} - V^m_h(s^m_h) }_+ \tag{$V^{\pi_m}_h(s^m_h) = c(s^m_h, a^m_h) + P^m_hV^{\pi_m}_{h+1}$}\\
    	&\lesssim H\summp\sumhm\rbr{ (P^m_h-\P^m_h)V^m_{h+1} + b^m_h }_+. \tag{definition of $V^m_h$}\\
        &\lesssim H\summp\sumhm\rbr{\sqrt{\frac{S\fV(P^m_h, V^m_{h+1})}{\N^m_h}} + \frac{SJ}{\N^m_h}} + H\sqrt{SA\summp\sumhm\fV(P^m_h, V^m_{h+1})} + HJS^{1.5}A \tag{\pref{lem:anytime bernstein}, Cauchy-Schwarz inequality, and \pref{lem:sum b}}\\
        &\lesssim H\sqrt{S^2A\summp\sumhm\fV(P^m_h, V^m_{h+1})} +  HJS^2A. \tag{Cauchy-Schwarz inequality and \pref{lem:sum N}}
    \end{align*}
    Thus, by \pref{lem:sum var} with $\norm{V^{\pi_m}_h - V^m_h}_{\infty}\leq H$ and $\summp (V^{\pi_m}_{H_m+1}(s^m_{H_m+1})-V^m_{H_m+1}(s^m_{H_m+1}))^2\lesssim H^2SA$ since $V^{\pi_m}_{H_m+1}(s^m_{H_m+1})\neq V^m_{H_m+1}(s^m_{H_m+1})$ only when the number of visits to some state-action pair is doubled, we have with probability at least $1-\delta$,
    \begin{align*}
        \summp\sumhm\fV(P^m_h, V^{\pi_m}_{h+1} - V^m_{h+1}) \lesssim H\sqrt{S^2A\summp\sumhm\fV(P^m_h, V^m_{h+1})} + H^2S^2A.
    \end{align*}
    Plugging this back and applying AM-GM inequality completes the proof.
\end{proof}

\begin{lemma}[Coarse Bound]
    \label{lem:coarse bound BPI}
    For any $m'\geq 1$, with probability at least $1-2\delta$, we have for all $M'\geq m'$, $\sum_{m=m'}^{M'}\sumhm\fV(P^m_h, V^m_{h+1}) \lesssim JC_{M'} + J^2S^2A$.
\end{lemma}
\begin{proof}
    Without loss of generality, we assume $m'=1$.
    By the definition of $V^m_h$ and $(a)_+-(b)_+\leq(a-b)_+$, we have with probability at least $1-\delta$,
    \begin{align}
        &\summp\sumhm(V^m_h(s^m_h) - P^m_hV^m_{h+1})_+ \lesssim \summp\sumhm(c^m_h + (\P^m_h-P^m_h)V^m_{h+1})_+\notag\\ 
        &\lesssim \summp\sumhm c^m_h + \summp\sumhm\rbr{\sqrt{\frac{S\fV(P^m_h, V^m_{h+1})}{\N^m_h}} + \frac{SJ}{\N^m_h}} \tag{\pref{lem:dPV}}\\
        &\lesssim \summp\sumhm c^m_h + \sqrt{S^2A\summp\sumhm\fV(P^m_h, V^m_{h+1})} + JS^2A,\label{eq:sum V-PV}
    \end{align}
    where the last step is by \pref{lem:sum N}.
    Therefore,
    \begin{align*}
        \summp\sumhm(V^m_h(s^m_h)^2 - (P^m_hV^m_{h+1})^2) &\lesssim J\summp\sumhm(V^m_h(s^m_h) - P^m_hV^m_{h+1})_+ \tag{$a^2-b^2\leq(a+b)(a-b)_+$ for $a,b>0$}\\
        &\lesssim J\summp\sumhm c^m_h + J\sqrt{S^2A\summp\sumhm\fV(P^m_h, V^m_{h+1})} + J^2S^2A.
    \end{align*}
    Thus by \pref{lem:sum var}, we have with probability at least $1-\delta$,
    \begin{align*}
        &\sum_{m=m'}^{M'}\sumhm\fV(P^m_h, V^m_{h+1})\\
        &\lesssim \summp\sumhm V^m_{H_m+1}(s^m_{H_m+1})^2 + J\summp\sumhm c^m_h + J\sqrt{S^2A\summp\sumhm\fV(P^m_h, V^m_{h+1})} + J^2S^2A\\
        &\lesssim JC_{M'} + J\sqrt{S^2A\summp\sumhm\fV(P^m_h, V^m_{h+1})} + J^2S^2A. \tag{$V^m_{H_m+1}(s^m_{H_m+1}) \leq 2c^m_{H_m+1}$}
    \end{align*}
    Solving a quadratic inequality w.r.t $\summp\sumhm\fV(P^m_h, V^m_{h+1})$ completes the proof.
\end{proof}

\begin{lemma}[Refined Bound]
    \label{lem:fine bound BPI}
    For any $m'\geq 1$, with probability at least $1-29\delta$, for all $M'\geq m'$, we have $\sum_{m=m'}^{M'}\sumhm\fV(P^m_h, V^m_{h+1}) \lesssim B_{M'}^2(M'-m'+1) + J^4H^{2.5}S^2A$.
\end{lemma}
\begin{proof}
Without loss of generality, we assume $m'=1$ and write $B_{M'}$ as $B$ for simplicity.
By \pref{lem:reg BPI} and \pref{lem:coarse bound BPI}, we have with probability at least $1-11\delta$,
\begin{align*}
    R_{M'} = C_{M'} - \summp V^m_1(s^m_1) \lesssim \sqrt{JSAC_{M'}} + JS^2A.
\end{align*}
Solving a quadratic inequality w.r.t $C_{M'}$ gives
\begin{equation}
    \label{eq:bound C}
    C_{M'}\lesssim \summp V^m_1(s^m_1) + JS^2A \lesssim BM' + JS^2A.
\end{equation}
Thus, with probability at least $1-16\delta$,
\begin{align}
    &\summp(V^{\pi_m}_1(s^m_1) - \optV_1(s^m_1)) \lesssim \summp(V^{\pi_m}_1(s^m_1) - C^m) + \summp(C^m - \optV_1(s^m_1)) \notag\\
    &\lesssim \sqrt{SA\summp\sumhm\fV(P^m_h, V^m_{h+1})} + H^{1.5}S^2A \tag{\pref{lem:dev BPI}, \pref{lem:reg BPI}, and \pref{lem:BPI opt}}\\
    &\lesssim \sqrt{JSAC_{M'}} + H^{1.5}S^2A \lesssim \sqrt{JBSAM} + H^{1.5}S^2A,\label{eq:Vpi - optV}
\end{align}
where the last two steps are by \pref{lem:coarse bound BPI} and \pref{eq:bound C} respectively.
Now note that
\begin{align}
    &\summp\sumhm(V^m_h(s^m_h)^2 - (P^m_hV^m_{h+1})^2)\notag\\
    &\leq\summp\sumhm(V^m_h(s^m_h) + P^m_hV^m_h)(V^m_h(s^m_h) - P^m_hV^m_h)_+ \tag{$a^2-b^2\leq (a+b)(a-b)_+$ for $a,b>0$}\\
    &= 2B\summp\sumhm(V^m_h(s^m_h) - P^m_hV^m_h)_+ + \summp\sumhm(V^m_h(s^m_h) + P^m_hV^m_h-2B)(V^m_h(s^m_h) - P^m_hV^m_h)_+\notag\\
    &\lesssim BC_{M'} + B\sqrt{S^2A\summp\sumhm\fV(P^m_h, V^m_{h+1})} + JBS^2A + J^2\summp\sum_{h=H/2+1}^{H_m}\Ind\{s^m_{H/2+1}\neq g\},\label{eq:V2-PV2}
\end{align}
where the last step is by \pref{eq:sum V-PV} and $\norm{V^m_h}_{\infty}\leq B$ when $h\leq H/2+1$.
Similarly, we have
\begin{equation}
    \summp V^m_{H_m+1}(s^m_{H_m+1})^2 = B^2M' + \summp(V^m_{H_m+1}(s^m_{H_m+1})^2-B^2) \leq B^2M' + J^2\summp\Ind\{s^m_{H/2+1}\neq g\}.\label{eq:V2 end}
\end{equation}
It suffices to bound $\summp\Ind\{s^m_{H/2+1}\neq g\}$.
Let $\tilV^{\pi}_1$ and $\opttilV_1$ be the value function and optimal value function of $\calM_{H/4, \mathbf{0}}$, where $\mathbf{0}$ represents constant function with value $0$.
By the value of $H$ and \citep[Lemma 1]{chen2021implicit}, we have $\optV(s)\leq \opttilV_1(s) + \frac{\epsilon}{4J}$ for all $s\in\calS_+$.
Moreover, when $s^m_{H/2+1}\neq g$, we have $\sum_{h=H/4+1}^{H/2}c^m_h \geq 2J$.
Denote by $P_m(\cdot)$ the probability distribution conditioned on the events before episode $m$.
We have
\begin{align*}
    &2J\summp P_m(s^m_{H/2+1} \neq g) + \summp(\tilV^{\pi_m}_1(s^m_1) - \opttilV_1(s^m_1))\\
    &\leq \summp(V^{\pi_m}_1(s^m_1) - \optV(s^m_1)) + \frac{M'\epsilon}{J} \lesssim \sqrt{JBSAM} + H^{1.5}S^2A + \frac{M'\epsilon}{J}. \tag{\pref{eq:Vpi - optV}}
\end{align*}
Therefore, $\summp P_m(s^m_{H/2+1}\neq g)\lesssim \sqrt{SAM'} + \frac{H^{1.5}S^2A}{J} + \frac{M'\epsilon}{J^2}$ by $\tilV^{\pi_m}_1(s^m_1)\geq \opttilV_1(s^m_1)$, and so does $\summp \Ind\{s^m_{H/2+1}\neq g\}$ with probability at least $1-\delta$ by \pref{lem:e2r}.
Plugging this back to \pref{eq:V2-PV2}, \pref{eq:V2 end}, and by \pref{lem:sum var}, we have with probability at least $1-\delta$,
\begin{align*}
    &\summp\sumhm\fV(P^m_h, V^m_{h+1})\\
    &\lesssim BC_{M'} + B^2M' + B\sqrt{S^2A\summp\sumhm\fV(P^m_h, V^m_{h+1})} + JH^{2.5}S^2A + J^2H\sqrt{SAM'} + M'H\epsilon.
\end{align*}
Solving a quadratic inequality w.r.t $\summp\sumhm\fV(P^m_h, V^m_{h+1})$, applying AM-GM inequality on $J^2H\sqrt{SAM'}$, and applying \pref{eq:bound C} completes the proof.
\end{proof}

\begin{lemma}\citep[Lemma 9]{chen2022near}
    \label{lem:sum var}
    For any sequence of value functions $\{V^m_h\}_{m,h}$ with $V^m_h\in[0, B]^{\calS_+}$ for some $B>0$, we have $\summp\sumhm\fV(P^m_h,V^m_{h+1})\lesssim \summp V^m_{H_m+1}(s^m_{H_m+1})^2 + \summp\sumhm\rbr{V^m_h(s^m_h)^2 - (P^m_hV^m_{h+1})^2}  + B^2$ for all $M'\geq 1$ with probability at least $1-\delta$.
\end{lemma}
\begin{proof}
    For any $M'\geq 1$, we decompose the sum as follows:
    \begin{align*}
        &\summp\sumhm\fV(P^m_h,V^m_{h+1}) = \summp\sumhm \rbr{P^m_h(V^m_{h+1})^2 - V^m_{h+1}(s^m_{h+1})^2}\\
        &\qquad + \summp\sumhm\rbr{V^m_{h+1}(s^m_{h+1})^2-V^m_h(s^m_h)^2} + \summp\sumhm\rbr{V^m_h(s^m_h)^2 - (P^m_hV^m_{h+1})^2}.
    \end{align*}
    For the first term, by \pref{lem:anytime freedman} and \pref{lem:X2}, with prbability at least $1-\delta$,
    \begin{align*}
        &\summp\sumhm \rbr{P^m_h(V^m_{h+1})^2 - V^m_{h+1}(s^m_{h+1})^2}\\
        &\lesssim \sqrt{\summp\sumhm\fV(P^m_h, (V^m_{h+1})^2)} + B^2 \lesssim B\sqrt{\summp\sumhm\fV(P^m_h, V^m_{h+1})} + B^2.
    \end{align*}
    The second term is bounded by $\summp V^m_{H_m+1}(s^m_{H_m+1})^2$.
    Plugging these back and solving a quadratic inequality w.r.t $\summp\sumhm\fV(P^m_h, V^m_{h+1})$ completes the proof.
\end{proof}

\begin{lemma}\citep[Lemma 10]{chen2022near}
    \label{lem:sum b}
    With probability at least $1-\delta$, for all $M'\geq 1$,
    $\summp\sumhm b^m_h\lesssim \sqrt{SA\summp\sumhm\fV(P^m_h, V^m_{h+1})} + JS^{1.5}A$.
\end{lemma}
\begin{proof}
    Note that
    \begin{align*}
        \summp\sumhm b^m_h &\lesssim \summp\sumhm\rbr{\sqrt{\frac{\fV(\P^m_h, V^m_{h+1})}{\N^m_h}} + \frac{J}{\N^m_h}} \tag{$\max\{a,b\}\leq a + b$}\\
        &\lesssim \summp\sumhm\rbr{ \sqrt{\frac{\fV(P^m_h, V^m_{h+1})}{\N^m_h}} + \frac{J\sqrt{S}}{\N^m_h} } \tag{\pref{lem:barPV to PV}}\\
        &\lesssim \sqrt{SA\summp\sumhm\fV(P^m_h, V^m_{h+1})} + JS^{1.5}A. \tag{Cauchy-Schwarz inequality and \pref{lem:sum N}}
    \end{align*}
    This completes the proof.
\end{proof}


\begin{lemma}
    \label{lem:sum N}
    $\summp\sumhm\frac{1}{\N^m_h}\lesssim SA$.
\end{lemma}

\section{Horizon-free Regret is Impossible in SSP under general costs}
\label{app:hf}


Recently \cite{zhang2022horizon} show that in finite-horizon MDPs it is possible to achieve a horizon-free regret bound with no horizon dependency even in logarithmic terms.
For SSPs, \cite{tarbouriech2021stochastic} achieves a nearly horizon-free regret bound $R_K\lesssim\B\sqrt{SAK}\ln\frac{1}{\lambda} + \B S^2A + \lambda\T K$ for any given $\lambda>0$ in $K$ episodes without knowledge of $\T$, where regret $R_K=\sumk(\sum_{i=1}^{I_k}c(s^k_i,a^k_i)-\optV(\sinit))$, and $R_K=\infty$ if $I_k=\infty$ for some $k$.
If a prior knowledge $\uT=\T$ is available, their result is nearly horizon-free with logarithmic dependency on $\T$.
A natural question to ask is whether (completely) horizon-free regret is possible in SSPs without prior knowledge.
We show that this is actually impossible. 

\begin{definition}
    \label{def:hf}
    We say an algorithm is $(c_1, c_2)$-horizon-free if when it takes number of episodes $K\geq 1$, failure probability $\delta\in(0,1)$, and an SSP instance $\calM$ with parameters $\B$, $S$, $A$ as input, it achieves $R_K\leq c_1(\B, S, A, K, \delta)\sqrt{K}+c_2(\B, S, A, K, \delta)$ on $\calM$ with probability at least $1-\delta$,
    where $c_1$, $c_2$ are functions of $\B$, $S$, $A$, $K$, $\delta$ that have poly-logarithmic dependency on $K$ (no dependency on $\T$ and $\frac{1}{\cmin}$).
\end{definition}

\begin{theorem}
    \label{thm:hf}
    For any $c_1,c_2$ that are functions of $\B$, $S$, $A$, $K$, $\delta$, and have poly-logarithmic dependency on $K$, there is no $(c_1, c_2)$-horizon-free algorithm.
\end{theorem}
Note that in regret minimization the regret bound can scale with $\T$ even without knowledge of $\T$, while in sample complexity we cannot (\pref{thm:lower.bound.eps.T}). Therefore, PAC learning in SSP is in some sense more difficult than regret minimization.

\begin{proof}[of \pref{thm:hf}]
    Consider an SSP $\calM_0$ with $\calS=\{s_0, s_1\}$, $\calA=\{a_0, a_g\}$ and $\sinit=s_0$.
    The cost function satisfies $c(s_0, a_0)=0$, $c(s_0, a_g)=1$, and $c(s_1, a)=\frac{1}{2}$ for $a\in\calA$.
    The transition function satisfies $P(g|s_0, a_g)=1$, $P(s_0|s_0,a_0)=1$, and $P(g|s_1, a)=1$ for $a\in\calA$.
    Clearly, $\cmin=0$ and $\B=1$ in $\calM_0$.
    Suppose the learner is a $(c_1, c_2)$-horizon-free algorithm for some functions $c_1$, $c_2$ as described in \pref{def:hf}.
    Pick $\delta\in(0,\frac{1}{8})$ and $K$ large enough as input to the learner, such that $c^0_1\sqrt{K}+c^0_2<\frac{K}{2}$ and $c^1_1\sqrt{K}+c^1_2<\frac{K}{2}$, where $c^0_i=c_i(1, 2, 2, K, \delta)$ and $c^1_i=c_i(\frac{1}{2}, 2, 2, K, \delta)$.
    Let $\calE_1$ be the event that the learner reaches the goal state through $(s_0, a_g)$ in all $K$ episodes.
    Since the learner ensures finite regret with high probability, we have $P(\calE_1)\geq 1-\delta$ in $\calM_0$.
    Denote by $t$ the number of times the learner visits $(s_0, a_0)$.
    By $P(t\Ind_{\calE_1}<\infty)=1$ in $\calM_0$, there exists an integer $n\geq 2$ such that $P(t\Ind_{\calE_1}\leq n)\geq\frac{7}{8}$ in $\calM_0$.
    Define $\calE_2=\{t\leq n\}$ and $\calE=\calE_1\cap\calE_2$.
    We have $P(\calE)=P(\calE_1\cap\{t\Ind_{\calE_1}\leq n\})\geq\frac{3}{4}$ in $\calM_0$ by $\delta\in(0,\frac{1}{8})$.

    Now consider another MDP $\calM_1$ that is the same as $\calM_0$ except that $P(s_1|s_0, a_0)=\frac{1}{n}$ and $P(s_0|s_0,a_0)=1-P(s_1|s_0, a_0)$.
    Clearly, $\B=\optV(\sinit)=\frac{1}{2}$ in $\calM_1$.
    Let $W$ be the interaction history between the learner and the environment, and $L_j(w)=P_j(W=w)$, where $P_j$ is the distribution w.r.t $\calM_j$.
    Also define $\gamma(w)=\Ind\{L_0(w)>0\}$.
    Note that $\frac{L_1(W)}{L_0(W)}\Ind_{\calE}(W)\gamma(W)=(1-\frac{1}{n})^t\Ind_{\calE}(W)\gamma(W)\geq(1-\frac{1}{n})^n\Ind_{\calE}(W)\gamma(W)\geq\frac{\Ind_{\calE}(W)\gamma(W)}{4}$.
    Therefore,
    \begin{align*}
        P_1(\calE_1)\geq P_1(\calE) \geq \E_1[\Ind_{\calE}(W)\gamma(W)] = E_0\sbr{\frac{L_1(W)}{L_0(W)}\Ind_{\calE}(W)\gamma(W)} \geq \frac{P_0(\calE)}{4} \geq \frac{3}{16} > \frac{1}{8}.
    \end{align*}
    Note that the learner ensures $R_K\leq c^i_1\sqrt{K}+c^i_2$ with probability at least $1-\delta$ in $\calM_i$ for $i\in\{0, 1\}$.
    Moreover, when $\calE_1$ is true, $R_K=\frac{K}{2}>c^1_1\sqrt{K}+c^1_2$ in $\calM_1$.
    Therefore, the learner must ensure $P_1(\calE_1)<\delta<\frac{1}{8}$, a contradiction.
    This completes the proof.
\end{proof}

\section{Auxiliary Lemmas}

\begin{lemma}{\citep[Lemma 6]{rosenberg2020adversarial}}
	\label{lem:hitting}
	Let $\pi$ be a policy whose expected hitting time starting from any state is at most $\tau$.
        Then the probability that $\pi$ takes more than $n$ steps to reach the goal state is at most $2e^{-\frac{n}{4\tau}}$.
\end{lemma}

\begin{lemma}\citep[Lemma 14]{zhang2020reinforcement}
    \label{lem:mvp}
    Define $\Upsilon=\{ v\in[0, B]^{\calS_+}: v(g)=0 \}$.
    Let $f: \Delta_{\calS_+}\times\Upsilon\times\fR^+\times\fR^+\times\fR^+\rightarrow\fR^+$ with $f(p, v, n, B, \iota)=pv-\max\Big\{c_1\sqrt{\frac{\fV(p, v)\iota}{n}}, c_2\frac{B\iota}{n}\Big\}$ with $c_1^2\leq c_2$.
    Then, $f$ is non-increasing in $v$, that is, for all $p\in\Delta_{\calS_+}, v,v'\in\Upsilon$ and $n, \iota>0$,
    $$v(s)\leq v'(s), \forall s\in\calS^+ \implies f(p, v, n, B, \iota)\leq f(p, v', n, B, \iota).$$
\end{lemma}

\begin{lemma}\citep[Lemma 16]{chen2022near}
    \label{lem:X2}
    For any random variable $X\in[-C, C]$ for some $C>0$, we have $\var[X^2]\leq 4C^2\var[X]$.
\end{lemma}

\begin{lemma}\citep[Lemma 34]{chen2021implicit}
	\label{lem:anytime bernstein}
	Let $\{X_t\}_t$ be a sequence of i.i.d random variables with mean $\mu$, variance $\sigma^2$, and $0\leq X_t \leq B$.
	Then with probability at least $1-\delta$, the following holds for all $n\geq 1$ simultaneously:
	\begin{align*}
		\abr{\sum_{t=1}^n(X_t-\mu)} &\leq 2\sqrt{2\sigma^2 n\ln\frac{2n}{\delta}} + 2B\ln\frac{2n}{\delta}.\\
		\abr{\sum_{t=1}^n(X_t-\mu)} &\leq 2\sqrt{2\hat{\sigma}^2_nn\ln\frac{2n}{\delta}} + 19B\ln\frac{2n}{\delta}.
	\end{align*}
	where $\hat{\sigma}_n^2=\frac{1}{n}\sum_{t=1}^nX_t^2 - (\frac{1}{n}\sum_{t=1}^nX_t)^2$.
\end{lemma}

\begin{lemma}\citep[Lemma 38]{chen2021improved}
	\label{lem:anytime freedman}
	Let $\{X_i\}_{i=1}^{\infty}$ be a martingale difference sequence adapted to the filtration $\{\calF_i\}_{i=0}^{\infty}$ and $|X_i|\leq B$ for some $B>0$.
	Then with probability at least $1-\delta$, for all $n\geq 1$ simultaneously,
	\begin{align*}
		\abr{\sum_{i=1}^nX_i}\leq 3\sqrt{\sum_{i=1}^n\E[X_i^2|\calF_{i-1}]\ln\frac{4B^2n^3}{\delta} } + 2B\ln\frac{4B^2n^3}{\delta}.
	\end{align*}
\end{lemma}

\begin{lemma}\citep[Lemma 51]{chen2022policy}
    \label{lem:e2r}
    Let $\{X_i\}_{i=1}^{\infty}$ with $X_i\in[0, B]$ be a martingale sequence w.r.t the filtration $\{\calF_i\}_{i=0}^{\infty}$.
    Then with probability at least $1-\delta$, for all $n\geq 1$,
    \begin{align*}
        \sum_{i=1}^n\E[X_i|\calF_{i-1}] &\leq 2\sum_{i=1}^nX_i + 4B\ln\frac{4n}{\delta}.
    \end{align*}
\end{lemma}

\end{document}